%% file: main.tex
\documentclass[runningheads]{llncs}

\usepackage{eccv}

\usepackage{eccvabbrv}

\usepackage{graphicx}
\usepackage{booktabs}

\usepackage[accsupp]{axessibility}  %

\usepackage[colorlinks,hyperfootnotes=false]{hyperref}

\usepackage{orcidlink}
\usepackage{comment}
\usepackage{pst-3d,graphicx,multido}
\usepackage{graphicx}	
\usepackage{amsmath}	
\usepackage{amssymb}	
\usepackage{booktabs}
\usepackage{microtype}
\usepackage{epsfig}
\usepackage{caption}
\usepackage{float}
\usepackage{placeins}
\usepackage{color, colortbl}
\usepackage{stfloats}
\usepackage{enumitem}
\usepackage{tabularx}
\usepackage{xstring}
\usepackage{multirow}
\usepackage{xspace}
\usepackage{url}
\usepackage{subcaption}
\usepackage{xcolor}
\usepackage[para]{footmisc}
\usepackage{stackengine}

\usepackage{wrapfig}

\usepackage{circuitikz}
\usepackage{tikz}
\usepackage{scalefnt}
\usepackage{soul}
\usepackage{xcolor}
\usepackage{soul}
\usepackage{amssymb}
\usepackage{csvsimple}
\usepackage{multicol}
\usepackage{annotate-equations}
\usepackage{dashbox}
\usepackage{verbatim}
\usepackage{pgfplots}
\usepackage{siunitx}
\usepackage{algorithm}
\usepackage{algorithmic}
\usepackage{mathtools}
\usepackage{amsmath}
\usepackage{alphalph}
\usepackage{rotating}
\usepackage{cancel}

\usetikzlibrary{spy}

\DeclareMathOperator*{\argmin}{arg\,min}

\usepackage{scalerel,xparse}

\newcommand\dashedph[1][H]{\setlength{\fboxsep}{0pt}\setlength{\dashlength}{2.2pt}\setlength{\dashdash}{1.1pt} \dbox{\phantom{#1}}}

\newtheorem{assumption}{Assumption}
\newtheorem{proposition*}{Proposition}
\colorlet{myred}{white}
\colorlet{myblue}{blue!20}
\colorlet{mynoise}{yellow!70}
\colorlet{myinv}{green!25}
\colorlet{myrec}{cyan!20}
\colorlet{mysampling}{olive!30}
\colorlet{myblack}{black!10}
\colorlet{myedit}{magenta!30}

\let\emptyset\varnothing
\colorlet{mygold}{yellow!30}
\colorlet{mysilver}{black!10}
\colorlet{mybronze}{white}

\newcommand{\hlc}[2][yellow]{{%
    \colorlet{foo}{#1}%
    \sethlcolor{foo}\hl{#2}}%
}
\newcommand{\mathcolorbox}[2]{\colorbox{#1}{$\displaystyle #2$}}

\makeatletter
\renewcommand{\boxed}[2][mymagenta]{
    \colorlet{foo}{#1}%
    \textcolor{foo}{%
\tikz[baseline={([yshift=-1ex]current bounding box.center)}] \node [line width = 0.5mm, rectangle, minimum width=1ex,minimum height=0.75cm,rounded corners,draw] {\normalcolor\m@th$\displaystyle#2$};}}
 \makeatother

\newcommand{\nbf}[1]{{\noindent \textbf{#1}}}

\makeatletter
\newcommand{\printfnsymbol}[1]{%
  \textsuperscript{\@fnsymbol{#1}}%
}
\makeatother

\newcommand*\samethanks[1][\value{footnote}]{\footnotemark[#1]}

\begin{document}

\title{Eta Inversion: Designing an Optimal Eta Function for Diffusion-based Real Image Editing} 

\titlerunning{Eta Inversion}

\author{Wonjun Kang\thanks{Authors contribute equally.}\inst{1} \and%
Kevin Galim\samethanks\inst{1} \and%
Hyung Il Koo\thanks{Corresponding author.}\inst{1,2}}%

\authorrunning{W. Kang, K. Galim et al.}

\institute{FuriosaAI, Seoul 06036, South Korea\\
\email{\{kangwj1995,kevin.galim,hikoo\}@furiosa.ai}\\
\and
Ajou University, Suwon 16499, South Korea\\
Code: \url{https://github.com/furiosa-ai/eta-inversion}
}

\maketitle

    \input{figs/fig_teaser}

\begin{abstract}
Diffusion models have achieved remarkable success in the domain of text-guided image generation and, more recently, in text-guided image editing. A commonly adopted strategy for editing real images involves inverting the diffusion process to obtain a noisy representation of the original image, which is then denoised to achieve the desired edits. However, current methods for diffusion inversion often struggle to produce edits that are both faithful to the specified text prompt and closely resemble the source image. To overcome these limitations, we introduce a novel and adaptable diffusion inversion technique for real image editing, which is grounded in a theoretical analysis of the role of $\eta$ in the DDIM sampling equation for enhanced editability. By designing a universal diffusion inversion method with a time- and region-dependent $\eta$ function, we enable flexible control over the editing extent. Through a comprehensive series of quantitative and qualitative assessments, involving a comparison with a broad array of recent methods, we demonstrate the superiority of our approach. Our method not only sets a new benchmark in the field but also significantly outperforms existing strategies.
\end{abstract}

\input{main/intro}

\input{main/related}

    \input{tables/tab_notation_1.tex}
\input{main/preliminaries.tex}

\input{main/framework_inversion.tex}

\input{main/theory}

\input{main/proposed_method.tex}

\input{main/exp.tex}

\input{main/discussion.tex}
\input{main/conclusion.tex}

\bibliographystyle{splncs04}
\bibliography{main}

\clearpage
\input{supp_files}

\end{document}

%% file: figs/fig_teaser.tex
\begin{figure*}
    \centering
    \bigskip %
        \resizebox{1\linewidth}{!}{%
\begin{tikzpicture}[spy using outlines={rectangle,yellow,magnification=2.7,size=1.5cm, connect spies}]
\tikzstyle{every node}=[font=\footnotesize]
\node at (6,6.1) {\textit{“the statue of liberty holding a \textbf{torch}”} $\rightarrow$ \textit{“the statue of liberty holding a \textbf{flower}”}};
\node at (0,3.7) {\includegraphics[width=2cm]{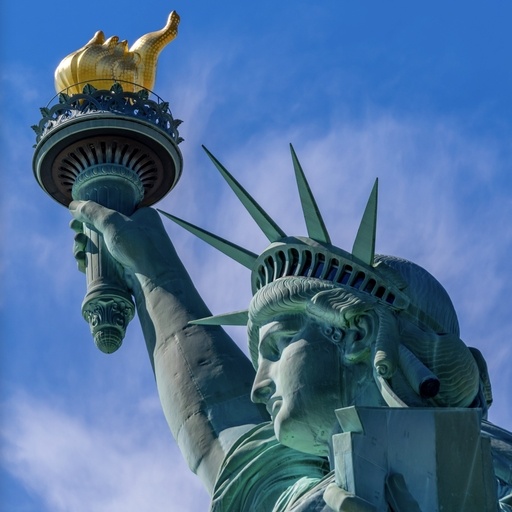}};
\node[fill=black!5, minimum width = 2cm,height = 0.18cm] at (0,2.4) {Real Image};
\spy on (-0.55,4.41) in node [left] at (1.75,4.95);
\node at (3,3.7) {\includegraphics[width=2cm]{figs/teaser/source}};
\node[fill=cyan!10, minimum width = 2cm,height = 0.18cm] at (3,2.4) {Reconstructed};
\spy on (2.45,4.41) in node [left] at (4.75,4.95);
\node at (6,3.7) {\includegraphics[width=2cm]{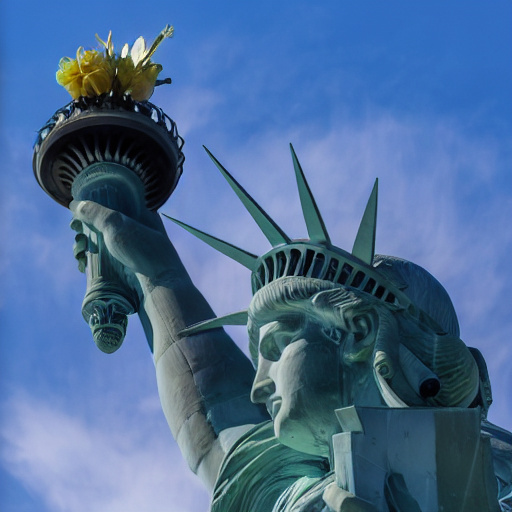}};
\node at (6,2.4) {};
\spy on (5.45,4.41) in node [left] at (7.75,4.95);
\node at (9,3.7) {\includegraphics[width=2cm]{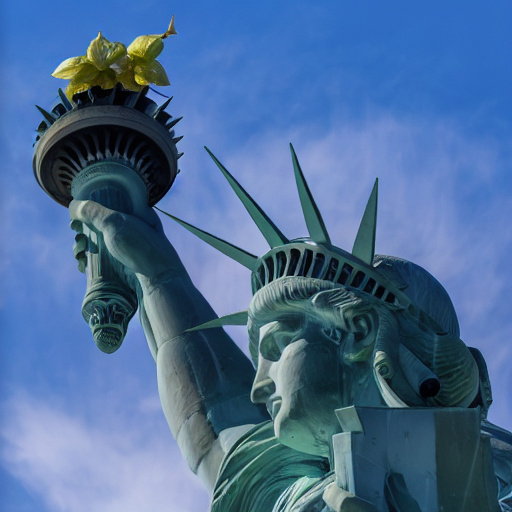}};
\node[fill=magenta!10, minimum width = 8cm,height = 0.18cm] at (9,2.4) {Real Image Editing w/ \textbf{Eta Inversion}};
\node at (9,2.4) {};
\spy on (8.45,4.41) in node [left] at (10.75,4.95);
\node at (12,3.7) {\includegraphics[width=2cm]{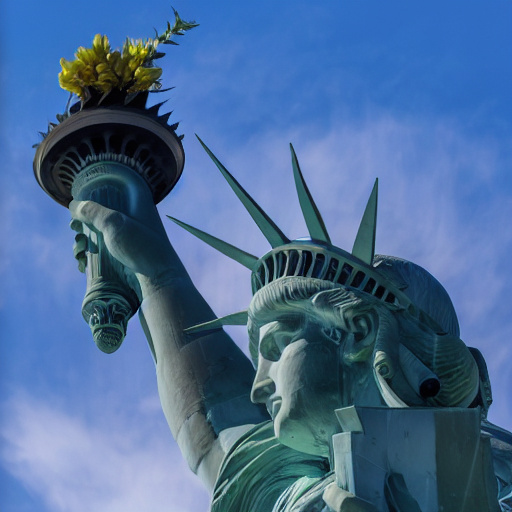}};
\node at (12,2.4) {};
\spy on (11.45,4.41) in node [left] at (13.75,4.95);
\node at (0,0) {\includegraphics[width=2cm]{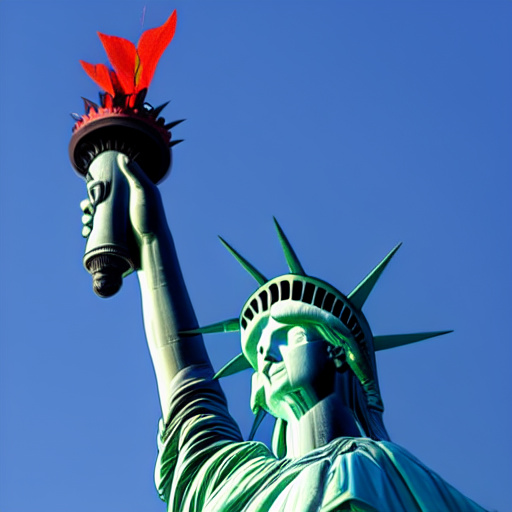}};
\node at (0,-1.3) {DDIM Inv.};
\spy on (-0.55,0.71) in node [left] at (1.75,1.25);
\node at (3,0) {\includegraphics[width=2cm]{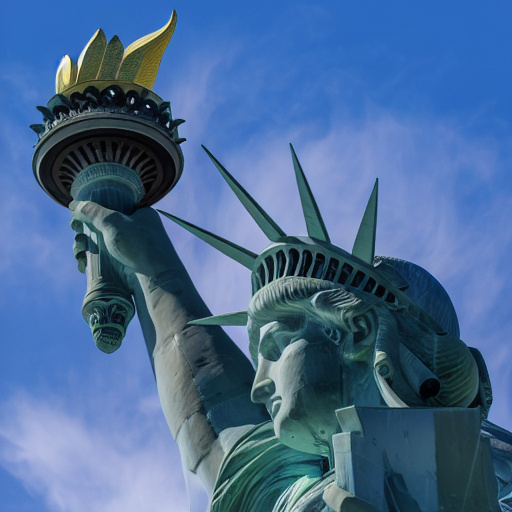}};
\node at (3,-1.3) {Null-text Inv.};
\spy on (2.45,0.71) in node [left] at (4.75,1.25);
\node at (6,0) {\includegraphics[width=2cm]{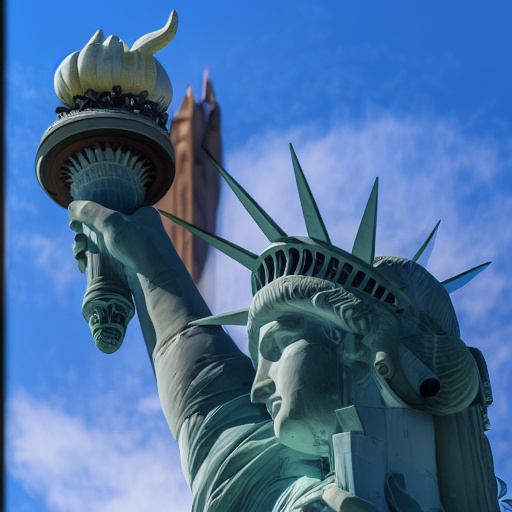}};
\node at (6,-1.3) {Negative Prompt Inv.};
\spy on (5.45,0.71) in node [left] at (7.75,1.25);
\node at (9,0) {\includegraphics[width=2cm]{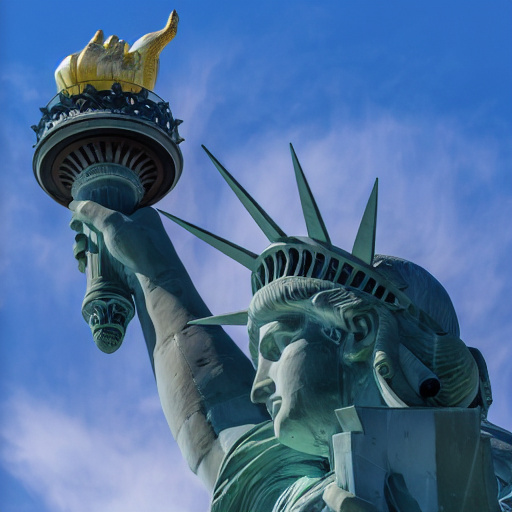}};
\node at (9,-1.3) {EDICT};
\spy on (8.45,0.71) in node [left] at (10.75,1.25);
\node at (12,0) {\includegraphics[width=2cm]{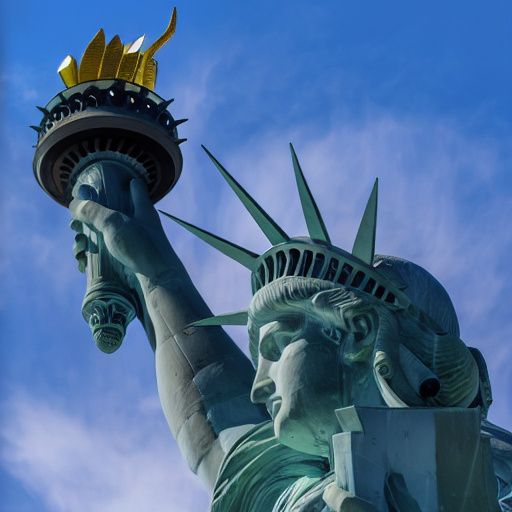}};
\node at (12,-1.3) {Direct Inv.};
\spy on (11.45,0.71) in node [left] at (13.75,1.25);

\end{tikzpicture}
            }%
\caption{Eta Inversion for real image editing. We design an optimal time- and region-dependent $\eta$ function for DDIM sampling\cite{song2020denoising} for superior results. In the example above, existing methods fail to change the torch into a flower or do not preserve the structure, while Eta Inversion creates various plausible results. Tested with PtP \cite{hertz2022prompt}.}
    \label{fig:teaser}
    \end{figure*}

%% file: main/intro.tex
\section{Introduction}
Text-guided image synthesis \cite{nichol2021glide,saharia2022photorealistic,ramesh2021zero,ramesh2022hierarchical,chang2022maskgit,chang2023muse,kang2023scaling} is one of the essential tasks in computer vision due to its enormous potential for design and art industries. Recent breakthroughs in diffusion models \cite{ho2020denoising, song2020denoising, dhariwal2021diffusion,rombach2022high,podell2023sdxl} drastically increased text-to-image generation performance. Due to the success of diffusion-based image generation, text-guided image editing with diffusion models is also gaining interest in the research community \cite{hertz2022prompt, cao2023masactrl, tumanyan2023plug, parmar2023zero, couairon2022diffedit, meng2021sdedit}. However, editing a real image is challenging and existing methods still struggle to produce consistent high-quality results, yet insufficient to the industry's high demand and interest. 

Given a source image, a source prompt describing that image, and a target prompt describing the desired output image, it is possible to 
invert the diffusion process for the source image and edit the inverse latent according to the target prompt. Similar to GAN inversion~\cite{xia2022gan,richardson2021encoding}, diffusion inversion seeks to identify the latent noise corresponding to a particular image. Unlike GANs~\cite{goodfellow2014generative}, which require a single generation step, diffusion models require many iterative steps, making inversion more challenging.

Despite recent advancements in diffusion inversion~\cite{song2020denoising, mokady2023null, wallace2023edict} and editing methods~\cite{cao2023masactrl, hertz2022prompt, tumanyan2023plug}, proper quantitative evaluation is lacking, particularly studies on all combinations of these techniques. We address this gap by reformulating and integrating existing strategies within a single framework, categorizing existing methods into two distinct groups: perfect reconstruction methods and imperfect reconstruction methods. Using this framework, we conduct a thorough evaluation of all methods under consistent and fair conditions, employing a variety of metrics.

Unlike previous methods that use a fixed $\eta$ value, such as 0 or 1, in the DDIM \cite{song2020denoising} sampling equation, our research explores whether a dynamic $\eta$ function is superior. Consequently, we analyze the role of $\eta$ in diffusion inversion and propose Eta Inversion, a perfect reconstruction method. Eta Inversion utilizes a time- and region-dependent $\eta$ to introduce optimal noise during the backward process, achieving better editing diversity. To our knowledge, we are the first to investigate an optimal time-dependent $\eta$ function to balance editing extent and source image similarity for improved performance. To prevent modifications to the background of the image, we make $\eta$ region-dependent, applying $\eta > 0$ only to specific object regions based on their cross-attention map. Comprehensive experiments validate our findings, demonstrating state-of-the-art performance both quantitatively and qualitatively. Our contributions are:

\begin{itemize}
    \item We formulate a generalized framework for diffusion inversion methods.%
    \item We formally explore the role of $\eta$ in diffusion inversion and real image editing.
    \item We design a time- and region-dependent $\eta$ function to inject optimal real noise and achieve state-of-the-art performance in diffusion inversion.
    \item We provide an extensive benchmark for diffusion inversion by evaluating existing inversion methods using various image editing methods.
\end{itemize}
\label{sec:intro}

%% file: main/related.tex
\section{Related Work}

\subsection{Diffusion Models for Image Generation and Editing}

Diffusion models offer more stable training and better diversity than GANs~\cite{goodfellow2014generative}, making them a common choice for image generation. Denoising Diffusion Probabilistic Models (DDPM) \cite{ho2020denoising} showcased the capabilities of diffusion models but require about 1000 inference steps for quality images. Denoising Diffusion Implicit Models (DDIM) \cite{song2020denoising} improve this by reducing inference steps to 50, removing stochastic elements from DDPM sampling. Although rooted in Variational Inference, diffusion models can also be viewed as score-based models using Stochastic Differential Equations (SDEs) \cite{song2020score}.
Latent Diffusion Models \cite{rombach2022high} perform denoising in compressed latent space, greatly reducing inference cost and time. Stable Diffusion \cite{rombach2022high} has become a standard for text-to-image generation due to its public availability and impressive performance.

Text-guided image editing methods \cite{brooks2023instructpix2pix,hertz2022prompt,cao2023masactrl,tumanyan2023plug,parmar2023zero, couairon2022diffedit} aim to align an image with a target prompt while maintaining its original structure. We focus on methods that require no additional training or optimization for better flexibility. 
Prompt-to-Prompt (PtP) \cite{hertz2022prompt} edits images by injecting cross-attention maps from the source into the target prompt's denoising process. Similarly, Plug-and-Play (PnP) \cite{tumanyan2023plug} not only injects cross-attention maps but also integrates spatial features. Furthermore, MasaCtrl \cite{cao2023masactrl} focuses on motion editing and employs self-attention maps instead of cross-attention maps.

\subsection{Diffusion Inversion Methods}

To perform real image editing, a noisy image or latent representation must first be obtained via diffusion inversion. DDIM Inversion~\cite{song2020denoising} achieves low error reconstruction in an unconditional image generation setting, but classifier-free guidance~\cite{ho2022classifier} leads to significant differences from the input image. 

To address this, Null-text Inversion (NTI) \cite{mokady2023null} optimizes the null-text embedding $\emptyset_{t}$ for each timestep, reducing the inversion gap but adding computational overhead. Negative Prompt Inversion (NPI) \cite{miyake2023negative} replaces the null-text with the source text embedding, providing a fast, inference-only inversion pipeline. ProxNPI \cite{han2023improving} enhances NPI with regularization and reconstruction guidance, improving accuracy with minimal cost.
EDICT \cite{wallace2023edict} achieves exact inversion via an auxiliary diffusion path but doubles inference time. DDPM Inversion \cite{huberman2023edit} and CycleDiffusion \cite{wu2022unifying} use stored variance noise from the forward path for exact inversion, but the non-normal distribution of this noise affects editing performance. Direct Inversion \cite{ju2023direct} preserves similarity to the source image by replacing latents during denoising with those from the DDIM Inversion forward path, though this may limit the extent of editing. They also provided a dataset for editing evaluation.

Unlike previous methods that use a static $\eta$ value, our contribution lies in enhancing editability by designing an optimal dynamic $\eta$ function, an aspect not previously explored. Furthermore, we are the first to employ real noise injection for real image editing. This innovation allows us to optimally add real Gaussian noise during editing with minimal inference overhead, achieving balanced and precise image editing.

\label{sec:related}

%% file: tables/tab_notation_1.tex
\begin{table}
    \caption{Table of notation.}
    \centering
        \small
        \centering
        \setlength{\tabcolsep}{4.25pt}
        \resizebox{\textwidth}{!}{\begin{tabular}{@{}p{0.3cm}p{3cm}|p{0.3cm}p{3.5cm}|p{0.3cm}p{4cm}|p{0.3cm}p{6cm}@{}}
            \toprule
             $\dashedph_{t}$  &: \space $\dashedph$ at timestep $t$ & \mathcolorbox{mynoise}{\dashedph} &: \space noise prediction & $\boldsymbol\epsilon_{t,\theta}$ &: \space noise prediction network & $q_{t}$ &: \space marginal distribution of \cref{eq:sde_forward}\\
            $\dashedph^{(s)}$ &: \space $\dashedph$ of source & \mathcolorbox{mysampling}{\dashedph} & : \space sampling & $\boldsymbol{s}_{t,\theta}$ &: \space score estimation network & $p_{t,\eta_t}$ &: \space marginal distribution of \cref{eq:flow_extend_p}\\
            $\dashedph^{(t)}$ &: \space $\dashedph$ of target & \mathcolorbox{myedit}{\dashedph} &: \space editing & $\alpha_{t}$ &: \space noise schedule & $\mathcal{M}_t$ &: \space attention map\\
            $\dashedph^{*}$ &: \space inverted $\dashedph$ & \mathcolorbox{myinv}{\dashedph} &: \space forward path & $\bar{\alpha}_{t}$ &: \space $\prod_{i=1}^{t}\alpha_{i}$ &$\boldsymbol{w}$& : \space standard Wiener process (forward)\\ 
             $\dashedph'$ &: \space reconstructed $\dashedph$ & \mathcolorbox{myrec}{\dashedph} & : \space backward path & $\boldsymbol\epsilon_{\mathrm{add}}$ &: \space additional noise $\sim \mathcal{N}(0,I)$ &$\boldsymbol{\bar{w}}$& : \space standard Wiener process (backward)\\\bottomrule
        \end{tabular}
        }
     \label{tab:notation}
\end{table}

%% file: main/preliminaries.tex
\section{Preliminaries}
\input{main/preliminaries/diffusion.tex}

\input{main/preliminaries/score}
\input{main/preliminaries/ddiminv.tex}

\label{sec:preliminaries}

%% file: main/preliminaries/diffusion.tex
\subsection{Diffusion Models}
Denoising Diffusion Probabilistic Models (DDPM) \cite{ho2020denoising} are generative models consisting of a noising forward path and a denoising backward path. During the forward path, Gaussian noise $\boldsymbol\epsilon$ is gradually added to the sample data point.

DDPM's backward path consists of a noise prediction step and a sampling step. %
Denoising Diffusion Implicit Models (DDIM) \cite{song2020denoising} are an extended version of DDPM which escape from the Markovian forward process. The general form of the sampling function of DDIM is given as below where $\boldsymbol\epsilon_t$ is the estimated noise at timestep $t$ for latent $\boldsymbol{x}_t$, computed as $\boldsymbol\epsilon_t \leftarrow \boldsymbol\epsilon_{t,\theta}(\boldsymbol{x}_t)$:
\begin{equation}
   \label{eq:ddim_sampling}
      \mathrm{Sample}(\boldsymbol{x}_{t},\boldsymbol{\epsilon}_{t}, \eta_{t})=\sqrt{1/{\alpha}_{t}}(\boldsymbol{x}_{t}-\sqrt{1-\bar{\alpha}_{t}}\boldsymbol{\epsilon}_{t})+\sqrt{1-\bar{\alpha}_{t-1}-\sigma^{2}_{t}}\boldsymbol{\epsilon}_{t} + \sigma_{t}\boldsymbol{\epsilon}_{\mathrm{add}}.
   \end{equation}
$\sigma_{t}$ is defined as $\sigma_{t} = \eta_{t}\sqrt{(1-\bar\alpha_{t-1})/(1-\bar\alpha_{t})}\sqrt{1-\bar\alpha_{t}/\bar\alpha_{t-1}}$, and $\eta_{t} \geq 0$ is a controllable hyperparameter. DDPM is a special case of DDIM where $\eta_{t}=1$ for all $t$, whereas DDIM sampling uses $\eta_{t}=0$, making the sampling procedure deterministic. 
For conditional image generation such as text-to-image generation, the noise estimation network receives an additional conditional input $\boldsymbol{c}$ as $\boldsymbol{\epsilon}_{t,\theta}(\boldsymbol{x}_{t},\boldsymbol{c})$. However, it has been empirically shown that the above conditioning is insufficient for reflecting text conditions, so classifier-free guidance \cite{ho2022classifier} is usually used to amplify the text condition as $\tilde{\boldsymbol\epsilon}_{t,\theta} = w \cdot \boldsymbol{\boldsymbol\epsilon}_{t,\theta}(\boldsymbol{x}_{t},\boldsymbol{c}) + (1-w) \cdot \boldsymbol{\epsilon}_{t,\theta}(\boldsymbol{x}_{t},\emptyset)$, where $w$ is the guidance scale parameter and $\emptyset$ is the empty prompt. %
We can summarize the text-to-image generation procedure as \hlc[mynoise]{Noise Prediction} $\boldsymbol\epsilon_t \leftarrow \tilde{\boldsymbol\epsilon}_{t,\theta}(\boldsymbol{x}_{t},\boldsymbol{c},\emptyset,w=7.5)$ and \hlc[mysampling]{Sampling} $\boldsymbol{x}_{t-1} \leftarrow \mathrm{Sample}(\boldsymbol{x}_{t},\boldsymbol{\epsilon}_{t},\eta_t=0)$ and simplify them as $\boldsymbol{x}_{t-1} \leftarrow \mathrm{DDIM}(\boldsymbol{x}_{t},\boldsymbol{c},\emptyset,w, \eta_t)$.%

%% file: main/preliminaries/score.tex
\subsection{Score-based Models}
\cref{eq:sde_forward} and \cref{eq:sde_backward} are the forward and backward SDE of score-based models corresponding to the forward and backward path of DDPM \cite{song2020score}. \cref{eq:flow} is the probability flow ODE and corresponds to DDIM ($\eta=0$) sampling \cite{song2020denoising,song2020score}. \cref{eq:flow_extend} is the extended version of the backward SDE which has the same marginal distribution $q_t$ for any $\eta \geq 0$, and DDIM sampling (\cref{eq:ddim_sampling}) is a numerical method of \cref{eq:flow_extend} \cite{zhang2022fast}. Similarly, we can train a score function ${\boldsymbol{s}_{t,\theta}(\boldsymbol{x})=-\boldsymbol{\epsilon}_{t,\theta}(\boldsymbol{x})/\sqrt{1-\alpha_t} \approx \nabla_{\boldsymbol{x}}\log q_{t}(\boldsymbol{x})}$ and apply a numerical method to \cref{eq:flow_extend_p}.  %
\begin{align}
\label{eq:sde_forward}
&\mathop{\mathrm{d}\boldsymbol{x}} = f_{t}\boldsymbol{x}\mathop{\mathrm{d}t} + g_{t}\mathop{\mathrm{d}\boldsymbol{w}}, \quad \biggl( f_{t} = \frac{1}{2}\frac{\mathop{\mathrm{d}\log \alpha_{t}}}{\mathop{\mathrm{d}t}}, g_{t} = \sqrt{-\frac{\mathop{\mathrm{d}\log \alpha_{t}}}{\mathop{\mathrm{d}t}}}\biggr)\\
\label{eq:sde_backward}
&\mathop{\mathrm{d}\boldsymbol{x}} = [f_{t}\boldsymbol{x} - g^{2}_{t}\nabla_{\boldsymbol{x}}\log q_{t}(\boldsymbol{x})]\mathop{\mathrm{d}t} + g_{t}\mathop{\mathrm{d}\boldsymbol{\bar{w}}} \\
\label{eq:flow}
&\mathop{\mathrm{d}\boldsymbol{x}} = [f_{t}\boldsymbol{x} - 0.5g^{2}_{t}\nabla_{\boldsymbol{x}}\log q_{t}(\boldsymbol{x})]\mathop{\mathrm{d}t} \\
\label{eq:flow_extend}
&\mathop{\mathrm{d}\boldsymbol{x}} = [f_{t}\boldsymbol{x} - 0.5(1+\eta_t^2)g^{2}_{t}\nabla_{\boldsymbol{x}}\log q_{t}(\boldsymbol{x})]\mathop{\mathrm{d}t} + \eta_t g_{t}\mathop{\mathrm{d}\boldsymbol{\bar{w}}}\\
\label{eq:flow_extend_p}
&\mathop{\mathrm{d}\boldsymbol{x}} = [f_{t}\boldsymbol{x} - 0.5(1+\eta_t^2)g^{2}_{t}\boldsymbol{s}_{t,\theta}(\boldsymbol{x})]\mathop{\mathrm{d}t} + \eta_t g_{t}\mathop{\mathrm{d}\boldsymbol{\bar{w}}}
\end{align}

%% file: main/preliminaries/ddiminv.tex
\subsection{DDIM Inversion}
\label{subsec:ddiminv}
DDIM Inversion \cite{song2020denoising} is an important technique for real image editing and can be derived from DDIM ($\eta=0$) sampling (\cref{eq:ddim_sampling}) by approximating $\boldsymbol{\epsilon}_{t} \approx \boldsymbol{\epsilon}_{t+1}$:
\begin{gather}
\label{eq:ddim_inversion}
    \boldsymbol{x}_{t+1} = \sqrt{\bar{\alpha}_{t+1}/\bar{\alpha}_{t}}\boldsymbol{x}_{t}+\sqrt{\bar{\alpha}_{t+1}}(\sqrt{1/\bar{\alpha}_{t+1}-1}-\sqrt{1/\bar{\alpha}_{t}-1})\boldsymbol{\epsilon}_{t}.
\end{gather}
DDIM Inversion can be written as $\boldsymbol{x}_{t+1} \leftarrow \mathrm{DDIM}_{\mathrm{inv}}(\boldsymbol{x}_{t},c,\emptyset,w)$.
With $w=1$, it encodes latent noise with negligible reconstruction error, but large $w$ values (e.g., $w=7.5$ in Stable Diffusion) result in significant error accumulation, leading to two issues:

\nbf{Reconstruction}
The reconstructed image from the inverted noise differs from the source image and fails to maintain the source's features during image editing.

\nbf{Editability}
The inverted noise deviates from a Gaussian distribution, causing poor editing results and unexpected behavior.

%% file: main/framework_inversion.tex
\section{Generalized Framework}

\subsection{Existing Text-guided Image Editing Methods}
We focus on training-free image editing using diffusion inversion for real image editing. We generate an edited image with a pre-trained text-to-image model $\boldsymbol\epsilon_{\theta}$ like Stable Diffusion and adjust certain input parameters for high-quality editing while maintaining the source image's structure. The process involves denoising with both the source and target prompts. By modifying $\boldsymbol{x}_t^{(t)}$ and $\boldsymbol{c}^{(t)}$ of the target process using information from the source branch, it is possible to steer the editing process for better results (notation in \cref{tab:notation}): 
\begin{equation}
\boldsymbol{x}_{t-1}^{(t)} \leftarrow \mathrm{DDIM}(\boxed[myedit]{\boldsymbol{x}_{t}^{(t)},\boldsymbol{c}^{(t)}},\emptyset,w;\boxed[myedit]{\mathcal{M}_{t}^{(t)}}).
\end{equation}
Source and target paths only differ in the modified input. In practice, existing methods like PtP \cite{hertz2022prompt}, MasaCtrl \cite{cao2023masactrl} and PnP \cite{tumanyan2023plug} inject U-Net's~\cite{ronneberger2015unet} attention maps of the source inference process into the target inference process. %

\subsection{Inversion and Real Image Editing}
\label{subsec:inversion}

To perform real image editing, we need to acquire the inverted source noise $\boldsymbol{x}_{T}^{(s)}$ from the source image $\boldsymbol{x}_{0}^{(s)}$. Applying DDIM Inversion (forward path) can yield $\boldsymbol{x}_{T}^{(s)}$, but with considerable reconstruction errors as discussed in \cref{subsec:ddiminv}. Therefore, diffusion inversion methods aim to enhance the \hlc[myinv]{forward path} for a precise and editable $\boldsymbol{x}_{T}^{(s)}$ and adjust the \hlc[myrec]{backward path} to ensure accurate reconstruction and optimal editing. Details of existing inversion methods are provided in the supplementary materials.

\subsubsection{\hlc[myinv]{Forward} Path of Source (Inversion)}
The forward path can be expressed as $\boldsymbol{x}_{t+1}^{(s)^{*}} \leftarrow \mathrm{DDIM}_{\mathrm{inv}}(\boldsymbol{x}_{t}^{(s)^{*}},\boldsymbol{c}^{(s)},\boxed[myinv]{\emptyset,w})$, with $\emptyset$ or $w$ usually modified. The goal is to emulate the ideal forward path, which is unknown in practice. Many methods use $\mathrm{DDIM}_{\mathrm{inv}}(w=1)$ to ensure $\boldsymbol{x}_{T}^{(s)^{*}}$ aligns well with a Gaussian distribution for better editability.

\subsubsection{\hlc[myrec]{Backward} Path of Source and Target (Reconstruction and Editing)}

The backward process aims to align with the ideal (unknown) or actual forward path. Existing methods focus on matching the actual forward path by controlling $\emptyset$ or $w$ like NTI \cite{mokady2023null} and NPI \cite{han2023improving,miyake2023negative}. When editing images, two backward paths are used: one for the source prompt and one for the target prompt, written as:
\begin{gather}
\label{eq:back_source}
\boldsymbol{x}_{t-1}^{(s)'} \leftarrow \mathrm{DDIM}(\boldsymbol{x}_{t}^{(s)'},\boldsymbol{c}^{(s)},\boxed[myrec]{\emptyset,w}), \\
\label{eq:back_target}
\boldsymbol{x}_{t-1}^{(t)'} \leftarrow \mathrm{DDIM}(\boxed[myedit]{\boldsymbol{x}_{t}^{(t)'},\boldsymbol{c}^{(t)}},\boxed[myrec]{\emptyset,w};\boxed[myedit]{\mathcal{M}_{t}^{(t)}}).
\end{gather}
Inversion methods strive to reduce the gap between $\boldsymbol{x}_{t-1}^{(s)^{*}}$ and $\boldsymbol{x}_{t-1}^{(s)'}$, but perfect reconstruction remains challenging.

\subsubsection{Perfect Reconstruction Methods}
\label{subsubsec:perfect}

To achieve perfect source reconstruction, intermediate latents from the forward path can be directly reused for image editing. By replacing the current latent in the backward source path with the corresponding latent from the forward path at each timestep by setting ${{\boldsymbol{x}_{t-1}^{(s)'}}\leftarrow\boxed[myrec]{\boldsymbol{x}_{t-1}^{(s)^{*}}}}$, we ensure that the backward source path precisely matches the forward path. This alignment guarantees perfect reconstruction, and is employed by CycleDiffusion \cite{wu2022unifying}, DDPM Inversion \cite{huberman2023edit}, and Direct Inversion \cite{ju2023direct}.

%% file: main/theory.tex
\section{Theoretical Analysis of the Role of $\eta_t$}
Diffusion inversion demands accurate source image reconstruction and editable target images (\cref{subsec:ddiminv}). Existing perfect reconstruction inversion methods (\cref{subsubsec:perfect}) satisfy the former but lack editability and often yield images too similar to the source. To enhance editability, we explore improving these methods without compromising diffusion model properties. %

\subsubsection{Motivation}
Using deterministic DDIM sampling, the source and target backward paths differ only in the estimated noise per timestep, leading to limited editing and target images resembling the source. 
We aim to enable the target path to diverge from the source path by introducing a stochastic term (additional noise) using non-zero $\eta_t$ DDIM sampling. In particular, we investigate the optimal design of a function for $\eta_t$ to achieve superior performance.

\begin{proposition}[Proof in Supp.]
\label{prop1}
Let $\delta_{\eta_t} = \|\boldsymbol{x}_{t-1}^{(s)'}-\mathrm{DDIM}(\boldsymbol{x}_{t}^{(t)'},\boldsymbol{c}^{(t)},\eta_t)\|_2$ be the source-target branch distance at timestep $t$.
If $\delta_{0}$ is small, there exists an $\eta_t>0$ that satisfies $ \mathop{\mathbb{E}}_{\boldsymbol\epsilon_{add}}[\delta_{\eta_t}]> \delta_0$.
\end{proposition}
\cref{prop1} indicates that introducing a non-zero $\eta_t$ can encourage the target path to escape from the source path without losing the property of diffusion models.

We further study the role of $\eta_t$ theoretically to address two major problems for real image editing: (i.) inaccurate inversion ($p_T^{(t)} \neq p_T^{(t)'}$) and (ii.) inaccurate editing ($\boldsymbol{s}_{t,\theta}^{(t)}(x) \neq \nabla_{\boldsymbol{x}}\log q_{t}^{(t)}(x)$). We use the continuous-time framework of score-based models and measure the sample quality of generation (editing) with KL Divergence $D_{\mathrm{KL}}$.
\subsection{Inaccurate Inversion ($p_T^{(t)} \neq p_T^{(t)'}$)}
As diffusion inversion methods fail to obtain the ideal inverted $p_T^{(t)}$, the image generation (editing) procedure starts from an inaccurately inverted $p_T^{(t)'}$.
\begin{proposition}[Proof in Supp.]
\label{prop2}
Under mild conditions (see supp.), \cref{eq:prop2} is satisfied, wherein $D_{\mathrm{Fisher}}$ denotes the Fisher Divergence.
\end{proposition}
\begin{equation}
\label{eq:prop2}
D_{\mathrm{KL}}(p_0^{(t)'} \parallel p_0^{(t)}) = D_{\mathrm{KL}}(p_T^{(t)'} \parallel p_T^{(t)}) - \int_0^T \eta_{t}^2 g_t^2 D_{\mathrm{Fisher}}(p_t^{(t)'} \parallel p_t^{(t)})\mathop{\mathrm{d}t}
\end{equation}
\cref{prop2} shares a similar concept to \cite{nie2023blessing,lu2022maximum}, which is generalized to \cref{eq:flow_extend_p}. As $\int_0^T \eta_{t}^2 g_t^2 D_{\mathrm{Fisher}}(p_t^{(t)'} \parallel p_t^{(t)})\mathop{\mathrm{d}t} \geq 0$, we can reduce $D_{\mathrm{KL}}(p_0^{(t)'} \parallel p_0^{(t)})$ by applying $\eta_t$ with $\int_0^T \eta_{t}\mathop{\mathrm{d}t}>0$ (SDE) rather than setting $\eta_t=0$ for all $t$ (ODE). \cref{prop2} indicates that introducing a non-zero $\eta_t$ can improve the backward path of inaccurate diffusion inversion.  %
\subsection{Inaccurate Editing ($\boldsymbol{s}_{t,\theta}^{(t)}(x) \neq \nabla_{\boldsymbol{x}}\log q_{t}^{(t)}(x)$)}
If we would assume that the score estimation network is perfect, such that $\boldsymbol{s}_{t,\theta}^{(t)}(x) = \nabla_{\boldsymbol{x}}\log q_{t}^{(t)}(x)$, the choice of $\eta_t$ would not change the marginal distribution as $p_{t,\eta_t}^{(t)} = q_t^{(t)}$ \cite{cao2023exploring}. However, since we consider training-free image editing methods, and reuse the score estimation network from a pre-trained image generation model, a non-negligible score estimation error is introduced. As a result, $\eta_t$ impacts the marginal distribution and good performance cannot be guaranteed by setting $\eta_t=0$ \cite{cao2023exploring}. Therefore, it is beneficial to optimize $\eta_t$ for superior performance.
\begin{proposition}[Proof in Supp.]
\label{prop3}
Assuming mild conditions (see supp.), if the score estimation function $\boldsymbol{s}_{t,\theta}^{(t)}(x)$ undergoes perturbations only near timestep $T$ and near timestep $0$, there exist a timestep $T_a$ and a timestep $T_b$, along with a large constant $\eta_{\mathrm{const}}>0$, such that $D_{\mathrm{KL}}(p_0^{(t)} \parallel q_0^{(t)})$ becomes reduced when employing $\eta_t$ as \cref{eq:cases}, in comparison to $\eta_t=0$ for all $t$ or $\eta_t=\eta_{\mathrm{const}}$ for all $t$.
\end{proposition}
\begin{equation}
\label{eq:cases}
\eta_t = \begin{cases}
\eta_{\mathrm{const}} & \textrm{if } T \geq t \geq T_a\\
\eta_{\mathrm{const}}(t-T_b)/(T_a-T_b) & \textrm{if } T_a > t \geq T_b\\
0 & \textrm{if } T_b > t \geq 0
\end{cases}
\end{equation}
\cref{prop3} is inspired by several findings of \cite{cao2023exploring}. Even though we need to make assumptions for the score estimation function for our theory, it reveals the insight that decreasing $\eta$ during the backward process can better approximate the true target image distribution and lead to better editing results in practice.

%% file: main/proposed_method.tex
\input{figs/fig_ours.tex}

\section{Proposed Inversion Method}
In this section, we discuss how to design an optimal $\eta$ function based on our theoretical findings. Our full Eta Inversion algorithm is depicted in \cref{alg:eta}. \cref{fig:ours} provides an overview of our method.

\input{main/etainv_temp}

\subsection{Improving the Injected Noise $\boldsymbol\epsilon_{\mathrm{add}}$}

\input{figures/noise_histo}

Although methods like CycleDiffusion \cite{wu2022unifying}, DDPM Inversion \cite{huberman2023edit}, and Direct Inversion \cite{ju2023direct} ensure perfect reconstruction by closing the gap between the forward and backward source path with ${\boldsymbol{x}_{t-1}^{(s)'}} \leftarrow \boldsymbol{x}_{t-1}^{(s)^{*}}$, they can produce unexpected editing results if the distance $\|\boldsymbol{x}_{t-1}^{(s)^{*}}-\mathrm{DDIM}(\boldsymbol{x}_{t}^{(s)'},\boldsymbol{c}^{(s)},\eta_t)\|$ is too large. This issue arises because such compensation violates the properties of diffusion models. Specifically, CycleDiffusion \cite{wu2022unifying} and DDPM Inversion \cite{huberman2023edit} calculate $\boldsymbol\epsilon_{\mathrm{add}}$ to meet the condition $\boldsymbol{x}_{t-1}^{(s)^{*}}=\mathrm{DDIM}(\boldsymbol{x}_{t}^{(s)'},\boldsymbol{c}^{(s)},\eta_t)$. However, this $\boldsymbol\epsilon_{\mathrm{add}}$ deviates from a Gaussian distribution, which adversely impacts image generation and editing.

Our approach also employs the compensation strategy used in \cite{wu2022unifying, huberman2023edit, ju2023direct} to ensure perfect reconstruction but improves on it by sampling $\boldsymbol\epsilon_{\mathrm{add}}$ directly from a Gaussian distribution (\cref{fig:noise_histo}). To minimize the forward-backward gap and reduce the necessary compensation, we sample $\boldsymbol\epsilon_{\mathrm{add}}$ multiple times and select the noise that minimizes this gap using $\argmin \|\boldsymbol{x}_{t-1}^{(s)^{*}}-\mathrm{DDIM}(\boldsymbol{x}_{t}^{(s)'},\boldsymbol{c}^{(s)},\eta_t; \boldsymbol\epsilon_{\mathrm{add}})\|$ (\cref{alg:eta} Backward L. \ref{alg:eta:noise1}, \ref{alg:eta:noise2}).

\input{figs/fig_merge}

%% file: figs/fig_ours.tex
\begin{figure}
    \vspace*{-10pt}
    \centering
    \bigskip %
        \resizebox{0.9\linewidth}{!}{%
    \begin{circuitikz}
        \tikzstyle{every node}=[font=\Huge]

        \draw[line width=5pt,-latex,preaction={draw,mysampling,-,double=mysampling,line width=3cm}] (12.5,7.5) --node[sloped, anchor=center,above=3pt]{\Huge $\boldsymbol\epsilon_{t}^{(t)'}$} (15,7.5);

        \path [fill=mynoise] (12,23) to (12,14) to (7,16) to (7,21) to (12, 23);
        \path [fill=mynoise] (2,23) to (2,14) to (7,16) to (7,21) to (2, 23);
        \path [fill=mynoise] (12,12) to (12,3) to (7,5) to (7,10) to (12, 12);
        \path [fill=mynoise] (2,12) to (2,3) to (7,5) to (7,10) to (2, 12);
        \node at (18.5,18.5) {\includegraphics[width=6cm]{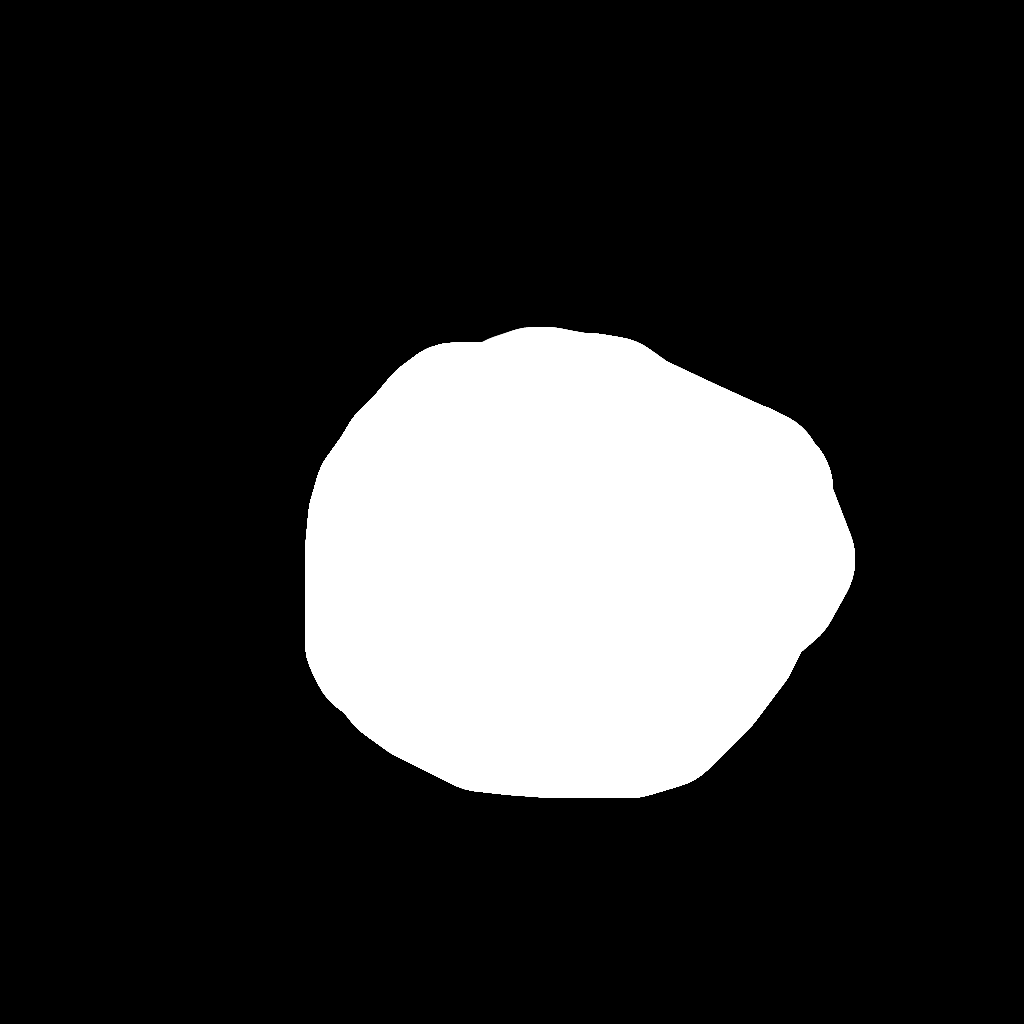}};
        \node [font=\Huge] at (18.5,14.75) {Binary Mask};
        \node at (18.5,7.5) {\includegraphics[width=6cm]{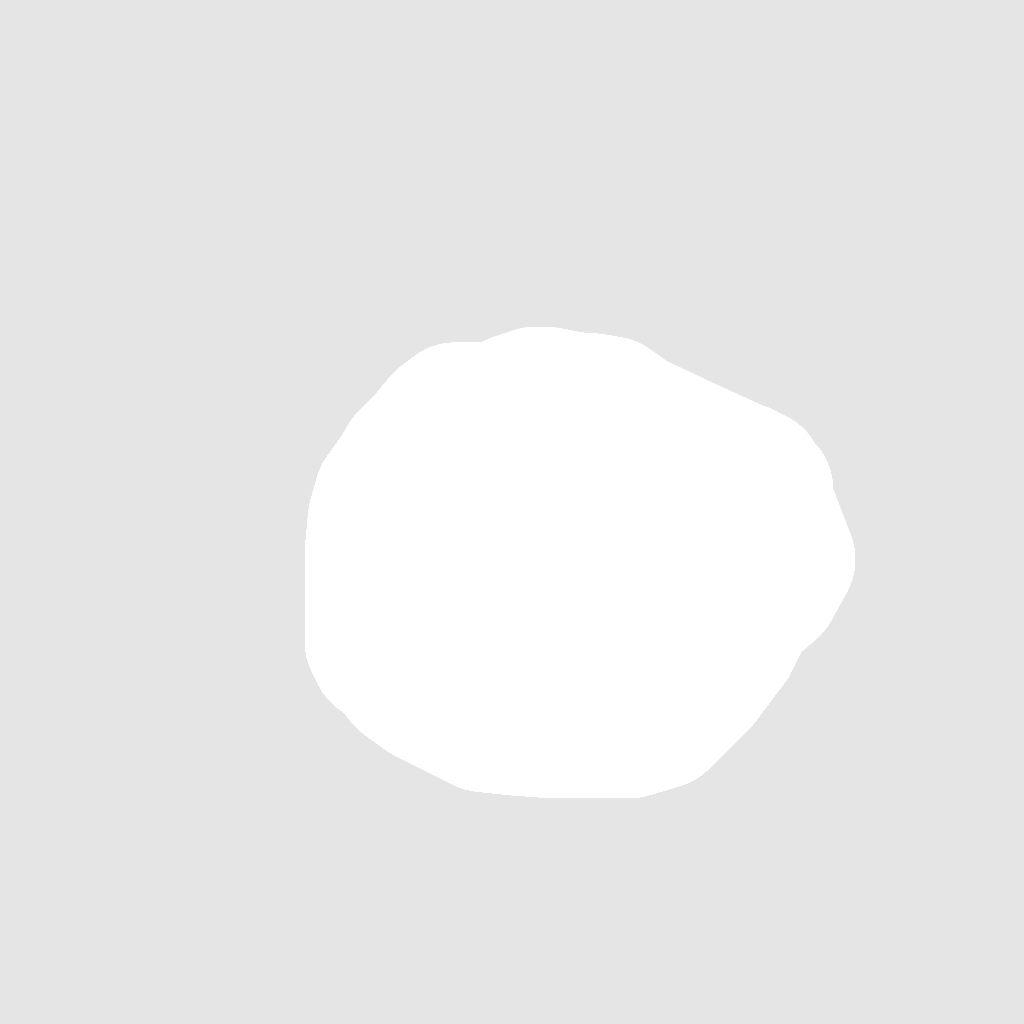}};
        \node [font=\Huge] at (18.5,3.5) {$\mathrm{Sample}(\boldsymbol{x}_{t}^{(t)'},\boldsymbol\epsilon_{t}^{(t)'},\eta_{t})$};
        \node [font=\Huge] at (18.75,7.125) {$\eta_{t} \geq 0$};
        \node [font=\Huge] at (18.5,9.375) {$\eta_{t}=0$};
        \node at (36,26) {\includegraphics[width=6cm]{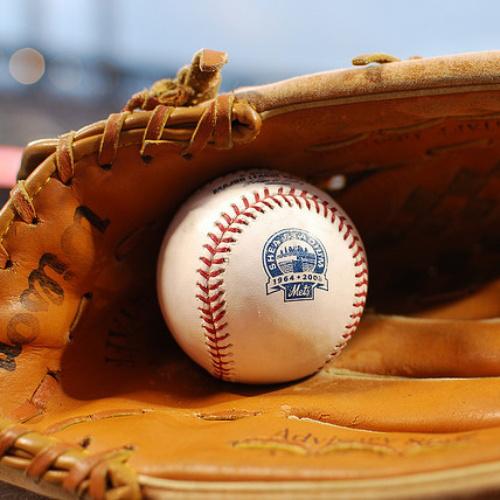}};
        \node [font=\Huge] at (25.5,26) {$\boldsymbol{x}_{t-1}^{(s)^{*}}$};
        \node [font=\Huge] at (31.5,26) {$\boldsymbol{x}_{0}^{(s)}$};
        \draw [line width=2.5pt](33,29) rectangle  node {\Huge } (39,23);
        \node at (36,18.5) {\includegraphics[width=6cm]{figs/concept/source_image}};
        \node [font=\Huge] at (25.5,18.5) {$\boldsymbol{x}_{t-1}^{(s)'}$};
        \node [font=\Huge] at (31.5,18.5) {$\boldsymbol{x}_{0}^{(s)'}$};
        \draw [line width=2.5pt](33,21.5) rectangle  node {\Huge } (39,15.5);
        \node at (36,7.5) {\includegraphics[width=6cm]{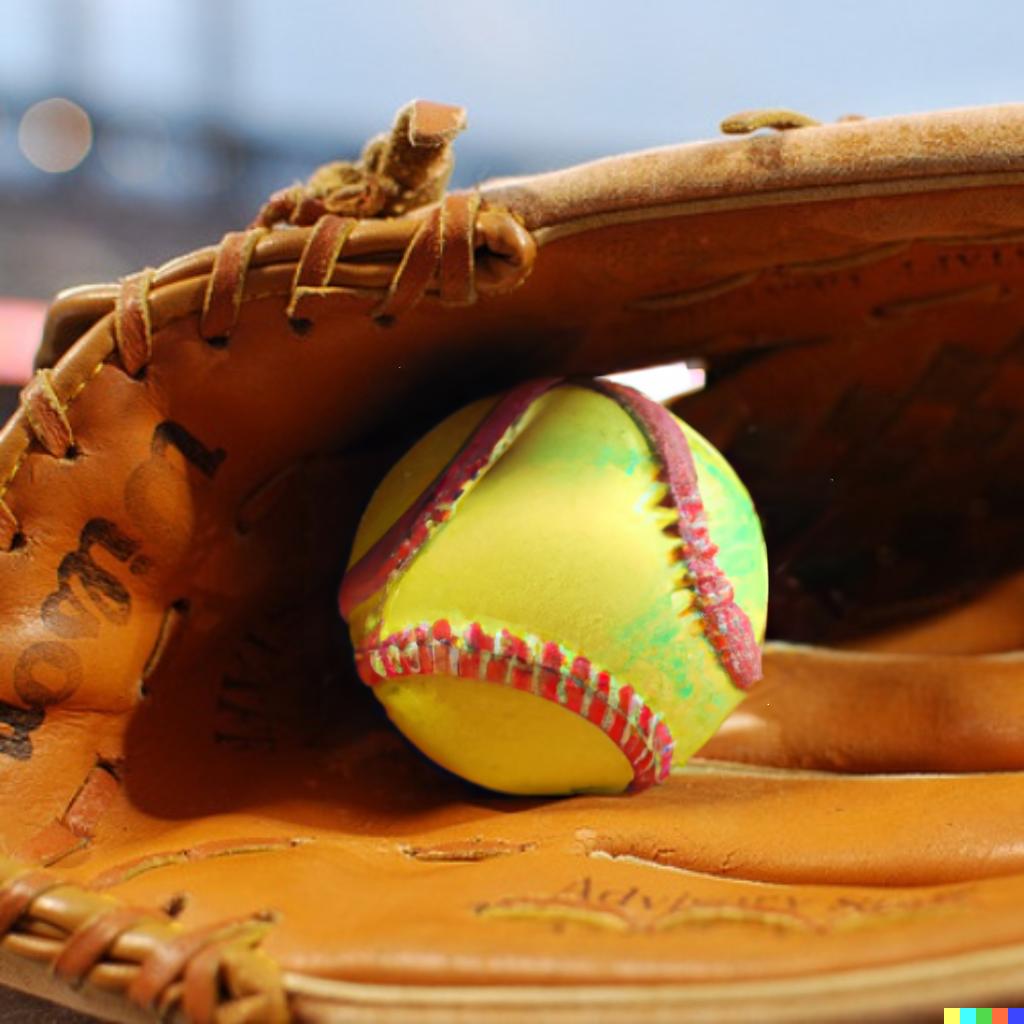}};
        \node [font=\Huge] at (25.5,7.5) {$\boldsymbol{x}_{t-1}^{(t)'}$};
        \node [font=\Huge] at (31.5,7.5) {$\boldsymbol{x}_{0}^{(t)'}$};
        \draw [line width=2.5pt] (33,10.5) rectangle  node {\Huge } (39,4.5);
        \node at (-12.5,18.5) {\includegraphics[width=6cm]{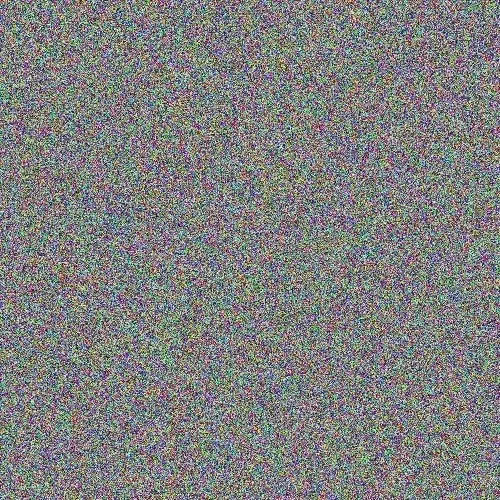}};
        \node at (-12.5,26) {\includegraphics[width=6cm]{figs/concept/noise}};
        \node at (-12.5,6.5) {\includegraphics[width=6cm]{figs/concept/noise}};
        \draw [line width=2.5pt, short,preaction={draw,white,-,double=white,line width=0.5cm}] (12,23) to (12,14) to (7,16) to (2,14) to (2,23) to (7,21) to (12, 23);
        \draw[line width=5pt,-latex] (0,15.25) to (1.5,15.25);
        \draw [line width=2.5pt, short,preaction={draw,white,-,double=white,line width=0.5cm}] (12,12) to (12,3) to (7,5) to (2,3) to (2,12) to (7,10) to (12, 12);
        \path [line width=2.5pt, rounded corners = 15.0] (4,14) rectangle  node {\hlc[myedit]{Editing Method}} (10,12);
        \draw[line width=5pt,-latex,preaction={draw,myedit,-,double=myedit,double distance=5pt}] (7,11.5) to (7,9.5);
        \draw[line width=5pt,-latex,preaction={draw,myedit,-,double=myedit,double distance=5pt}] (8,11.5) to (8,9.5);
        \draw[line width=5pt,-latex,preaction={draw,myedit,-,double=myedit,double distance=5pt}] (6,11.5) to (6,9.5);
        \draw[line width=5pt,-latex] (7,16.5) to (7,14.5);
        \draw[line width=5pt,-latex] (6,16.5) to (6,14.5);
        \draw[line width=5pt,-latex] (8,16.5) to (8,14.5);
        \draw[line width=5pt,-latex] (9,16.5) to (9,14.5);
        \draw[line width=5pt,-latex] (5,16.5) to (5,14.5);
        \draw[line width=5pt,-latex,preaction={draw,myedit,-,double=myedit,double distance=5pt}] (5,11.5) to (5,9.5);
        \draw[line width=5pt,-latex,preaction={draw,myedit,-,double=myedit,double distance=5pt}] (9,11.5) to (9,9.5);
        \draw [line width=2.5pt] (-9.5,21.5) rectangle  node {\Huge } (-15.5,15.5);
        \draw [line width=2.5pt] (-9.5,29) rectangle  node {\Huge } (-15.5,23);
        \draw[line width=5pt,-latex] (0,18.5) to (1.5,18.5);
        \draw [line width=2.5pt] (-15.5,9.5) rectangle  node {\Huge } (-9.5,3.5);
        \node [font=\Huge] at (-1.5,6.5) {$\boldsymbol{x}_{t}^{(t)'}$};
        \node [font=\Huge] at (-1.5,18.5) {$\boldsymbol{x}_{t}^{(s)'}$};
        \node [font=\Huge] at (-1.5,15.25) {$\boldsymbol{c}^{(s)}$};
        \node [font=\Huge] at (-7.5,6.5) {$\boldsymbol{x}_{T}^{(t)'}$};
        \node [font=\Huge] at (-7.5,18.5) {$\boldsymbol{x}_{T}^{(s)'}$};
        \node [font=\Huge] at (-1.5,10.75) {$\boldsymbol{c}^{(t)}$};
        \node [font=\Huge] at (-1.5,26) {$\boldsymbol{x}_{t}^{(s)^{*}}$};
        \node [font=\Huge] at (-7.5,26) {$\boldsymbol{x}_{T}^{(s)^{*}}$};
        \draw[line width=5pt,-latex,dashed,preaction={draw,myrec,-,double=myrec,line width=0.5cm}] (-6,6.5) to (-3,6.5);
        \draw[line width=5pt,latex-,dashed,preaction={draw,myinv,-,double=myinv,line width=0.5cm}] (-6,26) to (-3,26);

        \draw[line width=5pt,latex-,preaction={draw,myinv,-,double=myinv,line width=0.5cm}] (0,26) --node[sloped, anchor=center,above=12pt]{\Huge $\mathrm{DDIM}_{\mathrm{inv}}(\boldsymbol{x}_{t-1}^{(s)^{*}},\boldsymbol{c}^{(s)},\emptyset,w=1)$} (23.5,26);

        \draw[line width=5pt,-latex,preaction={draw,myedit,-,double=myedit,line width=0.75cm}] (0,10.75) to (1.5,10.75);

        \draw[line width=5pt,-latex,preaction={draw,myedit,-,double=myedit,line width=3cm}] (0,6.5) to (1.5,6.5);

        \draw [line width=2.5pt, rounded corners = 15.0, fill=myedit] (3.5,9) rectangle  node {\Huge \hlc[myedit]{$\mathcal{M}_{t}^{(t)'}$}} (10.5,6);

        \node [font=\Huge] at (7,1.5) {\hlc[myedit]{Modified Input} $\rightarrow$ \hlc[mynoise]{Noise Prediction} $\rightarrow$ \hlc[mysampling]{Sampling}};

        \draw[line width=5pt,-latex,preaction={draw,white,-,double=white,line width=5cm}] (11,18.5) --node[sloped, anchor=center,above=3pt]{\Huge Thres.} (15,18.5);
        \path[line width=5pt,-latex,preaction={draw,white,-,double=white,line width=5cm}] (10.5,18.5) --node[sloped, anchor=center,above=3pt]{\Huge } (11,18.5);
        \draw [line width=2.5pt, rounded corners = 15.0,fill=white,solid] (3.5,20) rectangle  node {\Huge $\mathcal{M}_{t}^{(s)'}$} (10.5,17);
                \draw[line width=5pt,-latex] (18.5,14) to (18.5,11);

                \draw[line width=5pt,-latex,preaction={draw,mysampling,-,double=mysampling,line width=3cm}] (22,7.5) to (23.5,7.5);
                    \draw[line width=5pt,-latex,dashed,preaction={draw,myrec,-,double=myrec,line width=0.5cm}] (27,7.5) to (30,7.5);
                    \draw[line width=5pt,latex-,dashed,preaction={draw,myinv,-,double=myinv,line width=0.5cm}] (27,26) to (30,26);

                    \draw[line width=5pt,dashed] (27,18.5) to (30,18.5);
                    \draw[line width=5pt,dashed] (-3,18.5) to (-6,18.5);

                    \draw[line width=5pt,-latex,preaction={draw,myrec,-,double=myrec,line width=0.5cm}] (25.5,24.5) to (25.5,20);
                    \draw[line width=5pt,-latex,preaction={draw,myrec,-,double=myrec,line width=0.5cm}] (31.5,24.5) to (31.5,20);
                    \draw[line width=5pt,-latex,preaction={draw,myrec,-,double=myrec,line width=0.5cm}] (-1.5,24.5) to (-1.5,20);
                    \draw[line width=5pt,-latex,preaction={draw,myrec,-,double=myrec,line width=0.5cm}] (-7.5,24.5) to (-7.5,20);
                    \draw[line width=5pt,-latex,preaction={draw,myrec,-,double=myrec,line width=0.5cm}] (-7.5,17) to (-7.5,8);
    \end{circuitikz}
            }%

    \caption{Eta Inversion for real image editing. We design an optimal time- and region-dependent $\eta$ function to inject real noise in the target path to improve editability. %
    }
    \label{fig:ours}
    \end{figure}

%% file: main/etainv_temp.tex
\subsection{Exploring the Optimal $\eta$ Function}

\subsubsection{Time-dependent $\eta$}
Image editing aims to modify high-level features (e.g., objects) while preserving low-level features (e.g., background details). High-level features are generated early (near timestep $T$), and low-level features are generated later (near timestep $0$) \cite{ho2020denoising, song2020denoising, yue2024exploring}. 
Therefore, we employ a larger $\eta_t$ value initially to edit high-level features and a smaller $\eta_t$ value later to maintain finer details, aligning with \cref{prop1,prop3} by progressively reducing $\eta_t$ for smaller timesteps.

\subsubsection{Region-dependent $\eta$ (Masked $\eta$)}
To improve editing, we employ a region-dependent $\eta$ inspired by existing editing methods~\cite{hertz2022prompt, cao2023masactrl, tumanyan2023plug} that use attention maps to propagate information from the source to the target path. Concurrent with our method, DiffEditor  \cite{mou2024diffeditor} also employs a region-dependent $\eta$ but requires an input mask. Our method, on the other hand, uses cross-attention maps to selectively apply a non-zero $\eta$ to targeted regions without requiring external input. %
By leveraging the cross-attention map for an object and applying noise ($\eta > 0$) only where the map exceeds a threshold (\cref{fig:ours}), we can edit the object while preserving the background. Adjusting the threshold changes the extent of the editing by modifying the region addressed.

%% file: figures/noise_histo.tex
\pgfplotsset{scaled y ticks=false}
\pgfmathdeclarefunction{norm}{3}{%
  \pgfmathparse{1 / (#3 * sqrt(2 * pi)) * exp(-0.5 * ((#1-#2) / #3) ^ 2)}%
}

\begin{figure}[h]
    \centering
\begin{subfigure}[t]{0.48\linewidth}
    \centering
        \resizebox{0.6\linewidth}{!}{%
{\scalefont{2.2}
\begin{tikzpicture}

    \begin{axis}[
        axis lines = left,
        xmin=-5, xmax=5,
        ymax=0.7,
        xlabel style={at={(0.5,-1ex)}},
        ylabel style={at={(-2ex, 0.5)}},
        xlabel = {$\epsilon_{\mathit{add}}^{t=1}$},
        ylabel = {density},
        reverse legend,
        ]

        \addplot[
                ybar interval, 
                color = orange,
                opacity = 0.1,
                fill = orange
            ] table {figures/noise_histo_etainv.dat};
        \addplot [very thick, draw=orange,  domain=-5:5, samples=100, smooth]
                (x, {norm(x, 0.0, 1)})
        node[above,sloped,pos=0.42,color=orange]{$\sigma=1$};
        
        \addplot[
                ybar interval, 
                color = ForestGreen,
                opacity = 0.3,
                fill = ForestGreen
            ] table {figures/noise_histo_ddpminv.dat};
        \addplot [very thick, draw=ForestGreen,  domain=-5:5, samples=100, smooth]
                (x, {norm(x, 0.0, 2.26)})
        node[above,sloped,pos=0.8,color=ForestGreen]{$\sigma=2.26$};
        \legend{,EtaInv,, DDPM Inv.}
    \end{axis}
\end{tikzpicture}}
}%
\caption{Histograms of noise $\epsilon_{\mathit{add}}^{t=1}$ for one sample image from PIE-Bench~\cite{ju2023direct} at $t=1$. Only Eta Inversion shows a true Gaussian shape with a standard deviation ($\sigma$) of 1.}
\label{fig:noise_histo_ddpminv}
\end{subfigure}
\hfill
\begin{subfigure}[t]{0.48\linewidth}
    \centering
        \resizebox{0.6\linewidth}{!}{%
{\scalefont{2.2}
\begin{tikzpicture}

    \begin{axis}[
        every axis y label/.style={at={(current axis.west)},left=5mm},
        axis lines = left,
        xtick={0,0.5,1},
        xticklabels={$T$,$0.5T$,$0$},
        ymax=9,
        xlabel style={at={(0.5,-1ex)}},
        ylabel style={at={(-1ex, 0.5)}},
        xlabel = {$t$},
        ylabel = {$\sigma$},
        ]

        \addplot[
            color = ForestGreen,
        ] table {figures/noise_line_ddpminv.dat};

        \addplot[
            color = orange,
        ] table {figures/noise_line_etainv.dat};

        \legend{DDPM Inv., EtaInv}
    \end{axis}
\end{tikzpicture}}
}%
\caption{The noise standard deviation $\sigma(\epsilon_{\mathit{add}}^t)$, across timesteps, averaged over 100 images. Only Eta Inversion consistently maintains $\sigma=1$ for all timesteps.}
\label{fig:noise_graph}
\end{subfigure}
        \caption{ 
        Noise distribution of DDPM Inversion and Eta Inversion. Eta Inversion applies unit Gaussian noise, unlike DDPM Inversion, which applies noise such that $x_{t}' - x_{t}^* = 0$.
        }
    \label{fig:noise_histo}
    \end{figure}
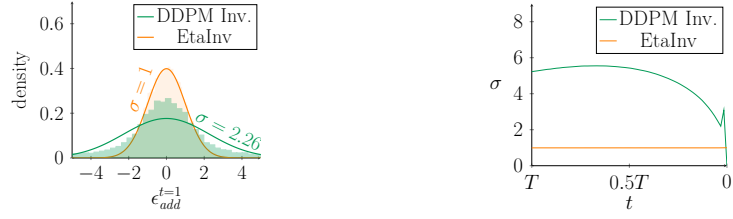

%% file: figs/fig_merge.tex
\begin{algorithm}[H]
    \caption{Eta Inversion}
    \begin{minipage}[t]{0.4\linewidth}
    \vspace{0pt}
    \textbf{Input}: $\boldsymbol{x}_{0}^{(s)}$ 
    
    \textbf{Output}: reconstructed $\boldsymbol{x}_{0}^{(s)'}$, 
    
    \hskip\algorithmicindent edited $\boldsymbol{x}_{0}^{(t)'}$ \vspace{2mm} \hrule \vspace{2mm}
    \textbf{Forward}:%
    \begin{algorithmic}[1] %
    \STATE  \textbf{initialize} $\boldsymbol{x}_{0}^{(s)^{*}} \leftarrow  \boldsymbol{x}_{0}^{(s)}$
    \FOR{$t = 0, 1, ..., T-1$}
    \STATE $\boldsymbol{x}_{t+1}^{(s)^{*}} \leftarrow \mathrm{DDIM}_{\mathrm{inv}}($\par
    \hskip\algorithmicindent $\boldsymbol{x}_{t}^{(s)^{*}},\boldsymbol{c}^{(s)},w=1)$
    \ENDFOR
    \STATE \textbf{return} $\boldsymbol{x}_{T}^{(s)^{*}}, \boldsymbol{x}_{T-1}^{(s)^{*}}, ... ,\boldsymbol{x}_{0}^{(s)^{*}}$
    \end{algorithmic}   
    \end{minipage}
    \vrule
    \vspace{2mm}
    \begin{minipage}[t]{0.58\linewidth}
    \vspace{0pt}
    \vspace{1mm}
    \textbf{Backward}:
    \begin{algorithmic}[1] %
    \STATE \textbf{initialize} $\boldsymbol{x}_{T}^{(s)'} ,\boldsymbol{x}_{T}^{(t)'} \leftarrow  \boldsymbol{x}_{T}^{(s)^{*}}$
    \STATE \textbf{define} time- and region-dependent $\eta_t$%
    \FOR{$t = T, T-1, ..., 1$}    
    \STATE $\boldsymbol{x}_{t-1}^{(s)'}(\boldsymbol\epsilon_{\mathrm{add}}) \coloneqq \mathrm{DDIM}($\par
    \hskip\algorithmicindent $\boldsymbol{x}_{t}^{(s)'}, \boldsymbol{c}^{(s)}, \eta_t, w=7.5; \boldsymbol\epsilon_{\mathrm{add}})$
    \label{alg:eta:ddim1}
    \STATE $\{\boldsymbol\epsilon\} \leftarrow \textrm{sample noise } n \textrm{ times} \sim \mathcal{N}(0,I) $
    \label{alg:eta:noise1}
    \STATE $\boldsymbol\epsilon_{\mathrm{min}} \leftarrow \argmin_{\boldsymbol\epsilon_{\mathrm{add}} \in \{\boldsymbol\epsilon\}} ||\boldsymbol{x}_{t-1}^{(s)^{*}}-\boldsymbol{x}_{t-1}^{(s)'}(\boldsymbol\epsilon_{\mathrm{add}})||$
    \label{alg:eta:noise2}
    \STATE $\boldsymbol{x}_{t-1}^{(s)'} \leftarrow \boldsymbol{x}_{t-1}^{(s)^{*}}$
    \label{alg:eta:dirinv}
    \STATE $\boldsymbol{x}_{t-1}^{(t)'} \leftarrow\mathrm{DDIM}(\boldsymbol{x}_{t}^{(t)'}, \boldsymbol{c}^{(t)},\eta_t, w =7.5; \boldsymbol\epsilon_{\mathrm{min}})$  
    \label{alg:eta:ddim2}
    \ENDFOR
    \STATE \textbf{return} $\boldsymbol{x}_{0}^{(s)'}, \boldsymbol{x}_{0}^{(t)'}$ (satisfying $\boldsymbol{x}_{0}^{(s)'} = \boldsymbol{x}_{0}^{(s)}$)
    \end{algorithmic}
    \end{minipage}
    \label{alg:eta}
    \vspace{-2mm}
    \end{algorithm}

%% file: main/exp.tex
\section{Experiments}
\input{tables/tab_result_edit.tex}

\input{tables/tab_result_edit_etainv3}
\input{figs/fig_result.tex}
\input{figs/fig_graph.tex}
\label{sec:exp}

\subsection{Setup}

We unify and re-implement existing diffusion inversion methods based on diffusers~\cite{von-platen-etal-2022-diffusers} and opt for Stable Diffusion v1.4 \cite{rombach2022high} with $T=50$ steps, using default settings for all methods.
For image editing, we apply PtP \cite{hertz2022prompt}, PnP \cite{tumanyan2023plug}, and MasaCtrl \cite{cao2023masactrl} on the dataset PIE-Bench \cite{ju2023direct}.
Evaluating image editing performance is challenging due to the lack of clear metrics. Prior works \cite{mokady2023null, tumanyan2023plug, ju2023direct} focused on two factors: (i.) text-image alignment, indicating the output image's faithfulness to the target prompt; and (ii.) structural similarity, showing how well the output image preserves the source image's structure.

For text-image alignment we use:
(i.) \textbf{CLIP similarity}: the dot product of normalized CLIP \cite{radford2021learning} embeddings of the target prompt and the output image; and
(ii.) \textbf{CLIP accuracy}: ratio of output images where the text-caption similarity with the target prompt is higher than with the source prompt \cite{parmar2023zero}. Text-caption similarity \cite{chefer2023attendandexcite}  is defined as the CLIP similarity between the target prompt and the BLIP-generated \cite{li2022blip} caption of the output image. For structural similarity we use:
(i.) \textbf{DINOv1 ViT} \cite{caron2021emerging}; 
(ii.) \textbf{LPIPS} \cite{zhang2018perceptual}; and 
(iii.) \textbf{BG-LPIPS} \cite{ju2023direct}, which computes LPIPS only on the background part (mask is provided by PIE-Bench).

We present our results on the complete PIE-Bench dataset, as well as on the change-style subset of PIE-Bench, which focuses exclusively on style transfer. In general, we found that a decreasing linear \(\eta\) schedule improves results, and that a larger \(\eta\) results in more editing, which aligns with our findings. Additionally, a larger noise sample count \(n\) achieves better structural similarity scores and more stable editing overall. We propose three distinct linear \(\eta\) functions, each optimized for a specific objective: structural similarity (EtaInv (1)), target prompt alignment (EtaInv (2)), and style transfer (EtaInv (3)). %
The \(\eta\) functions used, additional qualitative and quantitative results, and comprehensive hyperparameter grid search results are included in the supplementary materials.

\subsection{PIE-Bench Results}

\cref{tab:result_edit} presents our results on PIE-Bench with EtaInv (1) and (2).
For PtP, our method balances text-image alignment and structural similarity, achieving the highest CLIP text-image score and a low structural similarity score. PnP also shows our method as the best in CLIP similarity and accuracy. While our structural metrics are inferior, a too low score may indicate insufficient editing (like EDICT's PtP result in \cref{fig:good_images}). Lastly, with MasaCtrl, we achieve the second-best CLIP similarity but worse structural similarity compared to other techniques. \cref{fig:clip_dino} visualizes the trade-off between text-image and structural similarity for PtP (see supplementary for PnP and MasaCtrl).

\cref{fig:good_images} showcases qualitative results for the top-performing methods. Our proposed Eta Inversion demonstrates superior editing performance. Notably, EtaInv (2), which employs a higher $\eta$, promotes more editing. Furthermore, utilizing a region-dependent $\eta$ enhances structural similarity metrics by preserving more background (\cref{fig:mask}) while introducing a slight decrease in CLIP metrics.

\subsubsection{Style Transfer Results}

Style transfer requires changing the whole image to a higher degree than other tasks (e.g., object replacing). Thus, we disable $\eta$ masking and increase $\eta$ to introduce more noise to further enlarge the gap between the source and target branch for a better editing effect. 
\cref{tab:result_edit_style} shows that EtaInv (3) significantly improves CLIP similarity over previous methods, which we attribute to the injected real noise. Although DINO and LPIPS scores suggest underperformance, these metrics are less useful for style transfer, which requires complete image editing. \cref{fig:style_transfer} further demonstrates that EtaInv (3) achieves more impactful and faithful style transfer.

\input{figs/fig_mask_comparison}
\input{figs/fig_result_style.tex}

%% file: tables/tab_result_edit.tex
    \begin{table}[h]
        \caption{Evaluation results of inversion methods with various editing methods on PIE-Bench. Our method achieves the highest CLIP scores in most cases while maintaining relatively low structural similarity scores. \textbf{EtaInv (1)} and \textbf{EtaInv (2)} employ a region-dependent $\eta$, which further helps improve structural similarity compared to their versions without mask (w/o mask).}
        \centering
        \resizebox{1.0\linewidth}{!}{%
        \small
        \centering
        \setlength{\tabcolsep}{4.25pt}
        \sisetup{table-auto-round}
        \begin{tabular}{@{}l|*{6}{S[table-format=2.2,drop-exponent = true,fixed-exponent = -2,exponent-mode = fixed,]}|*{9}{S[table-format=2.2,drop-exponent = true,fixed-exponent = -2,exponent-mode = fixed,]}@{}}
        \toprule
        Metric $(\times {10}^{2})$&\multicolumn{3}{c}{CLIP similarity $\uparrow$} &\multicolumn{3}{c}{CLIP accuracy $\uparrow$} & \multicolumn{3}{c}{DINO $\downarrow$} & \multicolumn{3}{c}{LPIPS $\downarrow$}& \multicolumn{3}{c}{BG-LPIPS $\downarrow$}\\ 
        \cmidrule(r){1-1}\cmidrule(lr){2-4} \cmidrule(lr){5-7} \cmidrule(lr){8-10}\cmidrule(lr){11-13}\cmidrule(lr){14-16} 
        Method & {PtP} & {PnP} & {Masa} & {PtP} & {PnP}& {Masa}& {PtP} & {PnP} &{Masa} &{PtP}&{PnP}&{Masa} & {PtP}& {PnP}& {Masa} \\ \midrule
        
DDIM Inv. \cite{song2020denoising} & 0.3098963077579226 & 0.2937987298624856 & \cellcolor{mygold}0.3074240440130233 & 0.9457142857142857 & 0.8557142857142858 & \cellcolor{mygold}0.95 & 0.06942023908586374 & 0.061101810900228364 & 0.07545868670301778 & 0.4664535358973912 & 0.4084159646289689 & 0.4767742032238415 & 0.24970233517599158 & 0.20844711300505359 & 0.25374347115294116 \\
Null-text Inv. \cite{mokady2023null} & 0.3073391259780952 & 0.30753638380340165 & 0.30068163248045104 & 0.9257142857142857 & 0.9042857142857142 & 0.93 & \cellcolor{mybronze}0.01244411910651 & 0.03266450203722343 & 0.04494556720335303 & \cellcolor{mybronze}0.15125668101278 & 0.3051460664719343 & 0.2501945654968066 & \cellcolor{mybronze}0.05694563213929 & 0.14172093506670666 & 0.11923769590272319 \\
NPI \cite{miyake2023negative} & 0.3048740061053208 & 0.307270291192191 & 0.29544316662209374 & 0.9271428571428572 & 0.9128571428571428 & 0.8728571428571429 & 0.020258124567002857 & 0.026717880195771742 & 0.045089687282951284 & 0.19275611174692 & 0.26182728880750283 & 0.2603068265744618 & 0.08241629246801105 & \cellcolor{mybronze}0.11566226213575 & 0.12409586446343122 \\
ProxNPI \cite{han2023improving} & 0.3031321707155023 & 0.3053714159343924 & 0.29493104496172495 & 0.9242857142857143 & 0.9071428571428571 & 0.8814285714285715 & 0.019193061842317026 & \cellcolor{mybronze}0.02286067415328 & 0.03921245703839564 & 0.17685471942116107 & \cellcolor{mysilver}0.21758586830326 & \cellcolor{mybronze}0.22985350452895 & 0.07764945874233879 & \cellcolor{mysilver}0.09565648488966 & \cellcolor{mybronze}0.10990528252940 \\
EDICT   \cite{wallace2023edict} & 0.2927953909763268 & 0.2469451144444091 & 0.29676960666264807 & 0.9271428571428572 & 0.6342857142857142 & 0.9328571428571428 & \cellcolor{mygold}0.0041248430355751 & 0.042556588900874236 & \cellcolor{mysilver}0.00787121037403 & \cellcolor{mygold}0.0664862139676032 & 0.30217926079939517 & \cellcolor{mygold}0.0858801200972603 & \cellcolor{mygold}0.0309686089570043 & 0.14956702334493457 & \cellcolor{mysilver}0.04201239570757 \\
DDPM Inv. \cite{huberman2023edit} & 0.2942695520392486 & 0.30258001885243824 & 0.2956833677419594 & 0.9271428571428572 & 0.9485714285714286 & 0.93 & \cellcolor{mysilver}0.00416749393883 & \cellcolor{mygold}0.0104443566847060 & \cellcolor{mygold}0.0075142226933634 & \cellcolor{mysilver}0.06871223960737 & \cellcolor{mygold}0.1250114610046148 & \cellcolor{mysilver}0.08646675587764 & \cellcolor{mysilver}0.03273383083458 & \cellcolor{mygold}0.0584413428749030 & \cellcolor{mygold}0.0411543993033358 \\
Direct Inv. \cite{ju2023direct} & 0.3091879120469093 & 0.3131952231909548 & 0.30372176436441284 & 0.9471428571428572 & \cellcolor{mybronze}0.95142857142857 & \cellcolor{mysilver}0.94571428571428 & 0.012763145404169335 & \cellcolor{mysilver}0.02274692072533 & 0.04321882900914976 & 0.15788746399538858 & \cellcolor{mybronze}0.25591207100876 & 0.2691428067535162 & 0.06329163567514895 & 0.12984735441078166 & 0.1375892746267969 \\
\textbf{Eta Inversion (1)} & \cellcolor{mybronze}0.31007335275411 & 0.3132966907748154 & 0.30390003306525093 & \cellcolor{mybronze}0.95 & 0.9485714285714286 & 0.9314285714285714 & 0.013433658024296165 & 0.023384014786819795 & \cellcolor{mybronze}0.03661353356404 & 0.16576544786404285 & 0.27332021001194207 & 0.23115233425050974 & 0.06568354647575429 & 0.1405317166185601 & 0.11566979534219302 \\
\textbf{Eta Inversion (1) w/o mask} & 0.30998865040285245 & \cellcolor{mybronze}0.31341044721858 & 0.3037151514845235 & \cellcolor{mysilver}0.95285714285714 & 0.95 & 0.9271428571428572 & 0.013703394079181764 & 0.023664145005334702 & 0.03694477854961795 & 0.16847972245886922 & 0.2767769593266504 & 0.2340249225869775 & 0.06737695838071626 & 0.14328620091425753 & 0.117907582334592 \\
\textbf{Eta Inversion (2)} & \cellcolor{mysilver}0.31247970819473 & \cellcolor{mygold}0.3162599997861045 & \cellcolor{mysilver}0.30624428344624 & \cellcolor{mygold}0.9542857142857143 & \cellcolor{mysilver}0.95285714285714 & 0.9385714285714286 & 0.017007591246760316 & 0.03399889323993453 & 0.05237354726530612 & 0.21136770641963396 & 0.36594185411930086 & 0.33066985155854905 & 0.07995662840633096 & 0.18716991466014796 & 0.16644072913740404 \\
\textbf{Eta Inversion (2) w/o mask} & \cellcolor{mygold}0.3126584818959236 & \cellcolor{mysilver}0.31622652743543 & \cellcolor{mybronze}0.30624066016503 & \cellcolor{mygold}0.9542857142857143 & \cellcolor{mygold}0.9585714285714285 & \cellcolor{mybronze}0.94142857142857 & 0.018496008312795312 & 0.03581177577642458 & 0.05463061279883342 & 0.22771402167156338 & 0.384292874868427 & 0.34808154074209074 & 0.09027566545884058 & 0.20193475016559076 & 0.18027709391292385 \\

        \bottomrule
        \end{tabular}
        }%
        \label{tab:result_edit}
        \end{table}

%% file: tables/tab_result_edit_etainv3.tex
    \begin{table}
        \caption{Evaluation results on the change-style subset of PIE-Bench. \textbf{EtaInv (3)} is optimized for style transfer and uses a larger $\eta$ to significantly outperform previous methods in terms of CLIP similarity. Since style transfer requires changing the whole image, \textbf{EtaInv (3)} does not use $\eta$ masking.}
\centering
        \resizebox{0.93\linewidth}{!}{%
        \small
        \centering
        \setlength{\tabcolsep}{4.25pt}
        \sisetup{table-auto-round}
        \begin{tabular}{@{}l|*{6}{S[table-format=2.2,drop-exponent = true,fixed-exponent = -2,exponent-mode = fixed,]}|*{6}{S[table-format=2.2,drop-exponent = true,fixed-exponent = -2,exponent-mode = fixed,]}@{}}
        \toprule
        Metric $(\times {10}^{2})$&\multicolumn{3}{c}{CLIP similarity $\uparrow$} &\multicolumn{3}{c}{CLIP accuracy $\uparrow$} & \multicolumn{3}{c}{DINO $\downarrow$} & \multicolumn{3}{c}{LPIPS $\downarrow$}\\ 
        \cmidrule(r){1-1}\cmidrule(lr){2-4} \cmidrule(lr){5-7} \cmidrule(lr){8-10}\cmidrule(lr){11-13}
        Method & {PtP} & {PnP} & {Masa} & {PtP} & {PnP}& {Masa}& {PtP} & {PnP} &{Masa} &{PtP}&{PnP}&{Masa}  \\ \midrule

        DDIM Inv. \cite{song2020denoising} & 0.31003336552530525 & 0.3020978208631277 & \cellcolor{mysilver}0.30665721725672 & 0.8375 & 0.7375 & \cellcolor{mysilver}0.8625 & 0.06469417617190629 & 0.06087852059863508 & 0.06898356971796601 & 0.46758218854665756 & 0.42627033982425927 & 0.47422250993549825 \\
Null-text Inv. \cite{mokady2023null} & \cellcolor{mysilver}0.32064904756844 & \cellcolor{mysilver}0.32791467607021 & 0.2996522953733802 & 0.8875 & \cellcolor{mysilver}0.9125 & \cellcolor{mysilver}0.8625 & \cellcolor{mybronze}0.01596638003247 & 0.03983156188041903 & 0.04067557628732175 & 0.19596332940272987 & 0.37259303145110606 & 0.25807287434581666 \\
NPI \cite{miyake2023negative} & 0.31437396369874476 & 0.32366949375718834 & 0.2959686605259776 & \cellcolor{mygold}0.925 & \cellcolor{mybronze}0.9 & 0.75 & 0.022183657478308305 & 0.032950097799766806 & 0.0403935914800968 & 0.22182300006970762 & 0.32263175016269086 & 0.27369909612461923 \\
ProxNPI \cite{han2023improving} & 0.30879320614039896 & 0.3165548058226705 & 0.2938237514346838 & 0.8625 & 0.85 & 0.8 & 0.020226734987227248 & \cellcolor{mybronze}0.02521418955875 & \cellcolor{mybronze}0.03359189067850 & \cellcolor{mybronze}0.19179117949679 & \cellcolor{mysilver}0.25466724950820 & \cellcolor{mybronze}0.22675854773260 \\
EDICT   \cite{wallace2023edict} & 0.2944970604032278 & 0.2532451439648867 & 0.2992744956165552 & \cellcolor{mysilver}0.9125 & 0.5875 & \cellcolor{mygold}0.9 & \cellcolor{mygold}0.0040652410432812 & 0.04217571843182668 & \cellcolor{mysilver}0.00726909328659 & \cellcolor{mygold}0.0668496058380696 & 0.3110722420969978 & \cellcolor{mysilver}0.08567899155896 \\
DDPM Inv. \cite{huberman2023edit} & 0.29779865927994253 & 0.3063500413671136 & 0.29780505541712043 & \cellcolor{mybronze}0.9 & \cellcolor{mybronze}0.9 & \cellcolor{mygold}0.9 & \cellcolor{mysilver}0.00429258551012 & \cellcolor{mygold}0.0097482536919414 & \cellcolor{mygold}0.0065682405082043 & \cellcolor{mysilver}0.06991250173887 & \cellcolor{mygold}0.1252831040183082 & \cellcolor{mygold}0.0850224569672718 \\
Direct Inv. \cite{ju2023direct} & \cellcolor{mybronze}0.31713469438254 & \cellcolor{mybronze}0.32505877781659 & \cellcolor{mybronze}0.30370388943701 & \cellcolor{mysilver}0.9125 & \cellcolor{mygold}0.9375 & \cellcolor{mybronze}0.85 & 0.016382710146717726 & \cellcolor{mysilver}0.02468865788541 & 0.03792487672762945 & 0.19873547358438373 & \cellcolor{mybronze}0.27218190636485 & 0.26557301003485917 \\
\textbf{Eta Inversion (3)} & \cellcolor{mygold}0.3285034172236919 & \cellcolor{mygold}0.3312120966613293 & \cellcolor{mygold}0.3081982402130961 & \cellcolor{mybronze}0.9 & 0.8625 & \cellcolor{mysilver}0.8625 & 0.04185123109491542 & 0.05160329109057784 & 0.06689334658440202 & 0.4775574363768101 & 0.5266271749511361 & 0.46165712028741834 \\

         \bottomrule
        \end{tabular}
        }%
        \label{tab:result_edit_style}    
    \end{table}

%% file: figs/fig_result.tex
\begin{figure*}[t]
    \captionsetup[subfigure]{labelformat=empty}
    \centering
    
    \begin{minipage}{\linewidth}
        \centering
        \scriptsize{\textit{“an \textbf{orange} cat sitting on top of a fence”} $\rightarrow$ \textit{“a \textbf{black} cat sitting on top of a fence”}}
        \smallskip
        \end{minipage}
    
    \subfloat[]{%
    \includegraphics[width=.13\linewidth]{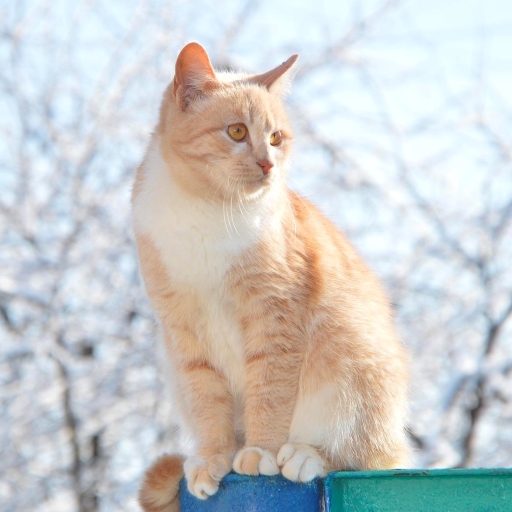}
    }%
    \subfloat[]{%
    \includegraphics[width=.13\linewidth]{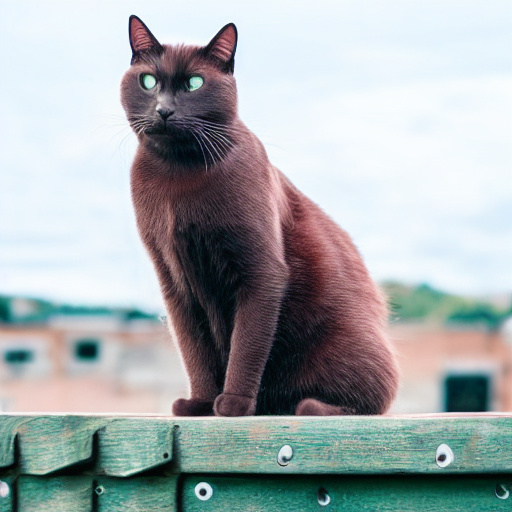}
    }%
    \subfloat[]{%
    \includegraphics[width=.13\linewidth]{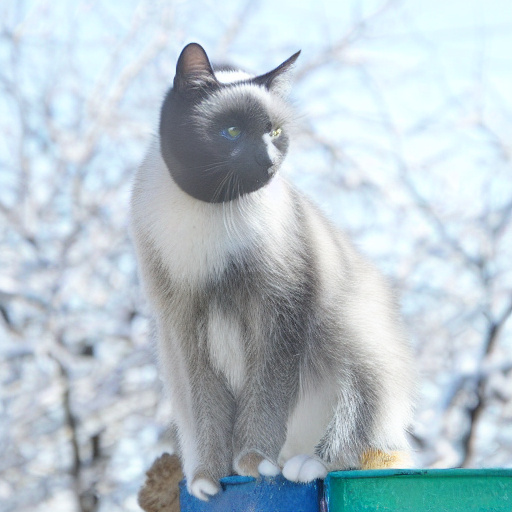}
    }%
    \subfloat[]{%
    \includegraphics[width=.13\linewidth]{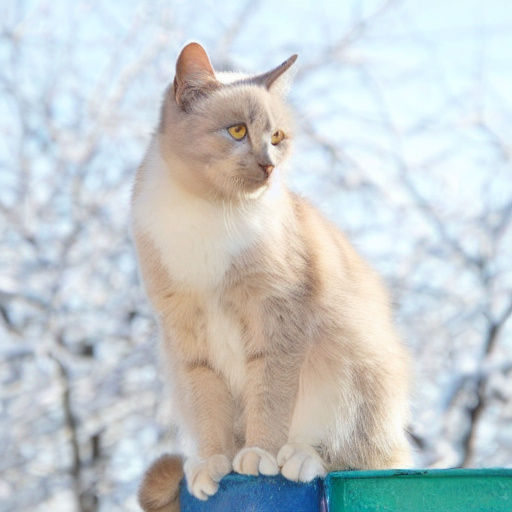}
    }%
    \subfloat[]{%
    \includegraphics[width=.13\linewidth]{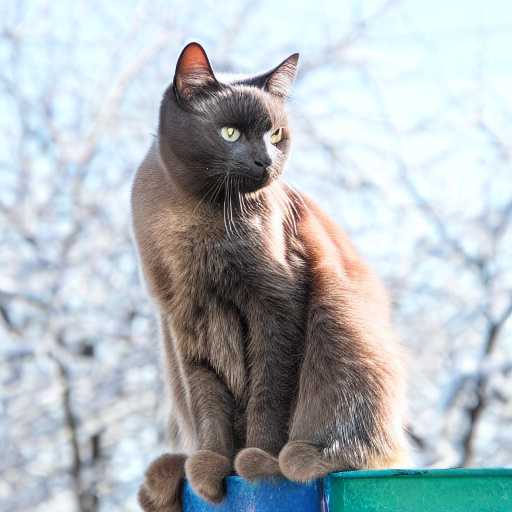}
    }%
    \subfloat[]{%
    \includegraphics[width=.13\linewidth]{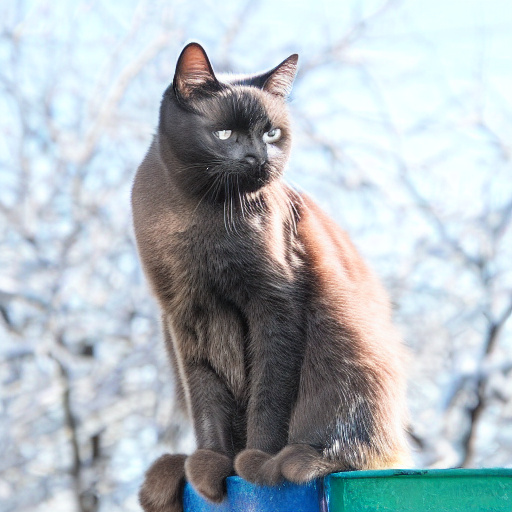}
    }%
    \subfloat[]{%
    \includegraphics[width=.13\linewidth]{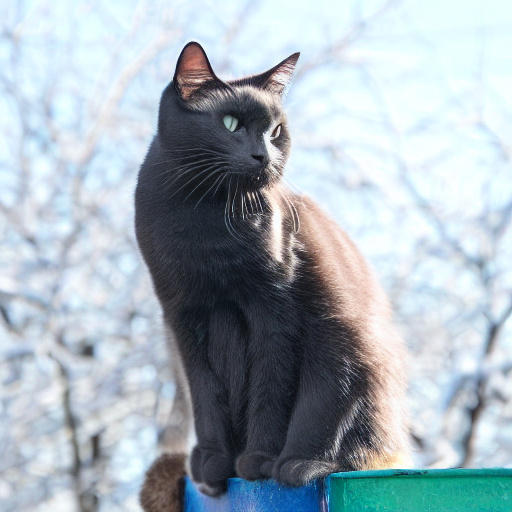}
    }
    \vspace*{-8pt}
    
    \begin{minipage}{\linewidth}
        \centering
        \scriptsize{\textit{“a woman in a \textbf{jacket} standing in the rain”} $\rightarrow$ \textit{“a woman in a \textbf{blouse} standing in the rain”}}
        \smallskip
        \end{minipage}
    
    \subfloat[]{%
    \includegraphics[width=.13\linewidth]{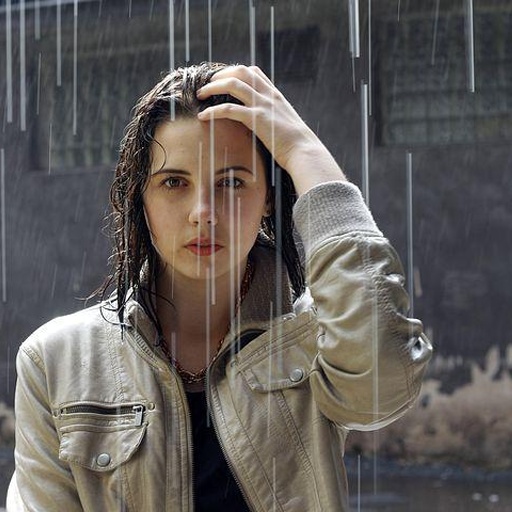}
    }%
    \subfloat[]{%
    \includegraphics[width=.13\linewidth]{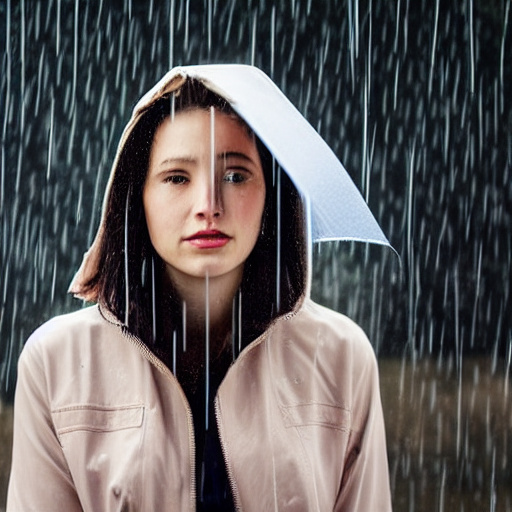}
    }%
    \subfloat[]{%
    \includegraphics[width=.13\linewidth]{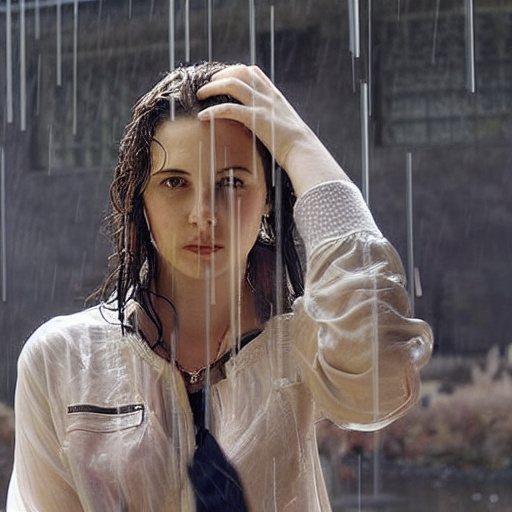}
    }%
    \subfloat[]{%
    \includegraphics[width=.13\linewidth]{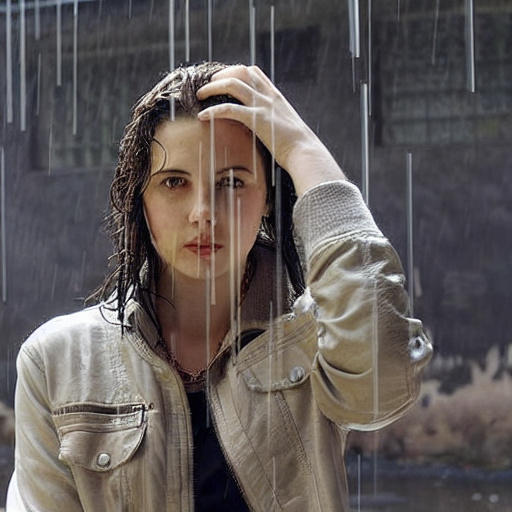}
    }%
    \subfloat[]{%
    \includegraphics[width=.13\linewidth]{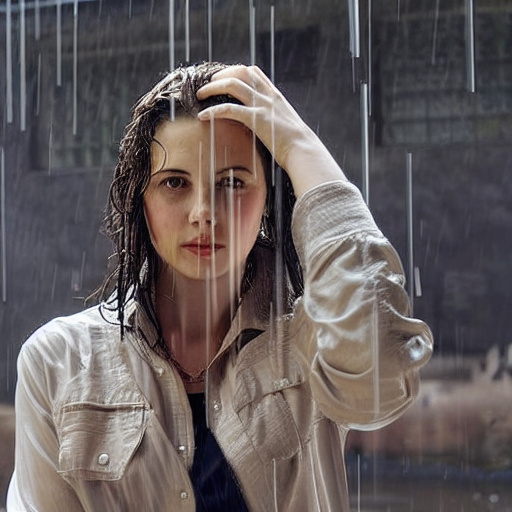}
    }%
    \subfloat[]{%
    \includegraphics[width=.13\linewidth]{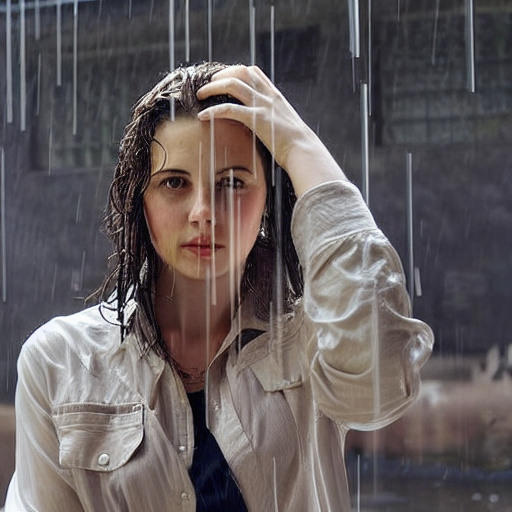}
    }%
    \subfloat[]{%
    \includegraphics[width=.13\linewidth]{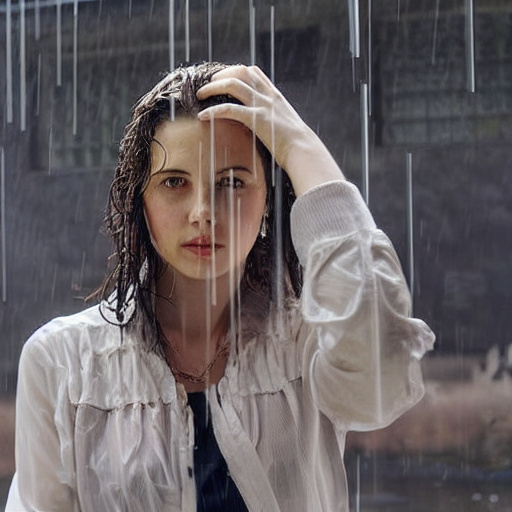}
    }
    
    \vspace*{-8pt}
    \begin{minipage}{\linewidth}
        \centering
        \scriptsize{\textit{“a \textbf{house} in the woods”} $\rightarrow$ \textit{“a \textbf{monster} in the woods”}}
        \smallskip
        \end{minipage}
    
    \subfloat[Source]{%
    \includegraphics[width=.13\linewidth]{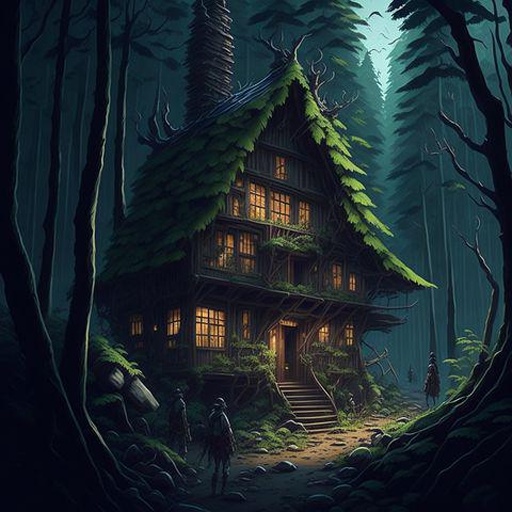}
    }%
    \subfloat[DDIM Inv.]{%
    \includegraphics[width=.13\linewidth]{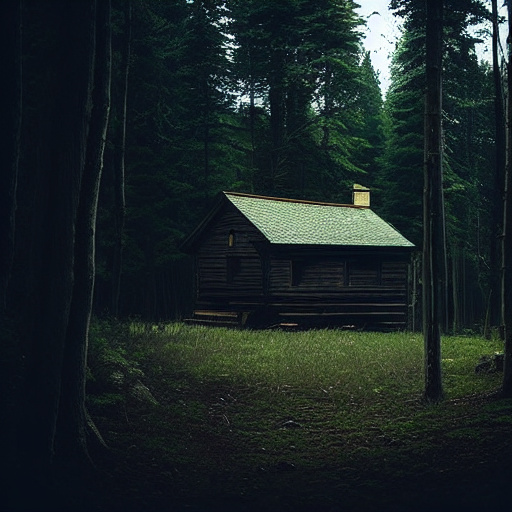}
    }%
    \subfloat[NTI]{%
    \includegraphics[width=.13\linewidth]{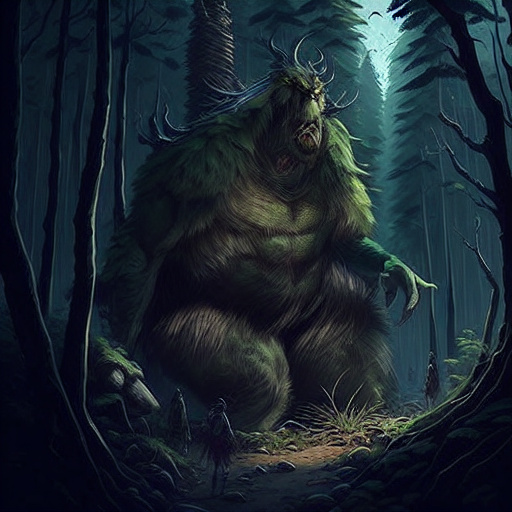}
    }%
    \subfloat[EDICT]{%
    \includegraphics[width=.13\linewidth]{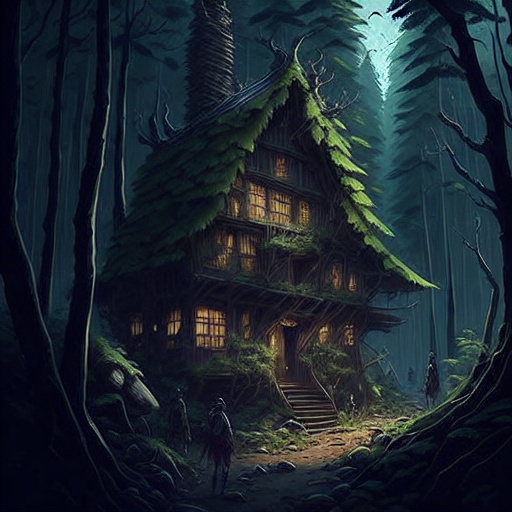}
    }%
    \subfloat[Direct Inv.]{%
    \includegraphics[width=.13\linewidth]{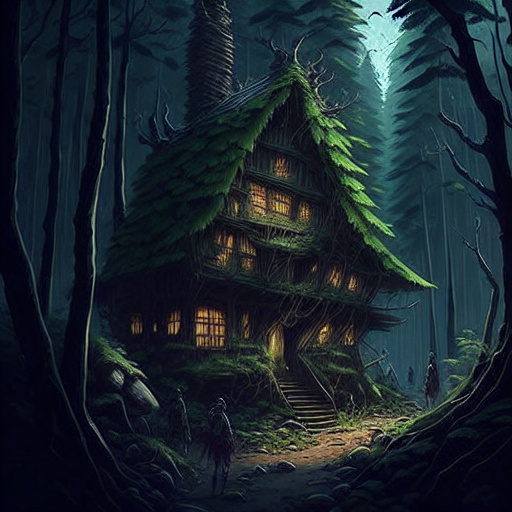}
    }%
    \subfloat[\textbf{EtaInv (1)}]{%
    \includegraphics[width=.13\linewidth]{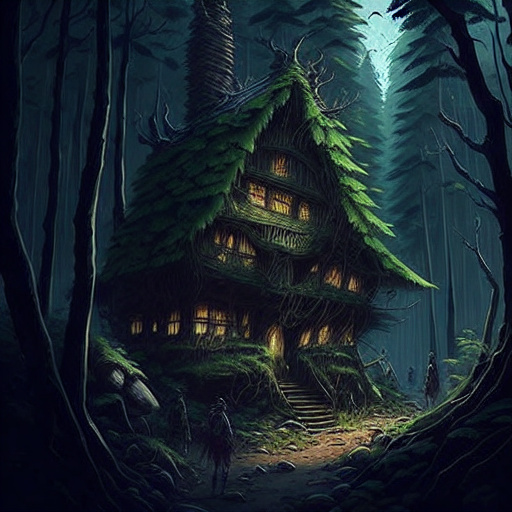}
    }%
    \subfloat[\textbf{EtaInv (2)}]{%
    \includegraphics[width=.13\linewidth]{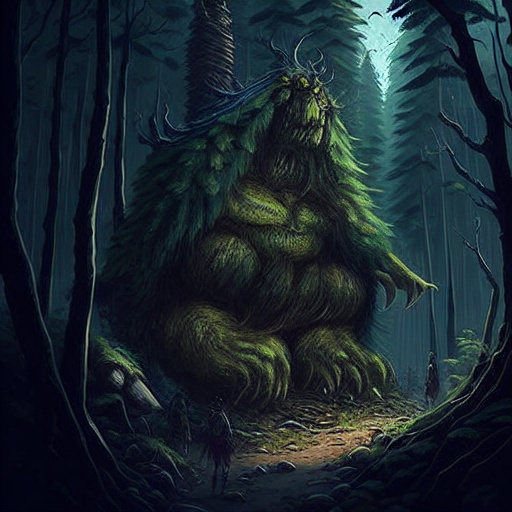}
    }
    \caption{Image editing qualitative results created with PtP \cite{hertz2022prompt} and various inversion methods. Our method, particularly \textbf{EtaInv (2)}, outperforms existing methods and edits the image to a greater degree. %
    We preserve the structure of the source image while correctly editing the image to match the target prompt.}
    \label{fig:good_images}
    \end{figure*}

%% file: figs/fig_graph.tex
\begin{figure}
\begin{subfigure}[b]{0.48\linewidth}
    \centering
    \hspace{-0.5cm}
    \begin{tikzpicture}[scale=0.6]
        \begin{axis}[
            axis lines = left,
            every axis y label/.style={at={(current axis.north west)},above=18mm},
            xlabel = {CLIP similarity ($\times 10^{2}$)},
            ylabel = {DINO ($\times 10^{2}$)},
            legend pos=north west,
            xtick={29.5,30,30.5,31},
            ytick={2,4,6},
            xmin=29, xmax=31.5,
            ymin=0, ymax=7,
            width=1.5\linewidth,
            height=1.0\linewidth
        ]
            \addplot[mark size=1pt,black,mark=*,mark options={fill=black},nodes near coords,only marks,
               point meta=explicit symbolic,
               visualization depends on={value \thisrow{anchor}\as\myanchor},
               every node near coord/.append style={anchor=\myanchor}
            ] table[meta=label] {
            x y label anchor
                30.99 6.94 {\small DDIM Inv.} north
                30.73 1.24 {\small NTI} south
                30.49 2.03 {\small NPI} south
                30.31 1.92 {\small ProxNPI} east
                29.28 0.41 {\small EDICT} south
                29.82 1.19 {\small DDPM $\textrm{Inv.}$} north
                30.92 1.28 {\small Dir. Inv.} north
            };\addlegendentry{PtP}
            \addplot[mark size=1pt,red,mark=*,mark options={fill=red},nodes near coords,only marks,
               point meta=explicit symbolic,
               visualization depends on={value \thisrow{anchor}\as\myanchor},
               every node near coord/.append style={anchor=\myanchor}
            ] table[meta=label] {
            x y label anchor
                31.01 1.34  {EtaInv (1)} west
                31.25 1.70   {EtaInv (2)} west
            };
        \end{axis}
        \end{tikzpicture}
        \caption{Visualization of CLIP text-image metrics (higher is better) and DINO structural similarity metrics (lower is better) on PIE-Bench for PtP.}
        \label{fig:clip_dino}
        \end{subfigure}
\hfill
\begin{subfigure}[b]{0.48\linewidth}
\centering
\scriptsize{\textit{“\textbf{car}”} $\rightarrow$ \textit{“\textbf{motorcycle}”}}
\smallskip
    
\includegraphics[width=.24\linewidth]{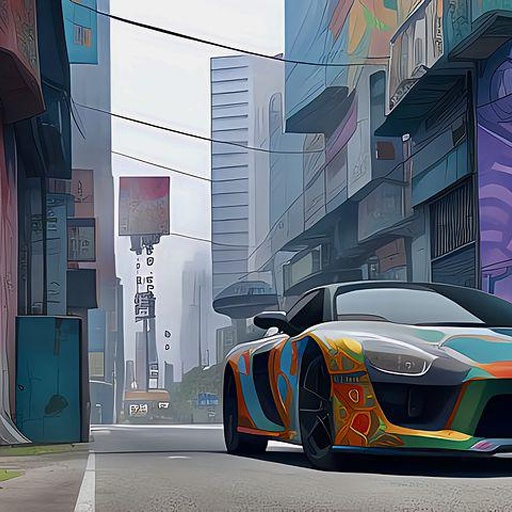}%
\hfill
\includegraphics[width=.24\linewidth]{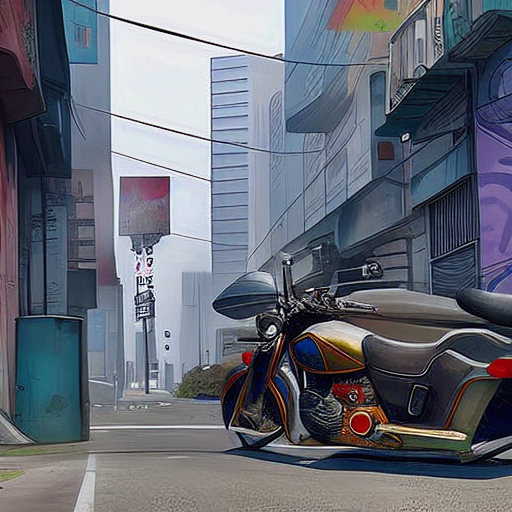}%
\hfill
\includegraphics[width=.24\linewidth]{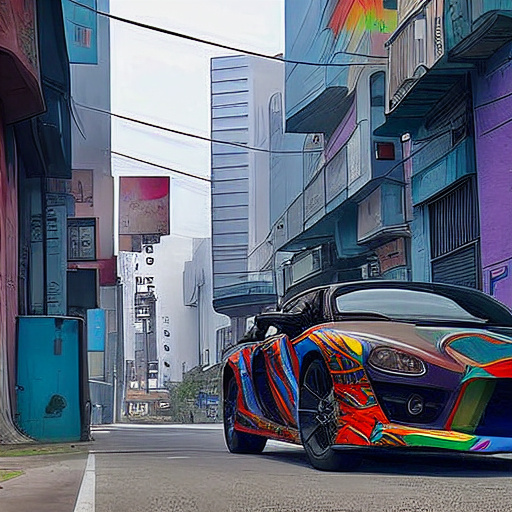}%
\hfill
\includegraphics[width=.24\linewidth]{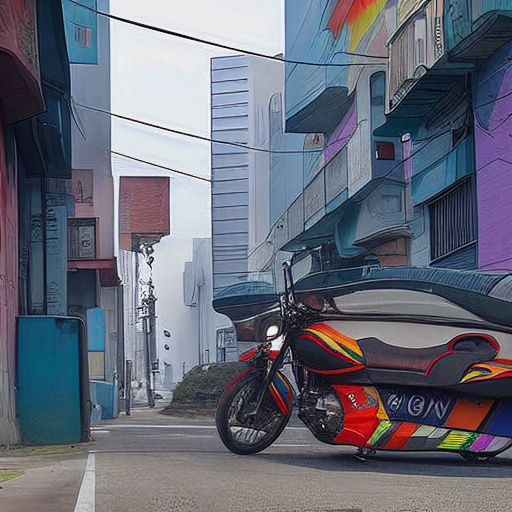}%

\vspace*{2pt}
\scriptsize{\textit{“\textbf{moon}”} $\rightarrow$ \textit{“\textbf{astronaut}”}}
\smallskip
    
\centering
\stackunder[4pt]{\includegraphics[width=.24\linewidth]{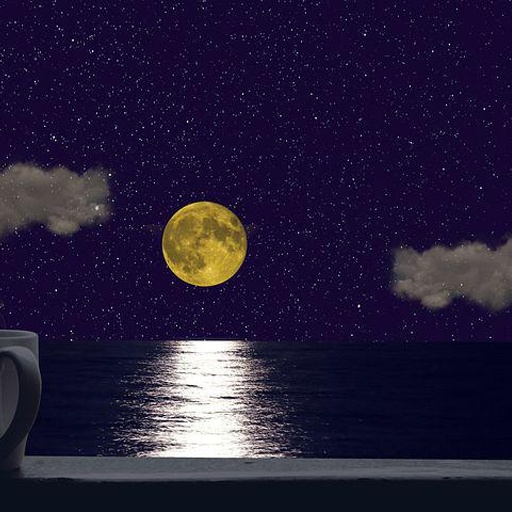}}{Source}%
\hfill
\stackunder[4pt]{\includegraphics[width=.24\linewidth]{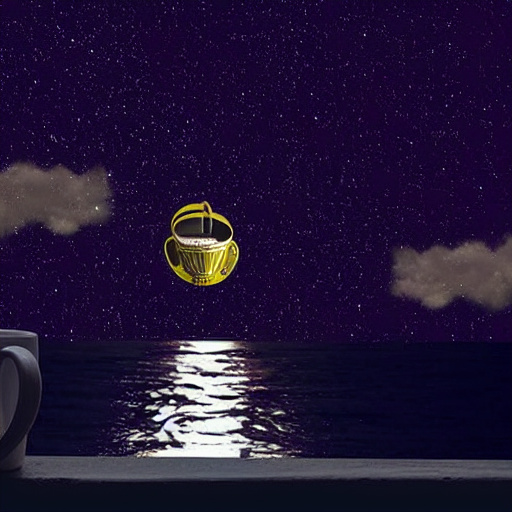}}{NTI}%
\hfill
\stackunder[4pt]{\includegraphics[width=.24\linewidth]{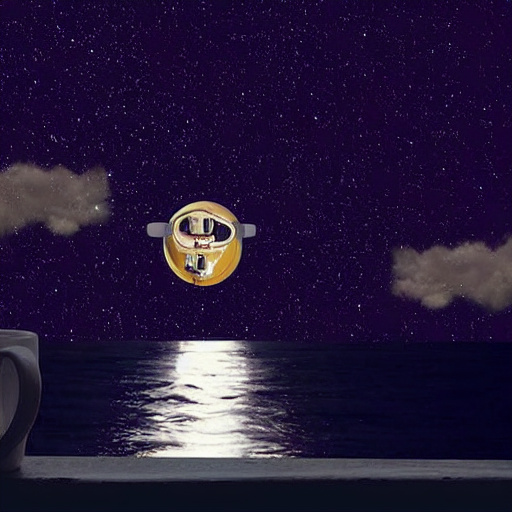}}{Dir. Inv.}%
\hfill
\stackunder[4pt]{\includegraphics[width=.24\linewidth]{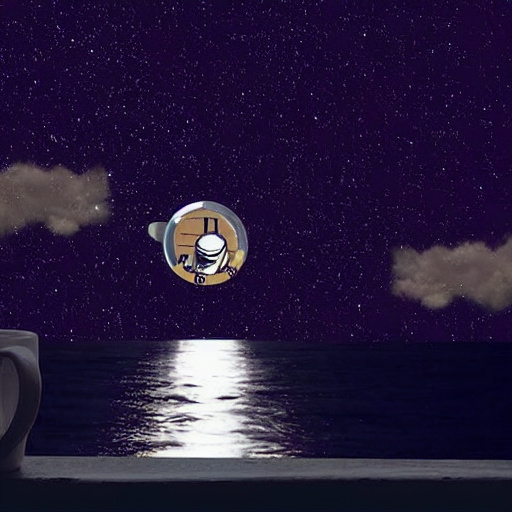}}{\textbf{EtaInv}}%
\medskip

\caption{Failure cases of our proposed method.}
\label{fig:bad_images}
\end{subfigure}
\caption{CLIP-DINO trade-off plot and failure cases.}
\end{figure}
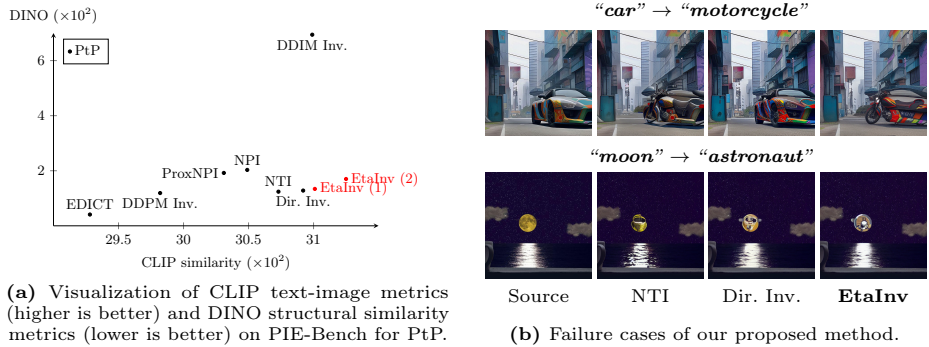

%% file: figs/fig_mask_comparison.tex
\begin{figure*}
    \captionsetup[subfigure]{labelformat=empty}
    \centering
    
    \begin{minipage}{\linewidth}
        \centering
        \scriptsize{\textit{“photo of a \textbf{goat} and a cat standing ...”} $\rightarrow$ \textit{“photo of a \textbf{horse} and a cat standing ...”}}
        \smallskip
        \end{minipage}
    
    \subfloat[Source]{%
    \includegraphics[width=.19\linewidth]{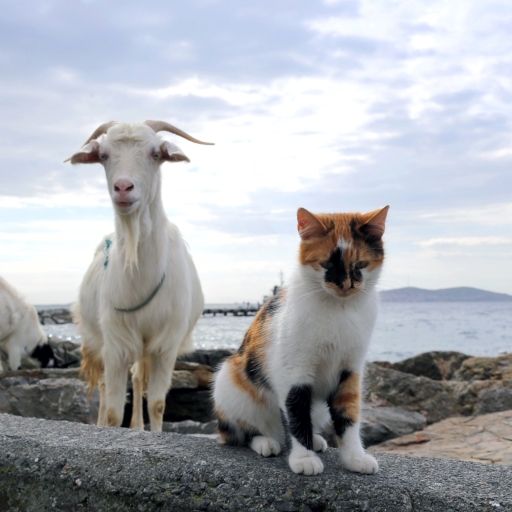}
    }%
    \subfloat[DDIM Inv.]{%
    \includegraphics[width=.19\linewidth]{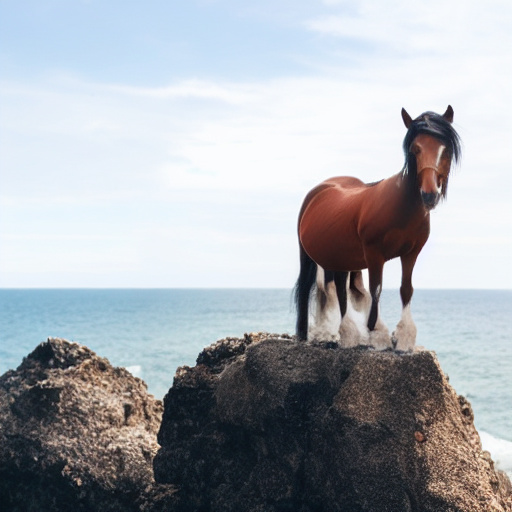}
    }%
    \subfloat[Direct Inv.]{%
    \includegraphics[width=.19\linewidth]{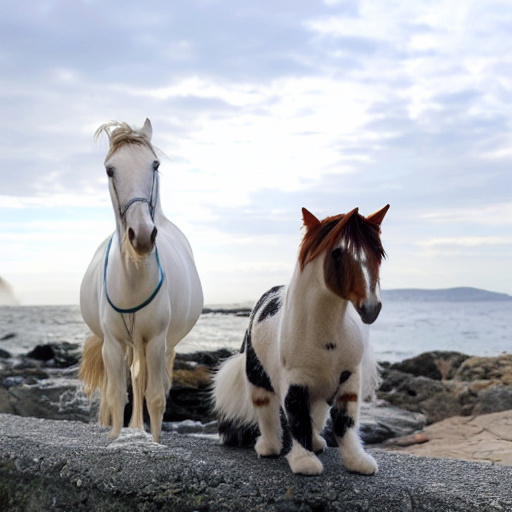}
    }%
    \subfloat[\textbf{EtaInv (no mask)}]{%
    \includegraphics[width=.19\linewidth]{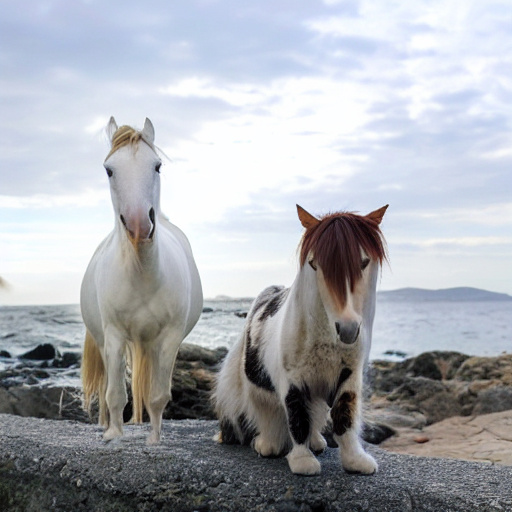}
    }%
    \subfloat[\textbf{EtaInv (mask)}]{%
    \includegraphics[width=.19\linewidth]{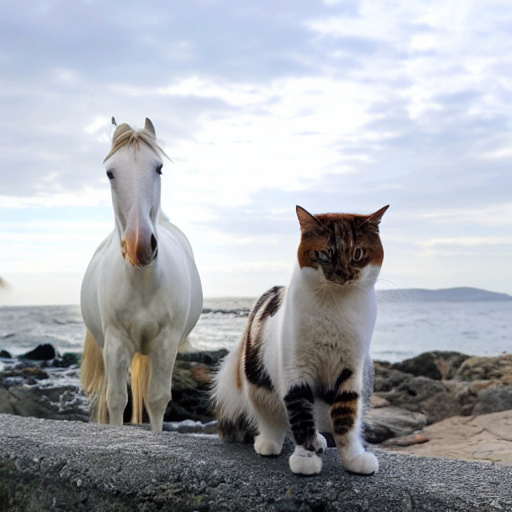}
    }
    \caption{Effectiveness of a region-dependent (masked) $\eta$ function. Only \textbf{EtaInv (mask)} preserves the cat in the original image.}
    \label{fig:mask}
    \end{figure*}

%% file: figs/fig_result_style.tex
\begin{figure*}[t]
    \captionsetup[subfigure]{labelformat=empty}
    \centering
    
    \begin{minipage}{\linewidth}
        \centering
        \scriptsize{\textit{“a kitchen”} $\rightarrow$ \textit{“\textbf{an oil painting of} ...”}}
        \smallskip
        \end{minipage}
    \subfloat[]{%
    \includegraphics[width=.13\linewidth]{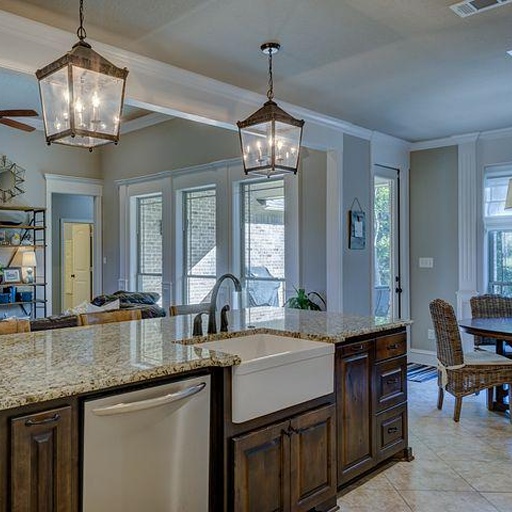}
    }%
    \subfloat[]{%
    \includegraphics[width=.13\linewidth]{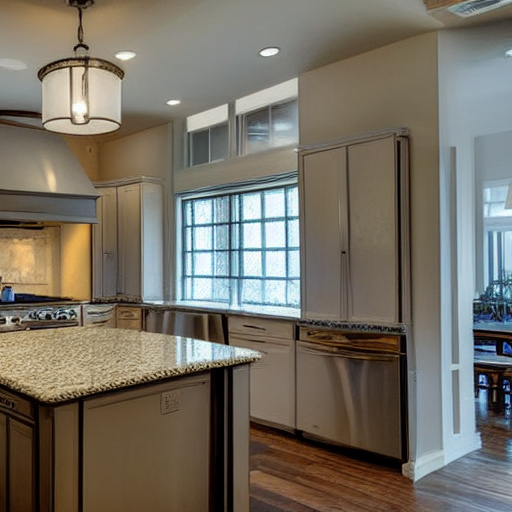}
    }%
    \subfloat[]{%
    \includegraphics[width=.13\linewidth]{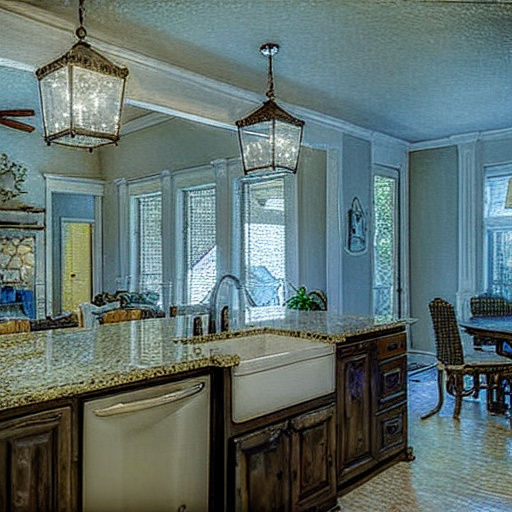}
    }%
    \subfloat[]{%
    \includegraphics[width=.13\linewidth]{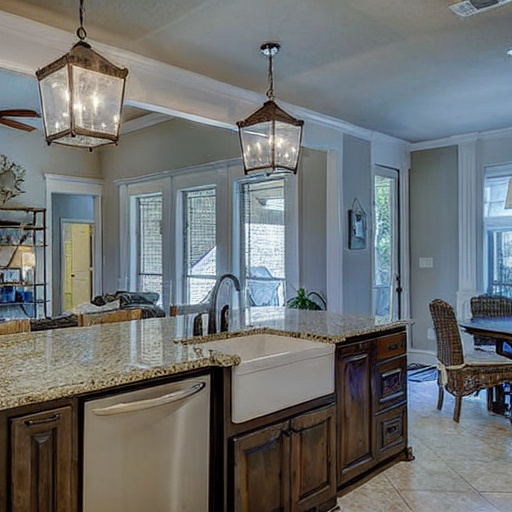}
    }%
    \subfloat[]{%
    \includegraphics[width=.13\linewidth]{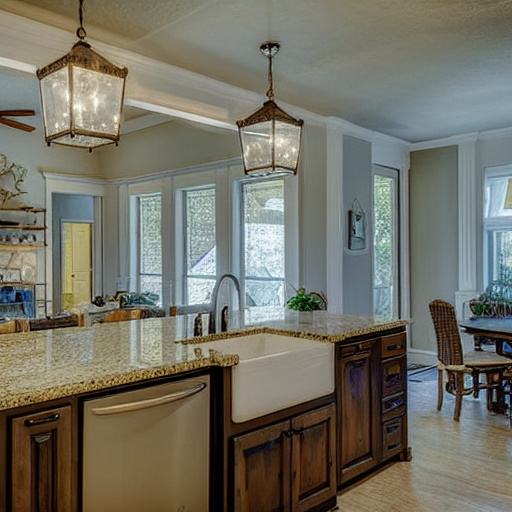}
    }%
    \subfloat[]{%
    \includegraphics[width=.13\linewidth]{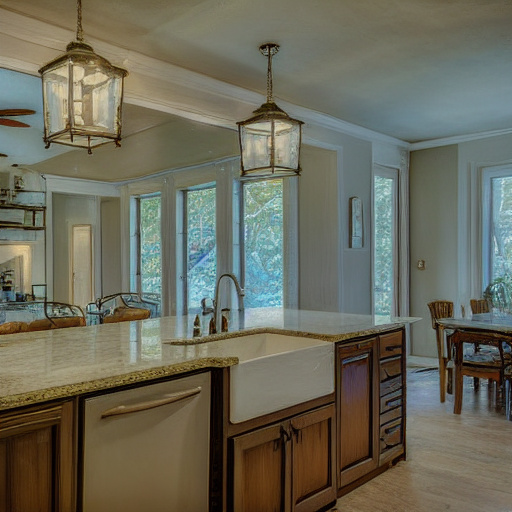}
    }%
    \subfloat[]{%
    \includegraphics[width=.13\linewidth]{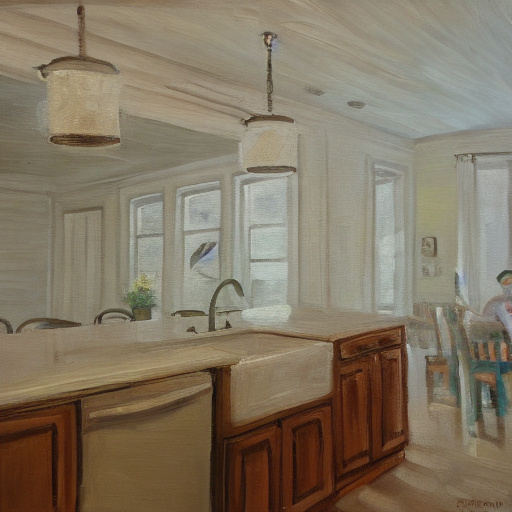}
    }
    \vspace{-8pt}
    \begin{minipage}{\linewidth}
        \centering
        \scriptsize{\textit{“a man with a long beard and a long sword in the forest”} $\rightarrow$ \textit{“\textbf{kids crayon drawing of} ...”}}
        \smallskip
        \end{minipage}
    \subfloat[Source]{%
    \includegraphics[width=.13\linewidth]{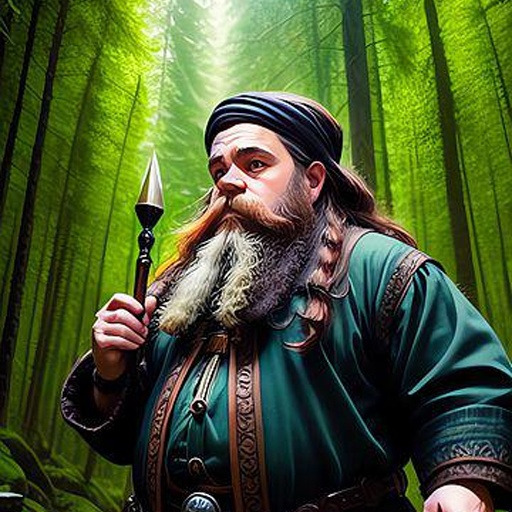}
    }%
    \subfloat[DDIM Inv]{%
    \includegraphics[width=.13\linewidth]{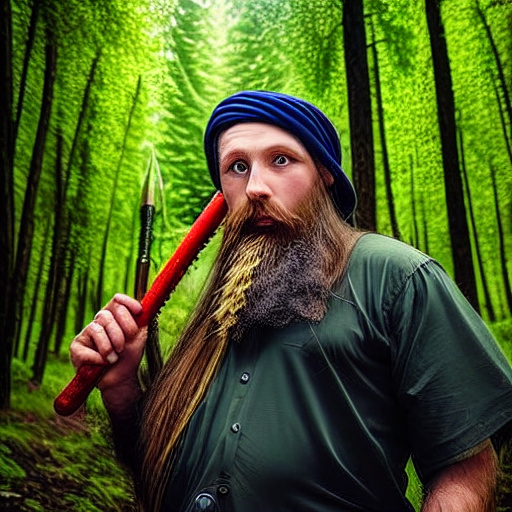}
    }%
    \subfloat[NTI]{%
    \includegraphics[width=.13\linewidth]{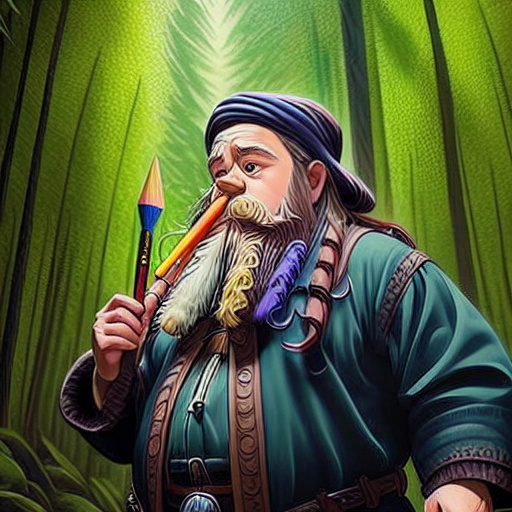}
    }%
    \subfloat[EDICT]{%
    \includegraphics[width=.13\linewidth]{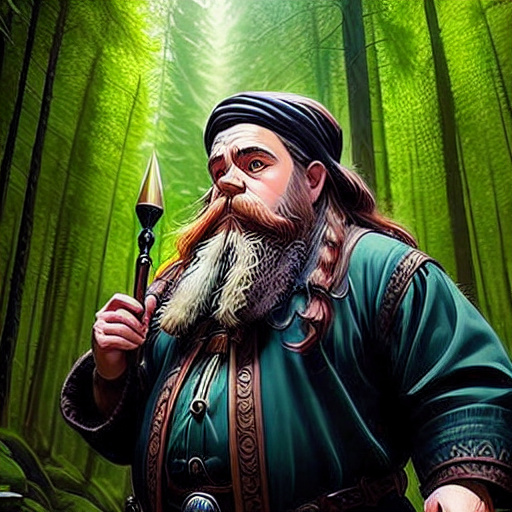}
    }%
    \subfloat[Direct Inv]{%
    \includegraphics[width=.13\linewidth]{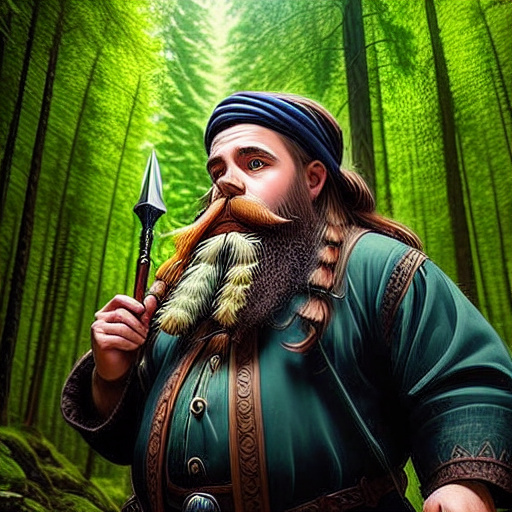}
    }%
    \subfloat[\textbf{EtaInv (2)}]{%
    \includegraphics[width=.13\linewidth]{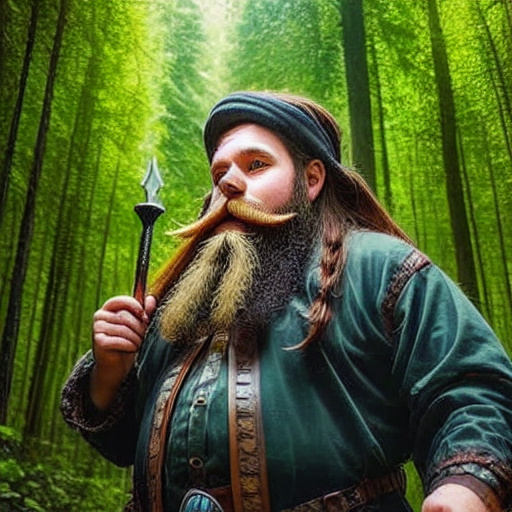}
    }%
    \subfloat[\textbf{EtaInv (3)}]{%
    \includegraphics[width=.13\linewidth]{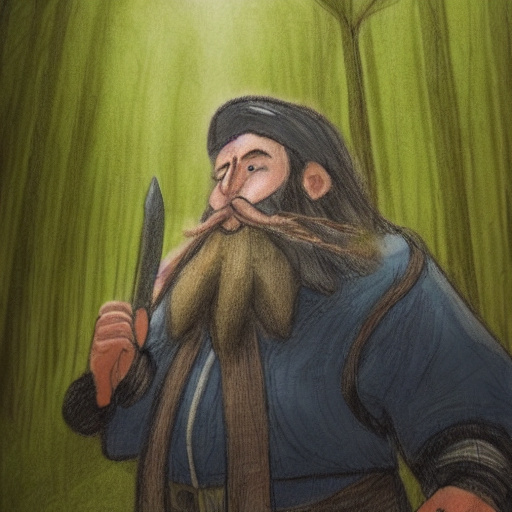}
    }
    \caption{Style transfer results created with PtP \cite{hertz2022prompt} and various inversion methods. \textbf{Eta Inversion (3)} with a larger $\eta$ function improves style transfer.}
    \label{fig:style_transfer}
    \end{figure*}

%% file: main/discussion.tex
\section{Limitations}

Some image edits yield unrealistic outcomes or insufficient changes, despite preserving the original structure (\cref{fig:bad_images}). Adjusting the seed and $\eta$ function can improve results, but no universal setting works for every edit. Future efforts will focus on automating the optimal $\eta$ selection. Furthermore, existing metrics for evaluating image editing are limited, as none measure both structural similarity with the source image and faithfulness to the target prompt. We propose exploring Multimodal Large Language Models \cite{openai2023gpt4, liu2023visual, ge2023making} for more effective image editing assessment in future research.

\input{figs/fig_fail.tex}

%% file: main/conclusion.tex
\section{Conclusion}
\label{sec:conclusion}

In this paper, we propose a unified framework for diffusion inversion and introduce Eta Inversion, a novel approach for real image editing. Our method incorporates real noise into the editing process by utilizing an optimally designed $\eta$ function within DDIM sampling for faithful image editing. Through detailed comparison and analysis of the role of $\eta$, we demonstrate state-of-the-art performance in real image editing across various metrics, offering both compelling qualitative outcomes and precise editing control.%

%% file: supp_files.tex
\appendix

\centerline{\LARGE{\textbf{Supplementary Materials}}}

\input{appendix/proof}

    \input{figs/fig_prop3}
\input{appendix/experiment_details.tex}
    \input{tables/tab_detail}

    \input{tables/tab_eta_setting}
\input{appendix/search_optimal}
    \input{figs/fig_optimal}

    \input{tables/tab_study1.tex}
    \clearpage
    \input{tables/tab_study3.tex}
    \clearpage
    \input{tables/tab_study2.tex}
    \clearpage
\input{appendix/inversion_methods}

    \input{tables/tab_inversion_1.tex}

\input{appendix/editing_methods}

\input{appendix/metrics}
\input{appendix/additional_results}

    \input{tables/tab_appendix_result}
    \input{figs/fig_graph_all}
    \input{tables/tab_category.tex}
    \input{tables/tab_appendix_gpt}
    \input{tables/tab_appendix_user}
    \input{tables/tab_appendix_ti2i}

\input{tables/tab_appendix_rec}
    \clearpage
    \input{figs/fig_ptp_result}
    \clearpage
    \input{figs/fig_pnp_result}
    \clearpage
    \input{figs/fig_masa_result}
    \clearpage
    \input{figs/fig_etainv3_result}
    \clearpage
    \input{figs/fig_rat_pig}

%% file: appendix/proof.tex
\section{Proofs of Propositions}
\subsection{Proof of Proposition 1}
\begin{proposition*}
Let $\delta_{\eta_t} = \|\boldsymbol{x}_{t-1}^{(s)'}-\mathrm{DDIM}(\boldsymbol{x}_{t}^{(t)'},\boldsymbol{c}^{(t)},\eta_t)\|_2$ be the source-target branch distance at timestep $t$.
If $\delta_{0}$ is small, there exists an $\eta_t>0$ that satisfies $ \mathop{\mathbb{E}}_{\boldsymbol\epsilon_{add}}[\delta_{\eta_t}]> \delta_0$.
\end{proposition*}
\begin{proof}
Given a normally distributed random variable $X \sim \mathcal{N}(\mu,\sigma^{2})$, it is known that the random variable $|X|$ follows a \textit{Folded Normal Distribution} with %
\begin{gather}
\label{eq:folded}
\mathop{\mathbb{E}}[|X|] = \sigma\sqrt{\frac{2}{\pi}}e^{-\mu^2/2\sigma^2}+\mu\mathop{\mathrm{erf}}(\frac{\mu}{\sqrt{2\sigma^2}}),\\
\label{eq:argmin}
\argmin_{\mu}\mathop{\mathbb{E}}[|X|] = 0,
\end{gather}
where $\mathop{\mathrm{erf}}(z) = \frac{2}{\sqrt{\pi}}\int_0^z e^{-t^2}\mathop{\mathrm{d}t}$.
Let $\boldsymbol{x} \in \mathbb{R}^d$ and %
\begin{equation}
\label{eq:define}
\boldsymbol{\mu}_t \coloneqq \frac{(\boldsymbol{x}_{t}^{(t)'}-\sqrt{1-\bar{\alpha}_{t}}\boldsymbol{\epsilon}_{t}^{(t)'})}{\sqrt{{\alpha}_{t}}}+\sqrt{1-\bar{\alpha}_{t-1}-\sigma^{2}_{t}}\boldsymbol{\epsilon}_{t}^{(t)'} -\boldsymbol{x}_{t-1}^{(s)'}.
\end{equation}
As \cref{eq:define} requires $1-\bar{\alpha}_{t-1}-\sigma^{2}_{t} \geq 0$, using the definition of $\sigma_t(\eta_{t})$ we write 
\begin{align}
\sqrt{1-\bar{\alpha}_{t-1}} &\geq \eta_{t}\sqrt{\frac{1-\bar\alpha_{t-1}}{1-\bar\alpha_{t}}}\sqrt{1-\frac{\bar\alpha_{t}}{\bar\alpha_{t-1}}},
\end{align}
resulting in the following condition for $\eta_t$:
\begin{equation}
\label{eq:condition}
\frac{\sqrt{1-\bar\alpha_{t}}}{\sqrt{1-\frac{\bar\alpha_{t}}{\bar\alpha_{t-1}}}} \geq \eta_t \geq 0.
\end{equation}
Assuming $\delta_0$ is sufficiently small with $\sqrt{1-\bar{\alpha}_{t-1}}\sqrt{\frac{2d}{\pi}} > \delta_0$, we show that
\begin{align}
\mathop{\mathbb{E}}_{\boldsymbol{\epsilon}_{add}}[\delta_{\eta_t}]&=\mathop{\mathbb{E}}_{\boldsymbol{\epsilon}_{add}}[\|\boldsymbol{\mu}_t + \sigma_{t}\boldsymbol{\epsilon}_{\mathrm{add}}\|_2]\\
&\geq\frac{1}{\sqrt{d}}\mathop{\mathbb{E}}_{\boldsymbol{\epsilon}_{add}}[\|\boldsymbol{\mu}_t + \sigma_{t}\boldsymbol{\epsilon}_{\mathrm{add}}\|_1] &&\text{(Cauchy–Schwarz inequality)} \\
&=\frac{1}{\sqrt{d}}\mathop{\mathbb{E}}[\Sigma_{i=1}^d|\mu_{t,i} + \sigma_{t}\epsilon_{\mathrm{add},i}|] &&\text{(Sum of $i^{\textrm{th}}$ dimension's)}\\
&=\frac{1}{\sqrt{d}}\sum_{i=1}^d\mathop{\mathbb{E}}[|\mu_{t,i} + \sigma_{t}\epsilon_{\mathrm{add},i}|] \\
&=\frac{1}{\sqrt{d}}\sum_{i=1}^d\mathop{\mathbb{E}}[|X_i|] && (X_i \sim \mathcal{N}(\mu_{t,i},\sigma_{t}^2)) \\
&\geq\frac{1}{\sqrt{d}}\sum_{i=1}^d\mathop{\mathbb{E}}[|X_i|]_{\mu_{t,i}=0} && (\cref{eq:argmin}) \\
&=\frac{1}{\sqrt{d}}\sum_{i=1}^d\sigma_{t}\sqrt{\frac{2}{\pi}} && (\cref{eq:folded}) \\
&=\sigma_t\sqrt{\frac{2d}{\pi}} \\
&=\eta_{t}\sqrt{\frac{1-\bar\alpha_{t-1}}{1-\bar\alpha_{t}}}\sqrt{1-\frac{\bar\alpha_{t}}{\bar\alpha_{t-1}}}\sqrt{\frac{2d}{\pi}}.
\end{align}
Thus, our proposition $\mathop{\mathbb{E}}_{\epsilon_{add}}[\delta_{\eta_t}]>\delta_0$ holds if we choose an $\eta_t$, which satisfies
\begin{equation}
\frac{\sqrt{1-\bar\alpha_{t}}}{\sqrt{1-\frac{\bar\alpha_{t}}{\bar\alpha_{t-1}}}} \geq \eta_t > \frac{\delta_0}{\sqrt{\frac{1-\bar\alpha_{t-1}}{1-\bar\alpha_{t}}}\sqrt{1-\frac{\bar\alpha_{t}}{\bar\alpha_{t-1}}}\sqrt{\frac{2d}{\pi}}}.
\end{equation}

\null\nobreak\hfill\ensuremath{\square}
\end{proof}
\subsection{Proof of Proposition 2}
\begin{assumption}
\label{assumption1}
We rewrite the following assumptions from prior works \cite{song2020score,lu2022maximum,nie2023blessing} using our notation for completeness.
\begin{enumerate}
\item $q_0(\boldsymbol{x}) \in \mathcal{C}^3$ and $\mathop{\mathbb{E}}_{q_0(\boldsymbol{x})}[\|\boldsymbol{x}\|_2^2] < \infty$.
\item $\forall t \in [0,T]$ : $\boldsymbol{f}_t(\cdot) \in \mathcal{C}^2$. And $\exists C > 0$, $\forall \boldsymbol{x} \in \mathbb{R}^d$, $t \in [0,T]$ : $\|\boldsymbol{f}_t(\boldsymbol{x})\|_2 \leq C(1+\|\boldsymbol{x}\|_2)$.
\item $\exists C > 0$, $\forall \boldsymbol{x}, \boldsymbol{y} \in \mathbb{R}^d$ : $\|\boldsymbol{f}_t(\boldsymbol{x}) - \boldsymbol{f}_t(\boldsymbol{y})\|_2 \leq C\|\boldsymbol{x}-\boldsymbol{y}\|_2$.
\item $g \in \mathcal{C}$ and $\forall t \in [0,T]$, $|g(t)|>0$.
\item Open bounded set $\forall O$, $\int_0^T \int_O \|q_t(\boldsymbol{x})\|_2^2 + d \cdot g(t)^2\|\nabla_{\boldsymbol{x}}\log q_{t}(\boldsymbol{x})\|_2^2 \mathop{\mathrm{d}\boldsymbol{x}} \mathop{\mathrm{d}t} < \infty$.
\item $\exists C > 0$, $\forall \boldsymbol{x} \in \mathbb{R}^d$, $t \in [0,T]$ : $\|\nabla_{\boldsymbol{x}}\log q_{t}(\boldsymbol{x})\|_2^2 \leq C(1+\|\boldsymbol{x}\|_2)$.
\item $\exists C > 0$, $\forall \boldsymbol{x}, \boldsymbol{y} \in \mathbb{R}^d$ : $\|\nabla_{\boldsymbol{x}}\log q_{t}(\boldsymbol{x})-\nabla_{\boldsymbol{y}}\log q_{t}(\boldsymbol{y})\|_2 \leq C \|\boldsymbol{x}-\boldsymbol{y}\|_2$.
\item $\exists C > 0$, $\forall \boldsymbol{x} \in \mathbb{R}^d$, $t \in [0,T]$ : $\|\boldsymbol{s}_{t,\theta}(\boldsymbol{x})\|_2 \leq C(1+\|\boldsymbol{x}\|_2)$.
\item $\exists C > 0$, $\forall \boldsymbol{x}, \boldsymbol{y} \in \mathbb{R}^d$ : $\|\boldsymbol{s}_{t,\theta}(\boldsymbol{x})-\boldsymbol{s}_{t,\theta}(\boldsymbol{y})\|_2 \leq C \|\boldsymbol{x}-\boldsymbol{y}\|_2$.
\item Novikov's condition: $\mathop{\mathbb{E}}[\exp{(\frac{1}{2}\int_0^T \|\nabla_{\boldsymbol{x}}\log q_{t}(\boldsymbol{x}) - \boldsymbol{s}_{t,\theta}(\boldsymbol{x})\|_2^2)}] < \infty$.
\item $\forall t \in [0,T]$, $\exists k > 0$ : $q_t(\boldsymbol{x}) = \mathcal{O}(e^{-\|\boldsymbol{x}\|^k_2})$, $p_{t,\eta_t}(\boldsymbol{x}) = \mathcal{O}(e^{-\|\boldsymbol{x}\|^k_2})$ as $\|\boldsymbol{x}\|_2 \rightarrow \infty$.
\end{enumerate}
\end{assumption}
\begin{proposition*}
Under Assumption \ref{assumption1}, \cref{eq:prop2_new} holds, wherein $D_{\mathrm{Fisher}}$ denotes the Fisher Divergence.
\end{proposition*}
\begin{equation}
\label{eq:prop2_new}
D_{\mathrm{KL}}(p_0^{(t)'} \parallel p_0^{(t)}) = D_{\mathrm{KL}}(p_T^{(t)'} \parallel p_T^{(t)}) - \int_0^T \eta_{t}^2 g_t^2 D_{\mathrm{Fisher}}(p_t^{(t)'} \parallel p_t^{(t)})\mathop{\mathrm{d}t}
\end{equation}
\begin{proof}
Given the general form of the SDE (\cref{eq:sde_general}), \cref{eq:pode1} and \cref{eq:pode2} are the equations for the probability flow ODE, and \cref{eq:fp} is the link between the Fokker–Planck equation\cite{song2020score} and the probability flow.
\begin{gather}
\label{eq:sde_general}
\mathop{\mathrm{d}\boldsymbol{x}} = \boldsymbol{f}_t(\boldsymbol{x})\mathop{\mathrm{d}t} + \boldsymbol{G}_t(\boldsymbol{x})\mathop{\mathrm{d}\boldsymbol{w}} \\
\label{eq:pode1}
\mathop{\mathrm{d}\boldsymbol{x}} = \tilde{\boldsymbol{f}}_t(\boldsymbol{x})\mathop{\mathrm{d}t} \\
\label{eq:pode2}
\tilde{\boldsymbol{f}}_t(\boldsymbol{x}) \leftarrow \boldsymbol{f}_t(\boldsymbol{x}) - \frac{1}{2}\nabla\cdot[\boldsymbol{G}_t(\boldsymbol{x})\boldsymbol{G}_t(\boldsymbol{x})^{\intercal}]- \frac{1}{2}\boldsymbol{G}_t(\boldsymbol{x})\boldsymbol{G}_t(\boldsymbol{x})^{\intercal}\nabla_{\boldsymbol{x}}\log p_{t}(\boldsymbol{x}) \\
\label{eq:fp}
\frac{\partial}{\partial t}p_t(\boldsymbol{x}) = - \nabla_{\boldsymbol{x}}\cdot[\tilde{\boldsymbol{f}}_t(\boldsymbol{x})p_t(\boldsymbol{x})]
\end{gather}
Recall the forward path (\cref{eq:sde_forward2}) and the extended backward path (\cref{eq:flow_extend2}) of diffusion models (score-based models). In practice, we use \cref{eq:flow_extend_p2} as the backward path to sample data with the score estimation network $\boldsymbol{s}_{t,\theta}(\boldsymbol{x})$ instead of the unknown ground truth $\nabla_{\boldsymbol{x}}\log q_{t}(\boldsymbol{x})$.
\begin{align}
\label{eq:sde_forward2}
&\mathop{\mathrm{d}\boldsymbol{x}} = f_{t}\boldsymbol{x}\mathop{\mathrm{d}t} + g_{t}\mathop{\mathrm{d}\boldsymbol{w}} \quad \biggl( f_{t} = \frac{1}{2}\frac{\mathop{\mathrm{d}\log \alpha_{t}}}{\mathop{\mathrm{d}t}}, g_{t} = \sqrt{-\frac{\mathop{\mathrm{d}\log \alpha_{t}}}{\mathop{\mathrm{d}t}}}\biggr)\\
\label{eq:flow_extend2}
&\mathop{\mathrm{d}\boldsymbol{x}} = [f_{t}\boldsymbol{x} - \frac{1+\eta_t^2}{2}g^{2}_{t}\nabla_{\boldsymbol{x}}\log q_{t}(\boldsymbol{x})]\mathop{\mathrm{d}t} + \eta_t g_{t}\mathop{\mathrm{d}\boldsymbol{\bar{w}}}\\
\label{eq:flow_extend_p2}
&\mathop{\mathrm{d}\boldsymbol{x}} = \underbrace{[f_{t}\boldsymbol{x} - \frac{1+\eta_t^2}{2}g^{2}_{t}\boldsymbol{s}_{t,\theta}(\boldsymbol{x})]}_{\boldsymbol{A}_t(\boldsymbol{x})}\mathop{\mathrm{d}t} + \eta_t g_{t}\mathop{\mathrm{d}\boldsymbol{\bar{w}}}
\end{align}
Using $\boldsymbol{f}_t(\boldsymbol{x}) \leftarrow \boldsymbol{A}_t(\boldsymbol{x})$ and $\boldsymbol{G}_t(\boldsymbol{x}) \leftarrow \eta_t g_{t}$, we rewrite the probability flow ODE~(\cref{eq:pode1}, \cref{eq:pode2}) as
\begin{align}
\label{eq:pode_A1}
\mathop{\mathrm{d}\boldsymbol{x}} &= \tilde{\boldsymbol{A}}_t(\boldsymbol{x})\mathop{\mathrm{d}t},\\
\label{eq:pode_A2}
\tilde{\boldsymbol{A}}_t(\boldsymbol{x}) &= \boldsymbol{A}_t(\boldsymbol{x}) + \frac{1}{2}\eta_t^{2}g_t^{2}\nabla_{\boldsymbol{x}}\log p_{t}(\boldsymbol{x}),
\end{align}
and the Fokker-Plank equation~(\cref{eq:fp}) as 
\begin{align}
\label{eq:fp_A}
\frac{\partial}{\partial t}p_t(\boldsymbol{x}) &= - \nabla_{\boldsymbol{x}}\cdot[\tilde{\boldsymbol{A}}_t(\boldsymbol{x})p_t(\boldsymbol{x})].
\end{align}

For real image editing (the backward path of the target), we are interested in the unknown true inverted marginal distribution $p_T^{(t)}$ and the inaccurately inverted marginal distribution $p_T^{(t)'}$, which is obtained by diffusion inversion. Since both distributions use the same backward path (\cref{eq:flow_extend_p2}), we can formulate the probability flow ODE for both as
\begin{align}
\boldsymbol{A}_t^{(t)}(\boldsymbol{x}) &= f_{t}\boldsymbol{x} - \frac{1+\eta_t^2}{2}g^{2}_{t}\boldsymbol{s}_{t,\theta}^{(t)}(\boldsymbol{x}),\\
\tilde{\boldsymbol{A}}_t^{(t)}(\boldsymbol{x}) &= \boldsymbol{A}_t^{(t)}(\boldsymbol{x}) + \frac{1}{2}\eta_t^{2}g_t^{2}\nabla_{\boldsymbol{x}}\log p_{t}^{(t)}(\boldsymbol{x}),\\
\tilde{\boldsymbol{A}}_t^{(t)'}(\boldsymbol{x}) &= \boldsymbol{A}_t^{(t)}(\boldsymbol{x}) + \frac{1}{2}\eta_t^{2}g_t^{2}\nabla_{\boldsymbol{x}}\log p_{t}^{(t)'}(\boldsymbol{x}),
\end{align} 
and the Fokker-Plank equation for both as
\begin{align}
\frac{\partial}{\partial t}p_t^{(t)}(\boldsymbol{x}) &= - \nabla_{\boldsymbol{x}}\cdot[\tilde{\boldsymbol{A}}_t^{(t)}(\boldsymbol{x})p_t^{(t)}(\boldsymbol{x})],\\
\frac{\partial}{\partial t}p_t^{(t)'}(\boldsymbol{x}) &= - \nabla_{\boldsymbol{x}}\cdot[\tilde{\boldsymbol{A}}_t^{(t)'}(\boldsymbol{x})p_t^{(t)'}(\boldsymbol{x})].
\end{align} 
Finally, we show
\begin{align}
&\frac{\partial D_{\mathrm{KL}}(p_t^{(t)'} \parallel p_t^{(t)})}{\partial t} \\
=& \frac{\partial}{\partial t} \int p_t^{(t)'}(\boldsymbol{x}) \log \frac{p_t^{(t)'}(\boldsymbol{x})}{p_t^{(t)}(\boldsymbol{x})} \mathop{\mathrm{d}\boldsymbol{x}} \\
=& \int \frac{\partial}{\partial t}p_t^{(t)'}(\boldsymbol{x}) \log \frac{p_t^{(t)'}(\boldsymbol{x})}{p_t^{(t)}(\boldsymbol{x})} \mathop{\mathrm{d}\boldsymbol{x}} - \int\frac{p_t^{(t)'}(\boldsymbol{x})}{p_t^{(t)}(\boldsymbol{x})}\frac{\partial}{\partial t}p_t^{(t)}(\boldsymbol{x})\mathop{\mathrm{d}\boldsymbol{x}} \\
=& -\int \nabla_{\boldsymbol{x}}\cdot[\tilde{\boldsymbol{A}}_t^{(t)'}(\boldsymbol{x})p_t^{(t)'}(\boldsymbol{x})] \log \frac{p_t^{(t)'}(x)}{p_t^{(t)}(\boldsymbol{x})} \mathop{\mathrm{d}\boldsymbol{x}} \notag\\
&+ \int\frac{p_t^{(t)'}(\boldsymbol{x})}{p_t^{(t)}(\boldsymbol{x})}\nabla_{\boldsymbol{x}}\cdot[\tilde{\boldsymbol{A}}_t^{(t)}(\boldsymbol{x})p_t^{(t)}(\boldsymbol{x})] \mathop{\mathrm{d}\boldsymbol{x}}\\
=& \int [\tilde{\boldsymbol{A}}_t^{(t)'}(\boldsymbol{x})p_t^{(t)'}(\boldsymbol{x})]^{\intercal}\nabla_{\boldsymbol{x}}\log \frac{p_t^{(t)'}(x)}{p_t^{(t)}(\boldsymbol{x})} \mathop{\mathrm{d}\boldsymbol{x}} \notag\\
&- \int[\tilde{\boldsymbol{A}}_t^{(t)}(\boldsymbol{x})p_t^{(t)}(\boldsymbol{x})]^{\intercal}\nabla_{\boldsymbol{x}}\frac{p_t^{(t)'}(\boldsymbol{x})}{p_t^{(t)}(\boldsymbol{x})} \mathop{\mathrm{d}\boldsymbol{x}} \qquad\qquad\textrm{(Assumption \ref{assumption1})}\\
=& \int p_t^{(t)'}(\boldsymbol{x})[\tilde{\boldsymbol{A}}_t^{(t)'}(\boldsymbol{x})^{\intercal}-\tilde{\boldsymbol{A}}_t^{(t)}(\boldsymbol{x})^{\intercal}][\nabla_{\boldsymbol{x}}\log p_t^{(t)'}(\boldsymbol{x})- \nabla_{\boldsymbol{x}}\log p_t^{(t)}(\boldsymbol{x})]\mathop{\mathrm{d}\boldsymbol{x}}\\
=& \frac{1}{2}\eta_t^2 g_t^2  \int p_t^{(t)'}(\boldsymbol{x})\|\nabla_{\boldsymbol{x}}\log p_t^{(t)'}(\boldsymbol{x})- \nabla_{\boldsymbol{x}}\log p_t^{(t)}(\boldsymbol{x})\|_2^2\mathop{\mathrm{d}\boldsymbol{x}}\\
\label{eq:fisher}
=& \eta_t^2 g_t^2 D_{\mathrm{Fisher}}(p_t^{(t)'} \parallel p_t^{(t)}). \qquad\qquad\qquad\qquad\qquad\textrm{(See \cite{nie2023blessing,lu2022maximum})}
\end{align}
Thus, \cref{eq:prop2_new} holds by integrating \cref{eq:fisher}.

\null\nobreak\hfill\ensuremath{\square}

\end{proof}

\subsection{Proof of Proposition 3}
We omit the superscript $\dashedph^{(t)}$ in this section for simplicity. We express the scale of the score estimation error as $\epsilon$, which we assume is sufficiently small:
\begin{gather}
\boldsymbol{s}_{t,\theta}(\boldsymbol{x}) = \nabla_{\boldsymbol{x}}\log q_t(\boldsymbol{x}) + \epsilon \mathop{\mathrm{Error}}(\boldsymbol{x}),
\end{gather}
with $\mathop{\mathrm{Error}}(x) = \mathcal{O}(1)$. It is known that  $D_{\mathrm{KL}}(p_0 \parallel q_0)$ has order of $\epsilon^2$ as below \cite{chen2023improved,chen2022sampling}:
\begin{gather}
D_{\mathrm{KL}}(p_0 \parallel q_0) = \epsilon^2 L(\eta_t) + \mathcal{O}(\epsilon^3).
\end{gather}

\begin{assumption}
\label{assumption3}
We rewrite the following assumptions from prior work \cite{cao2023exploring} using our notation for completeness.
\begin{enumerate}
\item Without loss of generality (time re-scaling),  $f_{t}= -\frac{1}{2}$, $g_t = 1$. 
\item $\exists c_U \in \mathbb{R}$, $\forall \boldsymbol{x} \in \mathbb{R}^d$ : $-\log p_0(\boldsymbol{x}) - \frac{|\boldsymbol{x}|^2}{2} \geq c_U$.
\item $\forall t \in [0,T]$, $-\log p_t(\boldsymbol{x})$ is strongly convex.
\item $\forall t \in [0,T]$, $ m_t \boldsymbol{I} \preceq \nabla^2(-\log p_t(\boldsymbol{x})) \preceq M_t \boldsymbol{I}$ where $m_t \geq 1$ for $t \in (0,T]$ and $m_0 > 1$.
\end{enumerate}
\end{assumption}
The proof for Proposition 3 is based on two propositions from prior work\cite{cao2023exploring}, which we reformulate to match our notation (\cref{lemma1} and \cref{lemma2}). Under Assumption \ref{assumption3} (1), $\eta_t$ in our notation corresponds to $h_t$ from \cite{cao2023exploring}. For the rest of this section, $\delta$ represents the Dirac delta function.

\begin{lemma}[Proposition 3.4 of \cite{cao2023exploring}]
\label{lemma1}
Suppose the score estimation function $\boldsymbol{s}_{t,\theta}(x)$ only undergoes perturbation at some fixed arbitrary timestep $t_a \in (0,T]$ with $\mathop{\mathrm{Error}}(\boldsymbol{x})=\delta_{t-t_a}E(\boldsymbol{x})$. Let $\eta_t = \eta$ ($\eta_t$ is constant for all $t$). Under Assumption \ref{assumption3}, and if $\eta$ is large enough, there exists an upper bound $L_{\mathrm{ub}}(\eta) \geq L(\eta)$, which is an exponentially decreasing function converging to 0. Thus, there exists an $\eta$ with $L(\eta) < \min(\epsilon, L(0))$.  %
\end{lemma}
\begin{lemma}[Proposition 3.5 of \cite{cao2023exploring}]
\label{lemma2}
Suppose the score estimation function $\boldsymbol{s}_{t,\theta}(x)$ only undergoes perturbation at some fixed timestep $t_b \ll 1$ near timestep 0 with $\mathop{\mathrm{Error}}(\boldsymbol{x})=\delta_{t-t_b}E(\boldsymbol{x})$. Let $\eta_t = \eta$ ($\eta_t$ is constant for all $t$). Under Assumption \ref{assumption3}, and if $\eta$ is large enough, we have $L(0) \ll L(\eta)$.
\end{lemma}

\begin{proposition*}
Under Assumption \ref{assumption3}, if the score estimation function $\boldsymbol{s}_{t,\theta}(x)$ undergoes perturbations only near the timestep $T$ and near the timestep $0$, there exists a timestep $T_a$ and a timestep $T_b$, along with a large constant $\eta_{\mathrm{const}}>0$, such that $D_{\mathrm{KL}}(p_{0,\eta_t} \parallel q_0)$ becomes reduced when employing $\eta_t$ as \cref{eq:cases_new}, in comparison to $\eta_t=0$ for all $t$ or $\eta_t=\eta_{\mathrm{const}}$ for all $t$.
\end{proposition*}
\begin{equation}
\label{eq:cases_new}
\eta_t = \begin{cases}
\eta_{\mathrm{const}} & \textrm{if } T \geq t \geq T_a\\
\eta_{\mathrm{const}}(t-T_b)/(T_a-T_b) & \textrm{if } T_a > t \geq T_b\\
0 & \textrm{if } T_b > t \geq 0
\end{cases}
\end{equation}
\begin{proof}
We define the following cases for three different possible $\eta_t$ functions (\cref{fig:prop3}):
\begin{itemize}
\item Case 1: \cref{eq:cases_new},
\item Case 2: $\eta_t=\eta_{\mathrm{const}}$ for all $t$,
\item Case 3: $\eta_t=0$ for all $t$.
\end{itemize}

We write $\mathop{\mathrm{Error}}(\boldsymbol{x})$ as below assuming perturbations only at $t_a$ (near timestep $T$) and at $t_b$ (near timestep 0):
\begin{gather}
\mathop{\mathrm{Error}}(\boldsymbol{x})=(\delta_{t-t_a}+\delta_{t-t_b})E(\boldsymbol{x}).
\end{gather}

Let $T_a$ be an arbitrary timestep with $t_a > T_a > t_b$. Assume we perform a shorter diffusion backward pass from $T$ to $T_a$ by setting the final diffusion step to $T_a$ instead of 0 and measure its sample quality with $D_{\mathrm{KL}}(p_{T_a} \parallel q_{T_a})$. Since we now only operate in the time interval $[T_a,T]$ , we can ignore the perturbation at $t_b$ and rewrite the error function as $\mathop{\mathrm{Error}}(\boldsymbol{x})=\delta_{t-t_a}E(\boldsymbol{x})$. By applying \cref{lemma1} to our new diffusion pass in $[T_a,T]$, without loss of generality, there exists a constant $\eta_{\mathrm{const,a}}$ so that
\begin{gather}
\label{eq:min}
L(\eta_{\mathrm{const,a}}) < \min (\epsilon, L(T_a)) \leq L(T_a),\\
\label{eq:approx}
D_{\mathrm{KL}}(p_{T_a,\eta_{\mathrm{const,a}}} \parallel q_{T_a}) = \epsilon^2 L(\eta_t) + \mathcal{O}(\epsilon^3) = \mathcal{O}(\epsilon^3) \approx 0.
\end{gather}

Similarly, let $T_b$ be an arbitrary timestep with $T_a > T_b > t_b$ and assume we perform a shorter diffusion backward pass starting from $T_b$ and ending at 0. From \cref{lemma2}, there exists a constant $\eta_{\mathrm{const,b}}$ that satisfies $L(0) \ll L(\eta_{\mathrm{const,b}})$.

Let $\eta_{\mathrm{const}}= \max (\eta_{\mathrm{const,a}}, \eta_{\mathrm{const,b}})$.  
We compare case 1 to case 2 and case 3 and show that case 1 has the best sampling quality among those three.

\nbf{i) Comparison to case 2 ($\eta_t=\eta_{\mathrm{const}}$ for all $t$)}
\begin{enumerate}
\item $T\geq t \geq T_a$: By \cref{lemma1} and \cref{eq:approx}, we have $D_{\mathrm{KL}}(p_{T_a,\eta_t} \parallel q_{T_a}) \approx 0$ for case 1 and $D_{\mathrm{KL}}(p_{T_a,\eta_{\mathrm{const}}} \parallel q_{T_a}) \approx 0$ for case 2.
\item $T_a \geq t \geq T_b$: Since we have $D_{\mathrm{KL}}(p_{T_a} \parallel q_{T_a}) \approx 0$ for both cases and the score function is accurate, the $\eta_t$ function does not affect $p_{T_b,\eta_t}$, thus we have $D_{\mathrm{KL}}(p_{T_b} \parallel q_{T_b}) \approx 0$ for both case 1 and case 2.
\item $T_b \geq t \geq 0$: Since $\eta_t=0$ for $t \leq T_b$ for case 1, case 1's $D_{\mathrm{KL}}(p_{0,\eta_t} \parallel q_0)$ is smaller than case 2's $D_{\mathrm{KL}}(p_{0,\eta_{\mathrm{const}}} \parallel q_0)$ by \cref{lemma2}.
\end{enumerate}

\nbf{ii) Comparison to case 3 ($\eta_t=0$ for all $t$)}
\begin{enumerate}
\item $T \geq t \geq T_a$: By \cref{lemma1}, \cref{eq:min} and \cref{eq:approx}, case 1 satisfies $D_{\mathrm{KL}}(p_{T_a,\eta_t} \parallel q_{T_a}) \approx 0$ while we have $D_{\mathrm{KL}}(p_{T_a,0} \parallel q_{T_a}) > D_{\mathrm{KL}}(p_{T_a,\eta_t} \parallel q_{T_a})$ for case 3.
\item $T_a \geq t \geq T_b$: Since the score function is accurate and $D_{\mathrm{KL}}(p_{T_a,\eta_t} \parallel q_{T_a}) \approx 0$ for case 1, $D_{\mathrm{KL}}(p_{T_b,\eta_t} \parallel q_{T_b}) \approx 0$ holds while we have $D_{\mathrm{KL}}(p_{T_a,0} \parallel q_{T_a}) > D_{\mathrm{KL}}(p_{T_a,\eta_t} \parallel q_{T_a})$  and therefore $D_{\mathrm{KL}}(p_{T_b,0} \parallel q_{T_b}) > D_{\mathrm{KL}}(p_{T_b,\eta_t} \parallel q_{T_b})$ for case 3.
\item $T_b \geq t \geq 0$: %
For case 1 we have $D_{\mathrm{KL}}(p_{T_b,\eta_t} \parallel q_{T_b}) \approx 0$ and for case 3 we have $D_{\mathrm{KL}}(p_{T_b,0} \parallel q_{T_b}) > D_{\mathrm{KL}}(p_{T_b,\eta_t} \parallel q_{T_b})$. Since both case 1 and case 3 follow the same ODE ($\eta_t=0$ for $t \leq T_b$), case 1's $D_{\mathrm{KL}}(p_{0,\eta_t} \parallel q_0)$ is smaller than case 3's $D_{\mathrm{KL}}(p_{0,0} \parallel q_0)$.
\end{enumerate}
Following \textbf{i)} and \textbf{ii)}, case 1 has the best sample quality, since its $D_{\mathrm{KL}}(p_{0,\eta_t} \parallel q_0)$ is smaller than $D_{\mathrm{KL}}(p_{0,\eta_{\mathrm{const}}} \parallel q_0)$ of case 2 and $D_{\mathrm{KL}}(p_{0,0} \parallel q_0)$ of case 3.

\null\nobreak\hfill\ensuremath{\square}

\end{proof}

%% file: figs/fig_prop3.tex
\begin{figure}[h]
    \centering
        \resizebox{0.45\textwidth}{!}{%
\begin{tikzpicture}
\begin{axis}[
    every axis y label/.style={at={(current axis.west)},left=5mm},
    axis lines = left,
    xlabel = timestep $t$,
    xtick={0,0.1,0.2,0.8,0.9,1},
    xticklabels={$T$,$t_a$,$T_a$,$T_b$,$t_b$,$0$},
    ytick={0,0.6},
    yticklabels={0,$\eta_{\mathrm{const}}$},
    xmin=0, xmax=1.0,
    ymin=0, ymax=0.8,
    legend pos=south west,
]
\addplot[color=blue,domain=0:1]{0.6}
node[above,sloped,pos=0.5]{case 2 ($\eta_t = \eta_{\mathrm{const}}$)};
\addplot[color=green,domain=0:1]{0}
node[above,sloped,pos=0.5]{case 3 ($\eta_t = 0$)};
\addplot[color=red,domain=0.2:0.8]{0.8-x}
node[above,sloped,pos=0.5]{case 1 \cref{eq:cases_new}};
\addplot[color=red,domain=0:0.2]{0.6};
\addplot[color=red,domain=0.8:1]{0};
\addplot[color=black,domain=0:1,->,dashed] coordinates {(0.1,0)(0.1,0.4)}
node[above,sloped,pos=0.5]{perturbation};
\addplot[color=black,domain=0:1,->,dashed] coordinates {(0.9,0)(0.9,0.4)}
node[above,sloped,pos=0.5]{perturbation};

\end{axis}
\end{tikzpicture}
}
    \caption{Proposition 3. We assume that the score estimation model is accurate and only has two perturbations at $t_a$ and $t_b$ ($t_b$ close to 0). We show that the $\eta$ function of case 1 provides better sample quality than case 2 and case 3.}
    \label{fig:prop3}
    \end{figure}
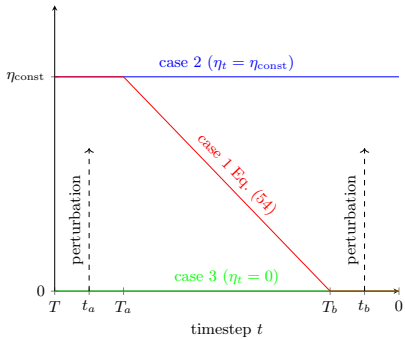

%% file: appendix/experiment_details.tex
\section{Experimental Details}

We provide our hyperparameters for diffusion inversion and real image editing in \cref{tab:inv_details} and \cref{tab:edit_details}. In general, all hyperparameters follow the official code implementation of the respective method. Additionally, \cref{tab:metrics_detail} shows which backbone we used for each metric. \cref{fig:eta_setting} visualizes the $\eta$ function for our three proposed Eta Inversion configurations. %
For EtaInv (1) and (2), we set the cross-attention map threshold to 0.2 and use a sampling count of \(n = 10\). For EtaInv (3), we do not use region-dependent \(\eta\) and use a sampling count of \(n = 1\).

%% file: tables/tab_detail.tex
\begin{table}[h]
    \caption{Experimental details.}
\centering
\begin{subtable}{0.48\linewidth}

    \caption{Inversion hyperparameters. In general, parameter values follow the official implementation.}
\centering
\resizebox{1\textwidth}{!}{%
        \small
        \centering
        \setlength{\tabcolsep}{4.25pt}
        \begin{tabular}{ll|l}
        \toprule
        Inversion Method & Parameter & Value \\
        \midrule
        \multirow{1}{*}{DDPM Inv. \cite{huberman2023edit}} & \footnotesize{\texttt{skip}} & 18 \\
\midrule
\multirow{3}{*}{EDICT \cite{wallace2023edict}} & \footnotesize{\texttt{init\char`_image\char`_strength}} & 1.0 \\
 & \footnotesize{\texttt{leapfrog\char`_steps}} & \footnotesize{\texttt{True}} \\
 & \footnotesize{\texttt{mix\char`_weight}} & 0.93 \\
\midrule
\multirow{2}{*}{Null-text Inv. \cite{mokady2023null}} & \footnotesize{\texttt{early\char`_stop\char`_epsilon}} & 1e-05 \\
 & \footnotesize{\texttt{num\char`_inner\char`_steps}} & 10 \\
\midrule
\multirow{5}{*}{ProxNPI \cite{han2023improving}} & \footnotesize{\texttt{dilate\char`_mask}} & 1 \\
 & \footnotesize{\texttt{prox}} & \footnotesize{\texttt{l0}} \\
 & \footnotesize{\texttt{quantile}} & 0.7 \\
 & \footnotesize{\texttt{recon\char`_lr}} & 1 \\
 & \footnotesize{\texttt{recon\char`_t}} & 400 \\
        \bottomrule
    \end{tabular}
    }
    \label{tab:inv_details}
    
\end{subtable}
\hfill
\begin{subtable}{0.48\linewidth}

\caption{Editing hyperparameters. In general, parameter values follow the official implementation.}
\centering
        \resizebox{1\linewidth}{!}{
        \small
        \centering
        \setlength{\tabcolsep}{4.25pt}

    \begin{tabular}{ll|l}
        \toprule
        Editing Method & Parameter & Value \\
        \midrule
        \multirow{3}{*}{PtP \cite{hertz2022prompt}} & \footnotesize{\texttt{cross\char`_replace\char`_steps}} & 0.4 \\
 & \footnotesize{\texttt{self\char`_replace\char`_steps}} & 0.6 \\
 &\footnotesize{\texttt{equilizer\char`_params\char`_values}} & 2.0 \\\midrule
\multirow{2}{*}{PnP \cite{tumanyan2023plug}} & \footnotesize{\texttt{pnp\char`_f\char`_t}} & 0.8 \\
 & \footnotesize{\texttt{pnp\char`_attn\char`_t}} & 0.5 \\\midrule
\multirow{2}{*}{MasaCtrl \cite{cao2023masactrl}} & \footnotesize{\texttt{step}} & 4 \\
 & \footnotesize{\texttt{layer}} & 10 \\\bottomrule
    \end{tabular}
    }
    \label{tab:edit_details}
\end{subtable}
\newline
\smallskip
\newline
\begin{subtable}{\linewidth}
\caption{Backbone models for metric computation.}
\centering
        \resizebox{0.5\linewidth}{!}{%
        \small
        \centering
        \setlength{\tabcolsep}{4.25pt}
    
    \begin{tabular}{p{5cm}|p{2cm}}
        \toprule
        Metric & Backbone \\
        \midrule
        CLIP similarity \cite{radford2021learning} & ViT-B16 \cite{dosovitskiy2021image} \\
        DINO structural similarity \cite{caron2021emerging} & ViT-B8 \cite{dosovitskiy2021image} \\
        Perceptual Similarity (LPIPS) \cite{zhang2018perceptual} & AlexNet \cite{alexnet} \\
        \bottomrule
    \end{tabular}
    }
    \label{tab:metrics_detail}
\end{subtable}

\end{table}

%% file: tables/tab_eta_setting.tex
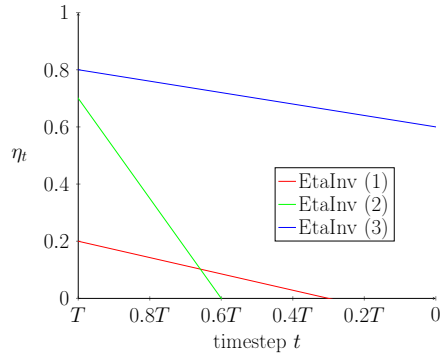
\begin{figure}
\centering
\resizebox{0.5\linewidth}{!}{
{\scalefont{2.2}
    \begin{tikzpicture}
\begin{axis}[
    every axis y label/.style={at={(current axis.west)},left=15mm},
    every axis x label/.style={at={(current axis.south)},below=10mm},
    legend style={at={(0.55,0.2)},anchor=south west},
    axis lines = left,
    xlabel = {timestep $t$},
    ylabel = {$\eta_t$},
    xtick={0,0.2,0.4,0.6,0.8,1},
    xticklabels={$T$,$0.8T$,$0.6T$,$0.4T$,$0.2T$,$0$},
    ytick={0,0.2,0.4,0.6,0.8,1.0},
    xmin=0, xmax=1,
    ymin=0, ymax=1,
    width=\textwidth,
    height=0.8\linewidth,
    scale only axis=true
]
\addplot[color=red]{0.28571428571428575 * ((1 - x) - 0.3)};
\addlegendentry{EtaInv (1)}
\addplot[color=green]{1.7499999999999998 * ((1 - x) - 0.6)};
\addlegendentry{EtaInv (2)}
\addplot[color=blue]{0.20000000000000007 * ((1 - x) - 0.0) + 0.6};
\addlegendentry{EtaInv (3)}

\end{axis}
\end{tikzpicture}
}}
\caption{
    $\eta$ function for our three different EtaInv settings. \textbf{EtaInv (1)} uses smaller $\eta$ favoring structural similarity, \textbf{EtaInv (2)} uses larger $\eta$ favoring prompt alignment, and \textbf{EtaInv (3)} uses very large $\eta$ even at later steps, optimized for style transfer.
}
\label{fig:eta_setting}

\end{figure}

%% file: appendix/search_optimal.tex
\section{Searching the Optimal Eta Function}
In this section, we present several hyperparameter study results on how we searched the optimal $\eta$ function for EtaInv (2). We initialize all hyperparameters to EtaInv (2) by default. Tests are performed using PyTorch \cite{paszke2019pytorch} on an NVIDIA V100 32GB GPU in 32-bit precision. \cref{fig:optimal} shows an overview of our experiments.

\subsection{Sign of Slope $\frac{\mathop{\mathrm{d}\eta_{(T-t)}}}{\mathop{\mathrm{d}t}}$}
We first formulate $\eta$ as a linear function for simplicity and explore how the slope $\frac{\mathop{\mathrm{d}\eta_{(T-t)}}}{\mathop{\mathrm{d}t}}$ of the graph affects the editing performance. \cref{fig:slope} displays the tested $\eta$ functions and \cref{tab:slope} shows the metric values for each function. The results demonstrate that $\frac{\mathop{\mathrm{d}\eta_{(T-t)}}}{\mathop{\mathrm{d}t}}<0$ shows better text-image alignment performance and a better editing effect than $\frac{\mathop{\mathrm{d}\eta_{(T-t)}}}{\mathop{\mathrm{d}t}}\geq0$, which aligns with our theoretical findings. Therefore, we use a decreasing slope with $\frac{\mathop{\mathrm{d}\eta_{(T-t)}}}{\mathop{\mathrm{d}t}} <0$ for further experiments.

\subsection{Optimal $t, \eta-$intercepts}
Next, we analyze how different linear $\eta$ functions affect performance. We define several intercepts on the time axis ($0.3, 0.4, 0.5, 0.6, 0.7$) and on the $\eta$ axis (0.6, 0.7, 0.8, 0.9, 1.0) and linearly interpolate between two intercepts as displayed in \cref{fig:intercepts}. \cref{tab:intercept} shows that a larger $\eta$ and a smaller $t$ (corresponding to applying noise even at later timesteps) improve text-image alignment while sacrificing structural similarity. EtaInv (2) with ($\eta\textrm{-intercept}=0.7, t\textrm{-intercept}=0.6$) provides a good balance for both.

\subsection{Non-zero Concavity $\frac{\mathop{\mathrm{d}^2\eta_{(T-t)}}}{\mathop{\mathrm{d}t^2}}$}
Furthermore, we perform several grid search experiments with non-linear $\eta$ functions by introducing an exponent $p$ (1/3, 1/2, 1, 2, 3) resulting in a concave (if $p < 1$) or convex (if $p > 1$) $\eta$ function. First, we fix the $t, \eta-$intercept to EtaInv (2) and compute metrics for different exponents in \cref{tab:power}. We can observe that making EtaInv (2) concave improves image alignment since more total noise is injected and that a convex EtaInv (2) achieves better structural similarity since less noise is injected. Second, we provide an extensive grid search over various intercepts and exponents in \cref{tab:intercept_power} where each power shows a similar trade-off for alignment and similarity when altering the $\eta$ and $t$ intercept. Based on these experiments we find that there is no immediate benefit of introducing a non-linear $\eta$ function and decide to fix it to linear for the remaining tests.

\subsection{Sampling Count $n$}
We test several different noise sampling counts $(1, 10, 10^2, 10^3, 10^4)$ in \cref{tab:grid_sample_count} and observe that a larger sampling count improves structural similarity while reducing text-image CLIP similarity. We argue that a larger sample count reduces the randomness in Eta Inversion by finding a noise that better approximates the true source-target branch distance. %

\subsection{Cross-attention Map Source}
There are three different sources for cross-attention maps: (i.) from the forward (inversion) path; (ii.) from the the backward path of the source latent; and (iii.) from the backward path of the target latent. Therefore, we provide results for each cross-attention map source in \cref{tab:grid_attn_src} while additionally including two more tests: GT, which uses the ground-truth foreground-background segmentation map provided by the dataset instead of cross-attention; and Source+Target which combines the backward attention maps from the source and the target branch with a $\max$ operation. We found that averaged attention masks from the forward path (i.) are most accurate and stable, since Eta Inversion injects no noise in the forward path, leading to a balanced trade-off of text-image alignment and structural similarity.

\subsection{Mask Threshold $\mathcal{M}_{th}$}
Finally, \cref{tab:grid_attn_thres} shows text-image alignment and structural similarity metrics for different attention map thresholds. Additionally, Smooth does not threshold the attention map but instead multiplies it to $\eta$, reducing $\eta$ at low attention values, which did not achieve good results. For the threshold experiments, a larger attention threshold reduces the region where $\eta>0$ and noise is injected (see \cref{fig:thres}), consequently showing worse text-image alignment and better structural similarity. We find that the threshold $\mathcal{M}_{th} = 0.2$ achieves the best results.

%% file: figs/fig_optimal.tex
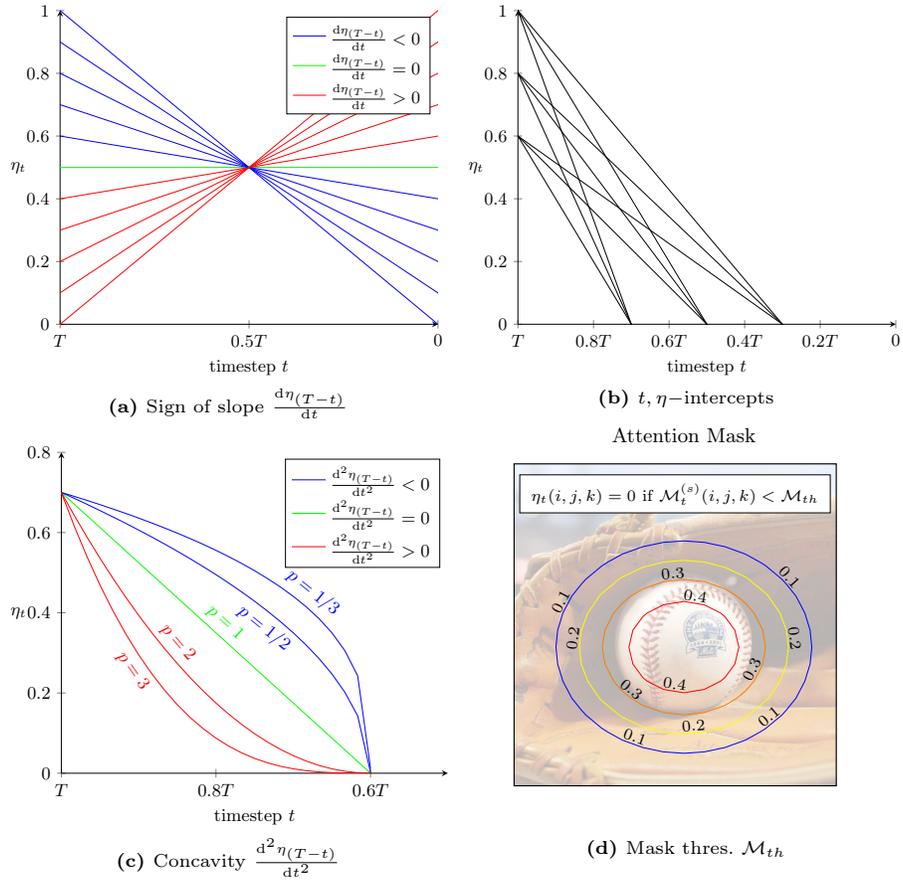
\begin{figure}[h]
    \centering
\begin{subfigure}[t]{0.49\linewidth}
    \centering
        \resizebox{\linewidth}{!}{%
\begin{tikzpicture}
\begin{axis}[
    every axis y label/.style={at={(current axis.west)},left=5mm},
    axis lines = left,
    xlabel = timestep $t$,
    ylabel = $\eta_t$,
    xtick={0,0.5,1},
    xticklabels={$T$,$0.5T$,$0$},
    ytick={0,0.2,0.4,0.6,0.8,1},
    xmin=0, xmax=1,
    ymin=0, ymax=1,
]
\addplot[color=blue]{1.0 * ((1 - x) - 0) ^ 1 + 0.0};
\addlegendentry{$\frac{\mathop{\mathrm{d}\eta_{(T-t)}}}{\mathop{\mathrm{d}t}} < 0$}
\addplot[color=green]{0.0 * ((1 - x) - 0) ^ 1 + 0.5};
\addlegendentry{$\frac{\mathop{\mathrm{d}\eta_{(T-t)}}}{\mathop{\mathrm{d}t}} = 0$}
\addplot[color=red]{-0.20000000000000018 * ((1 - x) - 0) ^ 1 + 0.6000000000000001};
\addlegendentry{$\frac{\mathop{\mathrm{d}\eta_{(T-t)}}}{\mathop{\mathrm{d}t}} > 0$}
\addplot[color=blue]{0.8 * ((1 - x) - 0) ^ 1 + 0.1};
\addplot[color=blue]{0.6000000000000001 * ((1 - x) - 0) ^ 1 + 0.2};
\addplot[color=blue]{0.3999999999999999 * ((1 - x) - 0) ^ 1 + 0.30000000000000004};
\addplot[color=blue]{0.19999999999999996 * ((1 - x) - 0) ^ 1 + 0.4};
\addplot[color=red]{-0.40000000000000013 * ((1 - x) - 0) ^ 1 + 0.7000000000000001};
\addplot[color=red]{-0.6000000000000001 * ((1 - x) - 0) ^ 1 + 0.8};
\addplot[color=red]{-0.8 * ((1 - x) - 0) ^ 1 + 0.9};
\addplot[color=red]{-1.0 * ((1 - x) - 0) ^ 1 + 1.0};
\end{axis}
\end{tikzpicture}
            }%
    \caption{Sign of slope $\frac{\mathop{\mathrm{d}\eta_{(T-t)}}}{\mathop{\mathrm{d}t}}$}
    \label{fig:slope}
    \end{subfigure}
\begin{subfigure}[t]{0.49\linewidth}
    \centering
        \resizebox{\linewidth}{!}{%
\begin{tikzpicture}
\begin{axis}[
    every axis y label/.style={at={(current axis.west)},left=5mm},
    axis lines = left,
    xlabel = timestep $t$,
    ylabel = $\eta_t$,
    xtick={0,0.2,0.4,0.6,0.8,1},
    xticklabels={$T$,$0.8T$,$0.6T$,$0.4T$,$0.2T$,$0$},
    ytick={0,0.2,0.4,0.6,0.8,1},
    xmin=0, xmax=1,
    ymin=0, ymax=1,
]
\addplot[color=black]{0.8571428571428572 * ((1 - x) - 0.3) ^ 1 + 0};
\addplot[color=black]{1.142857142857143 * ((1 - x) - 0.3) ^ 1 + 0};
\addplot[color=black]{1.4285714285714286 * ((1 - x) - 0.3) ^ 1 + 0};
\addplot[color=black]{1.2 * ((1 - x) - 0.5) ^ 1 + 0};
\addplot[color=black]{1.6 * ((1 - x) - 0.5) ^ 1 + 0};
\addplot[color=black]{2.0 * ((1 - x) - 0.5) ^ 1 + 0};
\addplot[color=black]{1.9999999999999996 * ((1 - x) - 0.7) ^ 1 + 0};
\addplot[color=black]{2.6666666666666665 * ((1 - x) - 0.7) ^ 1 + 0};
\addplot[color=black]{3.333333333333333 * ((1 - x) - 0.7) ^ 1 + 0};

\end{axis}
\end{tikzpicture}
                    }%
        \caption{$t, \eta-$intercepts}
    \label{fig:intercepts}
        \end{subfigure}

\begin{subfigure}[t]{0.49\linewidth}
    \centering
        \resizebox{\linewidth}{!}{%
\begin{tikzpicture}
\begin{axis}[
    every axis y label/.style={at={(current axis.west)},left=5mm},
    axis lines = left,
    xlabel = timestep $t$,
    ylabel = $\eta_t$,
    xtick={0,0.2,0.4,0.6,0.8,1},
    xticklabels={$T$,$0.8T$,$0.6T$,$0.4T$,$0.2T$,$0$},
    ytick={0,0.2,0.4,0.6,0.8,1},
    xmin=0, xmax=0.5,
    ymin=0, ymax=0.8,
]
\addplot[color=blue,domain=0:0.4]{0.9500461658082172 * ((1 - x) - 0.6) ^ 0.3333333333333333 + 0}
node[above,sloped,pos=0.5]{$p=1/3$};
\addlegendentry{$\frac{\mathop{\mathrm{d}^2\eta_{(T-t)}}}{\mathop{\mathrm{d}t^2}} < 0$}
\addplot[color=green,domain=0:0.4]{1.7499999999999998 * ((1 - x) - 0.6) ^ 1 + 0}
node[above,sloped,pos=0.5]{$p=1$};
\addlegendentry{$\frac{\mathop{\mathrm{d}^2\eta_{(T-t)}}}{\mathop{\mathrm{d}t^2}} = 0$}
\addplot[color=red,domain=0:0.4]{4.374999999999999 * ((1 - x) - 0.6) ^ 2 + 0}
node[above,sloped,pos=0.5]{$p=2$};
\addlegendentry{$\frac{\mathop{\mathrm{d}^2\eta_{(T-t)}}}{\mathop{\mathrm{d}t^2}} > 0$}
\addplot[color=blue,domain=0:0.4]{1.1067971810589328 * ((1 - x) - 0.6) ^ 0.5 + 0}
node[below,sloped,pos=0.5]{$p=1/2$};
\addplot[color=red,domain=0:0.4]{10.937499999999996 * ((1 - x) - 0.6) ^ 3 + 0}
node[below,sloped,pos=0.5]{$p=3$};

\end{axis}
\end{tikzpicture}
                    }%
        \caption{Concavity $\frac{\mathop{\mathrm{d}^2\eta_{(T-t)}}}{\mathop{\mathrm{d}t^2}}$}
    \label{fig:concavity}
        \end{subfigure}
\begin{subfigure}[t]{0.49\linewidth}
    \centering
        \resizebox{\linewidth}{!}{%
        \includegraphics{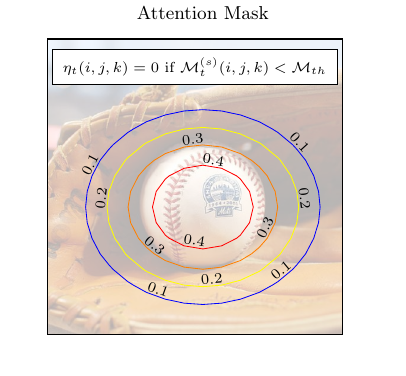}
                    }%
        \caption{Mask thres. $\mathcal{M}_{th}$}
            \label{fig:thres}
        \end{subfigure}
        \caption{Exploring the optimal $\eta$ function.}
    \label{fig:optimal}
    \end{figure}

%% file: tables/tab_study1.tex
\begin{table*}[t]
    \caption{Extensive parameter study for slope, intercept, and concavity of the $\eta$ function evaluated on PIE-Bench with PtP. Hyperparameters are set to EtaInv (2) by default.}
        \centering

\begin{subtable}{\linewidth}
        \caption{Slope $\frac{\mathop{\mathrm{d}\eta_{(T-t)}}}{\mathop{\mathrm{d}t}}$ results. A negative/decreasing slope leads to better text-image alignment. An increasing slope may lead to better similarity but, in practice, fails to edit the image sufficiently since no noise is injected in the early diffusion steps, which is needed to edit high-level features.}
        \centering
        \resizebox{0.8\linewidth}{!}{%
        \small
        \centering
        \setlength{\tabcolsep}{4.25pt}
        \sisetup{table-auto-round}
        \begin{tabular}{@{}c*{5}{S[table-format=2.2,drop-exponent = true,fixed-exponent = -2,exponent-mode = fixed,]}@{}}%
        \toprule
        & \multicolumn{5}{c}{Metric $(\times {10}^{2})$ }  \\ 
        \cmidrule(r){2-6}
         &\multicolumn{2}{c}{Text-Image Alignment (CLIP)} & \multicolumn{3}{c}{Structural Similarity} \\ 
        \cmidrule(r){2-3} \cmidrule(lr){4-6} 
        Slope $\frac{\mathop{\mathrm{d}\eta_{(T-t)}}}{\mathop{\mathrm{d}t}}$ & {text-img $\uparrow$} & {text-cap. $\uparrow$} & {DINOv1 $\downarrow$} & {LPIPS $\downarrow$} & {BG-LPIPS $\downarrow$}  \\ \midrule
        
        $-1.0$ & \cellcolor{mygold}0.31485865360924176 & 0.9571428571428572 & 0.026546814252422855 & 0.297251725279327 & 0.1076310270335359 \\ 
        $-0.8$ & 0.3145242207816669 & \cellcolor{mygold}0.9614285714285714 & 0.024899481675654117 & 0.2849586186744273 & 0.1037061723217734 \\ 
        $-0.6$ & 0.3138454021087715 & 0.9514285714285714 & 0.023619182594120502 & 0.27350145879334636 & 0.09996576243915894 \\ 
        $-0.4$ & 0.31326323264411515 & 0.9471428571428572 & 0.02221010511541473 & 0.2622098144170429 & 0.09645846517199451 \\ 
        $-0.2$ & 0.31327964086617743 & 0.9414285714285714 & 0.021037162935826928 & 0.2508402223219829 & 0.09301226013198695 \\ 
        $\phantom{-}0.0$ & 0.31276073193975856 & 0.95 & 0.02006368477529447 & 0.24024068778380753 & 0.08971116969737133 \\ 
        $\phantom{-}0.2$ & 0.3127705287720476 & 0.9457142857142857 & 0.019216401416342704 & 0.23145483411582454 & 0.08692340581963695 \\ 
        $\phantom{-}0.4$ & 0.3123838809984071 & 0.9414285714285714 & 0.01861261689503278 & 0.22448363900982907 & 0.08462591541184727 \\ 
        $\phantom{-}0.6$ & 0.31230001828500203 & 0.95 & 0.018174748348205216 & 0.21902038262624826 & 0.08284610633903607 \\ 
        $\phantom{-}0.8$ & 0.31195226096681183 & 0.95 & 0.017856704066174903 & 0.2154961786498981 & 0.08167070123911668 \\ 
        $\phantom{-}1.0$ & 0.31155111911041394 & 0.9457142857142857 & \cellcolor{mygold}0.017781001677004887 & \cellcolor{mygold}0.21414580272510647 & \cellcolor{mygold}0.08132684734146876 \\
        
        \bottomrule
        \end{tabular}
        }%
        \label{tab:slope}
        \end{subtable}
        
\bigskip
\medskip

\begin{subtable}{\linewidth}

        \caption{$t, \eta-$intercept results. A larger $\eta$ improves alignment while sacrificing similarity. A larger $t$ reduces the total injected noise and improves similarity while worsening alignment. We chose $\eta=0.7, t=0.6$ for Eta Inversion (2).}
        \centering
        \resizebox{0.8\linewidth}{!}{%
        \small
        \centering
        \setlength{\tabcolsep}{4.25pt}
        \begin{tabular}{@{}l*{10}{S[table-format=2.2]}@{}}%
        \toprule
        & \multicolumn{10}{c}{Metric $(\times {10}^{2})$ }  \\ 
        \cmidrule(r){2-11}
         &\multicolumn{5}{c}{Text-Image Alignment (CLIP) $\uparrow$} & \multicolumn{5}{c}{Structural Similarity (DINOv1) $\downarrow$} \\ 
        \cmidrule(r){2-6} \cmidrule(lr){7-11} 
        \tikz{\node[below left, inner sep=1pt] (def) {$\eta$};%
      \node[above right,inner sep=1pt] (abc) {$t$};%
      \draw (def.north west|-abc.north west) -- (def.south east-|abc.south east);} & 0.3 & 0.4 & 0.5 & \textbf{0.6} & 0.7 & 0.3 & 0.4 & 0.5 & \textbf{0.6} & 0.7  \\ \midrule

        0.6 & 31.22 & 31.19 & 31.19 & 31.14 & 31.12 & 1.76 & 1.71 & 1.67 & 1.60 & \cellcolor{mygold}1.53 \\
        \textbf{0.7} & 31.30 & 31.30 & 31.26 & \textbf{31.25} & 31.18 & 1.93 & 1.88 & 1.82 & \textbf{1.70} & 1.61 \\
        0.8 & 31.32 & 31.35 & 31.34 & 31.26 & 31.24 & 2.11 & 2.03 & 1.96 & 1.83 & 1.71 \\
        0.9 & 31.42 & 31.41 & 31.40 & 31.33 & 31.29 & 2.29 & 2.20 & 2.14 & 1.97 & 1.84 \\
        1.0 & \cellcolor{mygold}31.45 & 31.43 & 31.43 & 31.44 & 31.33 & 2.47 & 2.39 & 2.33 & 2.14 & 1.95 \\
        
        \bottomrule
        \end{tabular}
        }%
        \label{tab:intercept}
        \end{subtable}

\bigskip
\medskip

\begin{subtable}{\linewidth}
        \caption{Concavity results. An exponent $p > 1$ leads to a convex graph, which reduces $\eta$ and thus the total noise injected. Consequently, text-image alignment worsens while similarity improves. A linear $\eta$ function ($p=1$) is sufficient for a good balance of text-image alignment and structural similarity.}
        \centering
        \resizebox{0.8\linewidth}{!}{%
        \small
        \centering
        \setlength{\tabcolsep}{4.25pt}
        \sisetup{table-auto-round}
        \begin{tabular}{@{}c*{5}{S[table-format=2.2,drop-exponent = true,fixed-exponent = -2,exponent-mode = fixed,]}@{}}%
        \toprule
        & \multicolumn{5}{c}{Metric $(\times {10}^{2})$ }  \\ 
        \cmidrule(r){2-6}
         &\multicolumn{2}{c}{Text-Image Alignment (CLIP)} & \multicolumn{3}{c}{Structural Similarity} \\ 
        \cmidrule(r){2-3} \cmidrule(lr){4-6} 
        Exponent $p$ & {text-img $\uparrow$} & {text-cap. $\uparrow$} & {DINOv1 $\downarrow$} & {LPIPS $\downarrow$} & {BG-LPIPS $\downarrow$}  \\ \midrule
        
        $1/3$ & \cellcolor{mygold}0.3129876764331545 & \cellcolor{mygold}0.9542857142857143 & 0.019760140197551145 & 0.2385519772688193 & 0.08886439234384202 \\ 
        $1/2$ & 0.31294646699513706 & 0.9428571428571428 & 0.018914608835974442 & 0.22989207232370973 & 0.0860619368938621 \\ 
        $1$ & 0.31247970819473264 & \cellcolor{mygold}0.9542857142857143 & 0.017007591246760316 & 0.21136770641963396 & 0.07995662840633096 \\ 
        $2$ & 0.31136887124606544 & 0.95 & 0.015474852629538093 & 0.19341757196134754 & 0.0740910722051506 \\ 
        $3$ & 0.3108385650387832 & 0.9528571428571428 & \cellcolor{mygold}0.01477136943289744 & \cellcolor{mygold}0.18383140230551362 & \cellcolor{mygold}0.07106707164021957 \\
        
        \bottomrule
        \end{tabular}
        }%
        \label{tab:power}
        \end{subtable}
\end{table*}

%% file: tables/tab_study3.tex
\begin{table*}[t]
        \caption{$t, \eta-$intercept results for various concavity/exponents. Every exponent shows a similar trade-off for alignment and similarity, e.g., when increasing $\eta$ and decreasing $t$. We conclude that a linear function ($p=1$) is sufficient for good editing results. Remaining hyperparameters are set to match Eta Inversion (2).}

        \centering
        \resizebox{1\linewidth}{!}{%
        \small
        \centering
        \setlength{\tabcolsep}{4.25pt}
        \begin{tabular}{@{}cl*{10}{S[table-format=2.2]}@{}}%
        \toprule
        & & \multicolumn{10}{c}{Metric $(\times {10}^{2})$ }  \\ 
        \cmidrule(r){3-12}
         & & \multicolumn{5}{c}{Text-Image Alignment (CLIP) $\uparrow$} & \multicolumn{5}{c}{Structural Similarity (DINOv1) $\downarrow$} \\ 
        \cmidrule(r){3-7} \cmidrule(lr){8-12} 
        Exponent $p$ & \tikz{\node[below left, inner sep=1pt] (def) {$\eta$};%
      \node[above right,inner sep=1pt] (abc) {$t$};%
      \draw (def.north west|-abc.north west) -- (def.south east-|abc.south east);} & 0.3 & 0.4 & 0.5 & 0.6 & 0.7 & 0.3 & 0.4 & 0.5 & 0.6 & 0.7  \\ \midrule
        \multirow{5}{*}{$1/3$}
        & 0.6 & 31.25 & 31.26 & 31.28 & 31.28 & 31.21 & 1.98 & 1.94 & 1.89 & 1.80 & 1.70 \\
        & 0.7 & 31.37 & 31.33 & 31.35 & 31.30 & 31.26 & 2.19 & 2.13 & 2.07 & 1.98 & 1.86 \\
        & 0.8 & 31.39 & 31.42 & 31.40 & 31.37 & 31.31 & 2.40 & 2.34 & 2.28 & 2.17 & 2.01 \\
        & 0.9 & 31.43 & 31.46 & 31.45 & 31.40 & 31.39 & 2.60 & 2.54 & 2.46 & 2.36 & 2.19 \\
        & 1.0 & 31.43 & 31.49 & 31.52 & 31.52 & 31.46 & 2.81 & 2.74 & 2.68 & 2.55 & 2.38 \\
        \midrule
        
        \multirow{5}{*}{$1/2$}
        & 0.6 & 31.25 & 31.28 & 31.27 & 31.22 & 31.18 & 1.92 & 1.88 & 1.81 & 1.73 & 1.63 \\
        & 0.7 & 31.31 & 31.30 & 31.31 & 31.29 & 31.26 & 2.10 & 2.05 & 1.98 & 1.89 & 1.77 \\
        & 0.8 & 31.40 & 31.40 & 31.36 & 31.36 & 31.31 & 2.31 & 2.25 & 2.17 & 2.05 & 1.90 \\
        & 0.9 & 31.45 & 31.46 & 31.40 & 31.39 & 31.38 & 2.50 & 2.44 & 2.35 & 2.25 & 2.06 \\
        & 1.0 & 31.46 & 31.51 & 31.49 & 31.44 & 31.43 & 2.71 & 2.66 & 2.55 & 2.43 & 2.24 \\
        \midrule
        
        \multirow{5}{*}{$1$}
        & 0.6 & 31.22 & 31.19 & 31.19 & 31.14 & 31.12 & 1.76 & 1.71 & 1.67 & 1.60 & 1.53 \\
        & 0.7 & 31.30 & 31.30 & 31.26 & 31.25 & 31.18 & 1.93 & 1.88 & 1.82 & 1.70 & 1.61 \\
        & 0.8 & 31.32 & 31.35 & 31.34 & 31.26 & 31.24 & 2.11 & 2.03 & 1.96 & 1.83 & 1.71 \\
        & 0.9 & 31.42 & 31.41 & 31.40 & 31.33 & 31.29 & 2.29 & 2.20 & 2.14 & 1.97 & 1.84 \\
        & 1.0 & 31.45 & 31.43 & 31.43 & 31.44 & 31.33 & 2.47 & 2.39 & 2.33 & 2.14 & 1.95 \\
        \midrule
        
        \multirow{5}{*}{$2$}
        & 0.6 & 31.15 & 31.13 & 31.11 & 31.09 & 31.05 & 1.60 & 1.57 & 1.53 & 1.48 & 1.44 \\
        & 0.7 & 31.23 & 31.22 & 31.16 & 31.14 & 31.10 & 1.71 & 1.65 & 1.60 & 1.55 & 1.49 \\
        & 0.8 & 31.28 & 31.28 & 31.24 & 31.17 & 31.10 & 1.84 & 1.78 & 1.69 & 1.62 & 1.54 \\
        & 0.9 & 31.34 & 31.30 & 31.28 & 31.24 & 31.14 & 1.99 & 1.90 & 1.82 & 1.72 & 1.60 \\
        & 1.0 & 31.41 & 31.39 & 31.34 & 31.28 & 31.18 & 2.14 & 2.05 & 1.93 & 1.83 & 1.68 \\
        \midrule
        
        \multirow{5}{*}{$3$}
        & 0.6 & 31.11 & 31.10 & 31.08 & 31.05 & 31.04 & 1.53 & 1.50 & 1.46 & 1.43 & 1.39 \\
        & 0.7 & 31.16 & 31.15 & 31.11 & 31.08 & 31.05 & 1.60 & 1.56 & 1.52 & 1.48 & 1.43 \\
        & 0.8 & 31.24 & 31.21 & 31.16 & 31.09 & 31.05 & 1.69 & 1.64 & 1.58 & 1.52 & 1.47 \\
        & 0.9 & 31.28 & 31.26 & 31.20 & 31.14 & 31.07 & 1.81 & 1.74 & 1.67 & 1.58 & 1.51 \\
        & 1.0 & 31.33 & 31.29 & 31.25 & 31.18 & 31.10 & 1.92 & 1.85 & 1.77 & 1.66 & 1.56 \\

        \bottomrule
        \end{tabular}
        }%
        \label{tab:intercept_power}
\end{table*}

%% file: tables/tab_study2.tex
\begin{table*}[t]
    \caption{Extensive parameter study for noise sample count, attention source, and attention threshold evaluated on PIE-Bench with PtP. Hyperparameters are set to EtaInv (2) by default.}
        \centering

\begin{subtable}{\linewidth}
        \caption{Noise sample count $n$ results. Large sample counts generally lead to better similarity, while lower sample counts achieve better prompt alignment.}
        \centering
        \resizebox{0.9\linewidth}{!}{%
        \small
        \centering
        \setlength{\tabcolsep}{4.25pt}
        \sisetup{table-auto-round}
        \begin{tabular}{@{}l*{5}{S[table-format=2.2,drop-exponent = true,fixed-exponent = -2,exponent-mode = fixed,]}@{}}%
        \toprule
        & \multicolumn{5}{c}{Metric $(\times {10}^{2})$ }  \\ 
        \cmidrule(r){2-6}
         &\multicolumn{2}{c}{Text-Image Alignment (CLIP)} & \multicolumn{3}{c}{Structural Similarity} \\ 
        \cmidrule(r){2-3} \cmidrule(lr){4-6} 
        Sample count $n$ & {text-img $\uparrow$} & {text-cap. $\uparrow$} & {DINOv1 $\downarrow$} & {LPIPS $\downarrow$} & {BG-LPIPS $\downarrow$}  \\ \midrule
        
        $1$ & 0.3118507398877825 & 0.9557142857142857 & 0.017150170436860727 & 0.21263937968090713 & 0.08077677391522489 \\ 
        $\mathbf{10}$ & \cellcolor{mygold}0.31247970819473264 & 0.9542857142857143 & 0.017007591246760316 & 0.21136770641963396 & 0.07995662840633096 \\ 
        $10^2$ & 0.31222788829888615 & 0.95 & 0.01727236340554165 & 0.21266182084434798 & 0.0806505930916007 \\ 
        $10^3$ & 0.3111966674029827 & 0.9471428571428572 & \cellcolor{mygold}0.01695874663868121 & 0.21114044757559897 & 0.0804013620131965 \\ 
        $10^4$ & 0.31132705445800507 & \cellcolor{mygold}0.9571428571428572 & 0.017038725386041082 & \cellcolor{mygold}0.21060344410528029 & \cellcolor{mygold}0.07988158991608153 \\
        
        \bottomrule
        \end{tabular}
        }%
        \label{tab:grid_sample_count}
        \end{subtable}
        
\bigskip
\bigskip

\begin{subtable}{\linewidth}
        \caption{Cross-attention map source results. \textbf{GT} uses ground-truth foreground-background maps from PIE-Bench. \textbf{Forward} are cross-attention maps collected during the forward path. \textbf{Forward (mean)} averages all forward attention maps to one map. \textbf{Backward Source} and \textbf{Backward Target} are the attention maps from the backward source and target path respectively. \textbf{Backward Source+Target} combines both activations to one map via the $\max$-operator. We found that forward (mean) provides the best balance for text alignment and similarity.}
        \centering
        \resizebox{0.9\linewidth}{!}{%
        \small
        \centering
        \setlength{\tabcolsep}{4.25pt}
        \sisetup{table-auto-round}
        \begin{tabular}{@{}l*{5}{S[table-format=2.2,drop-exponent = true,fixed-exponent = -2,exponent-mode = fixed,]}@{}}%
        \toprule
        & \multicolumn{5}{c}{Metric $(\times {10}^{2})$ }  \\ 
        \cmidrule(r){2-6}
         &\multicolumn{2}{c}{Text-Image Alignment (CLIP)} & \multicolumn{3}{c}{Structural Similarity} \\ 
        \cmidrule(r){2-3} \cmidrule(lr){4-6} 
        Attention source & {text-img $\uparrow$} & {text-cap. $\uparrow$} & {DINOv1 $\downarrow$} & {LPIPS $\downarrow$} & {BG-LPIPS $\downarrow$}  \\ \midrule

        No mask & \cellcolor{mygold}0.3126584818959236 & \cellcolor{mygold}0.9542857142857143 &  0.018496008312795312 & 0.22771402167156338 & 0.09027566545884058 \\
        GT & 0.3122127622791699 & 0.9471428571428572 & \cellcolor{mygold}0.01671474074046793 & \cellcolor{mygold}0.2067473436226802 & \cellcolor{mygold}0.06996948878582251 \\ 
        Forward & 0.3123822163896901 & \cellcolor{mygold}0.9542857142857143 & 0.01750645055369075 & 0.2159298769783761 & 0.08268398533429717 \\ 
        \textbf{Forward (Mean)} & 0.31247970819473264 & \cellcolor{mygold}0.9542857142857143 & 0.017007591246760316 & 0.21136770641963396 & 0.07995662840633096 \\ 
        Backward Source & 0.31261885519538607 & 0.9514285714285714 & 0.018125142375938593 & 0.2229980825446546 & 0.08740944280110333 \\ 
        Backward Target & 0.31221912183931894 & 0.9471428571428572 & 0.018029548376466014 & 0.22271979062418854 & 0.08734988968391137 \\ 
        Backward Source+Target & 0.3125478547172887 & 0.9485714285714286 & 0.018283239707816392 & 0.22483578864218934 & 0.08843672327335558 \\
        
        \bottomrule
        \end{tabular}
        }%
        \label{tab:grid_attn_src}
        \end{subtable}
        
\bigskip
\bigskip

\begin{subtable}{\linewidth}
        \caption{Cross-attention threshold results. A higher threshold reduces the region where noise is being injected, resulting in less editing and better similarity while negatively affecting alignment. \textbf{Smooth} multiplies $\eta$ with the attention map activations instead of thresholding it, resulting in smaller $\eta$ being applied to less activated regions. $\mathbf{\mathcal{M}_{th} = 0.2}$ provides a good tradeoff between alignment and similarity metrics.}
        \centering
        \resizebox{0.9\linewidth}{!}{%
        \small
        \centering
        \setlength{\tabcolsep}{4.25pt}
        \sisetup{table-auto-round}
        \begin{tabular}{@{}c*{5}{S[table-format=2.2,drop-exponent = true,fixed-exponent = -2,exponent-mode = fixed,]}@{}}%
        \toprule
        & \multicolumn{5}{c}{Metric $(\times {10}^{2})$ }  \\ 
        \cmidrule(r){2-6}
         &\multicolumn{2}{c}{Text-Image Alignment (CLIP)} & \multicolumn{3}{c}{Structural Similarity} \\ 
        \cmidrule(r){2-3} \cmidrule(lr){4-6} 
        Attention threshold $\mathcal{M}_{th}$ & {text-img $\uparrow$} & {text-cap. $\uparrow$} & {DINOv1 $\downarrow$} & {LPIPS $\downarrow$} & {BG-LPIPS $\downarrow$}  \\ \midrule
        
        No mask & \cellcolor{mygold}0.3126584818959236 & 0.9542857142857143 &  0.018496008312795312 & 0.22771402167156338 & 0.09027566545884058 \\
        $0.1$ & 0.3125421884655952 & \cellcolor{mygold}0.9571428571428572 & 0.01805870528898335 & 0.22210704055481723 & 0.08589022974797575 \\ 
        $\mathbf{0.2}$ & 0.31247970819473264 & 0.9542857142857143 & 0.017007591246760316 & 0.21136770641963396 & 0.07995662840633096 \\ 
        $0.3$ & 0.311488395099129 & 0.9471428571428572 & 0.016380587407454315 & 0.20248329586748565 & 0.07615126278039368 \\ 
        $0.4$ & 0.31084525955574854 & 0.9528571428571428 & 0.01555208809341171 & 0.19348110478637473 & 0.07238848480820056 \\ 
        $0.5$ & 0.3104569662468774 & \cellcolor{mygold}0.9571428571428572 & 0.015136553602226611 & 0.1882278627609568 & 0.07020961104777858 \\ 
        Smooth & 0.31051539631826536 & 0.9498567335243553 & \cellcolor{mygold}0.015073008510683264 & \cellcolor{mygold}0.18765462182666193 & \cellcolor{mygold}0.07019156411817155 \\
        
        \bottomrule
        \end{tabular}
        }%
        \label{tab:grid_attn_thres}
        \end{subtable}

\end{table*}

%% file: appendix/inversion_methods.tex
\section{Existing Diffusion Inversion Methods}
\label{sec:inversion_methods}
We summarize existing diffusion inversion methods in \cref{tab:summary}. The table divides each forward and backward path into ($\eta,w$) and strategy. Former indicates the used $\eta$ parameter (DDIM or DDPM) with the guidance scale parameter $w$, while strategy indicates additional effort by the inversion method to reduce the gap to the (ideal) forward path.

%% file: tables/tab_inversion_1.tex
\begin{table}[h]
    \caption{Summary of existing diffusion inversion methods.} %
    \centering
    \resizebox{1\linewidth}{!}{%
    \small
    \centering
    \setlength{\tabcolsep}{4.25pt}
    \sisetup{table-auto-round}
    \begin{tabular}{@{}l|cc|ccc@{}}%
    \toprule
    & \multicolumn{2}{c}{\hlc[myinv]{forward path}} & \multicolumn{2}{c}{\hlc[myrec]{backward path}} \\ 
    \cmidrule(r){2-3} \cmidrule(lr){4-5} 
    Inversion Method & $(\eta, w)$ & strategy & $(\eta, w)$ & strategy  \\ \midrule
    DDIM Inv. \cite{song2020denoising} & $(0,1)$ & - & $(0,7.5)$ & -  \\
    NTI \cite{mokady2023null}                                            & $(0,1)$ & - & $(0,7.5)$ & optimized $\emptyset_{t}$  \\
    NPI \cite{miyake2023negative}                                                   & $(0,1)$ & - & $(0,7.5)$ & $\emptyset \leftarrow c^{(s)}$  \\ 
    ProxNPI \cite{han2023improving}                                                 & $(0,1)$ & - & $(0,7.5)$ & $\emptyset \leftarrow c^{(s)}$, modified $\tilde{\epsilon}_{\theta}$ \\ 
    Direct Inv. \cite{ju2023direct}                                                 & $(0,1)$ & - & $(0,7.5)$ & ${\boldsymbol{x}_{t-1}^{(s)'}} \leftarrow \boldsymbol{x}_{t-1}^{(s)^{*}}$ \\ %
    EDICT \cite{wallace2023edict}                                                   & $(0,3)$ & Coupled Transformations & $(0,3)$ & Coupled Transformations \\ 
    DDPM Inv. \cite{huberman2023edit}                                & - & $q(\boldsymbol{x}_{t}|\boldsymbol{x}_{0})=\mathcal{N}(\boldsymbol{x}_{t};\sqrt{\bar{\alpha}_{t}}\boldsymbol{x}_{0},(1-\bar{\alpha}_{t})I)$ & $(1,3.5)$ & ${\boldsymbol{x}_{t-1}^{(s)'}} \leftarrow \boldsymbol{x}_{t-1}^{(s)^{*}}$, modified $\epsilon_{\mathrm{add}}$ \\ \bottomrule
    \end{tabular}
    }%
    \label{tab:summary}
    \end{table}

%% file: appendix/editing_methods.tex
\section{Text-guided Image Editing Methods}
\label{sec:editing_methods}
In this section, we introduce commonly used training-free text-guided image editing methods. We also state the inversion method used by the original paper for real image editing. %

\subsection{Prompt-to-Prompt (PtP)  \cite{hertz2022prompt}}

\paragraph{Editing}
To forward information from the source to the target backward path, PtP replaces the cross-attention maps of the target with the maps from the source. Since PtP needs to know which attention maps to exchange, each word in the source prompt must be matched to a word in the target prompt. Thus, PtP introduces restrictions for specifying an appropriate source and target prompt. %

\paragraph{Inversion}
The proposed Prompt-to-Prompt in \cite{hertz2022prompt} applies DDIM Inversion ($w=1$) and DDIM sampling ($w=7.5$) without any additional strategy, which negatively affects structural similarity.

\subsection{MasaCtrl \cite{cao2023masactrl}}

\paragraph{Editing} 
MasaCtrl focuses on non-rigid real image editing (motion editing). Unlike PtP, MasaCtrl replaces self-attention maps instead of cross-attention maps. For real image editing, MasaCtrl uses an empty prompt for the source prompt. Otherwise, when performing synthetic editing and if a source prompt is specified, MasaCtrl optionally uses masked guidance by obtaining a mask from the cross-attention maps for the word to edit. %

\paragraph{Inversion}
MasaCtrl applies DDIM Inversion ($w=1$) and replaces the source prompt $c^{(s)}$ with an empty prompt $\emptyset$. Consequently, MasaCtrl does not require a source prompt for inversion.

\subsection{Plug-and-Play (PnP) \cite{tumanyan2023plug}}

\paragraph{Editing}
Plug-and-Play also focuses on modifying self-attention maps similar to MasaCtrl. In addition, PnP performs spatial feature injection in U-Net's decoder layers from the source to the target branch. They thereby report better structural preservation than PtP.

\paragraph{Inversion}
Same as MasaCtrl, Plug-and-Play applies DDIM Inversion ($w=1$) and replaces the source prompt $c^{(s)}$ with an empty prompt $\emptyset$.

%% file: appendix/metrics.tex
\section{Image Editing Metrics}
\label{sec: metrics}
It is unclear how to evaluate image editing performance. Instead of having a single metric measuring both text-image alignment with the target prompt and structural similarity with the source image, previous methods focus on evaluating these two concepts separately. For text-image alignment, CLIP \cite{radford2021learning} is commonly used, while for structural similarity, a strong feature extractor like DINO \cite{caron2021emerging} is utilized along with more classical evaluation metrics such as MS-SSIM \cite{wang2003multiscale}. Below, we give a detailed explanation of all our metrics.

\subsection{Text-Image Alignment}

\subsubsection{$\textrm{Text}^{\textrm{(t)}}$ - $\textrm{Image}^{\textrm{(t)}}$ CLIP Similarity (text-img)}
This metric first computes the target prompt text embeddings and the output image embeddings using CLIP, normalizes them, and finally computes the dot product of the two embeddings. The larger the dot product value, the better the target prompt and the output image are aligned.

\subsubsection{$\textrm{Text}^{\textrm{(t)}}$ - $\textrm{Image Caption}^{\textrm{(t)}}$ \cite{chefer2023attendandexcite} CLIP Similarity (text-cap)} %
Similar to text-image CLIP similarity, but instead of using the output image embeddings, a caption of the output image is obtained via BLIP \cite{li2022blip}. We then embed the generated caption by CLIP and compute the dot product of the target prompt text embeddings and the generated caption embeddings. Since there is a gap between CLIP's image space and CLIP's text space, this has the advantage that both embeddings lie in CLIP's text space.

\subsubsection{Directional CLIP Similarity \cite{patashnik2021styleclip} (directional)}
Unlike the above two metrics, this metric additionally incorporates the source prompt and the source image. The idea of this metric is that, in CLIP space, the direction from the source prompt to the target prompt should match the direction from the source image to the output image. Therefore, directional CLIP similarity computes the respective embeddings of both prompts and images and then retrieves the text direction and image direction by taking their respective difference. Lastly, the dot product of those two directions serves as the metric value.

\subsubsection{CLIP Accuracy (acc)}
This metric was originally introduced in \cite{parmar2023zero} and computes the ratio of output images where the text-image CLIP similarity is higher with the target prompt than with the source prompt. We noticed in our experiments that most inversion and editing methods reach a perfect score of 1. Thus, we decided to use text-caption similarity via BLIP instead of text-image similarity.

\subsection{Structural Similarity}

\subsubsection{DINO \cite{caron2021emerging} Self-similarity}

DINO self-similarity measures the similarity between the source image and the target image by obtaining embeddings via DINO and computing their MSE loss. DINO effectively extracts structural features of images compared to other feature extraction models. We provide metrics for both DINOv1 \cite{caron2021emerging} and DINOv2 \cite{oquab2024dinov2} but focus on the more commonly used DINOv1. 

\subsubsection{LPIPS \cite{zhang2018perceptual}}
LPIPS is another metric similar to DINO self-similarity operating on AlexNet \cite{alexnet} and focusing on matching human perception.

\subsubsection{BG-LPIPS \cite{parmar2023zero,ju2023direct}}

BG-LPIPS computes the LPIPS only on the background part of the source and output image, which should not be edited. This metric works well for edits like replacing or editing single objects rather than performing style transfer, where there is no clear background. The background mask is provided by the user or dataset.

\subsubsection{MS-SSIM \cite{wang2003multiscale}}

Multi-scale structural similarity (MS-SSIM) is an improved version of SSIM and computes similarity over various image scales by consecutive downsampling.

\subsection{Text-Image Alignment and Structural Similarity}

\subsubsection{VIEScore \cite{ku2024viescoreexplainablemetricsconditional}}
VIEScore is specifically introduced for evaluating image generation and editing performance using GPT-4V(ision)~\cite{openai2023gpt4}. For image editing, VIEScore assesses both text-image alignment and structural similarity. It employs a carefully designed prompt that rates the editing performance with three separate scores, ranging from 0 to 10, where higher is better: overall score, alignment score, and similarity score. The overall score evaluates the image editing in general, whereas the alignment score and similarity score focus on text-image alignment and structural similarity, respectively.

%% file: appendix/additional_results.tex
\section{Additional Quantitative and Qualitative Results}
\label{sec:additional_results}

\subsection{Additional Quantitative Results}

\cref{tab:result_appendix} provides results for PIE-Bench editing with additional metrics. \cref{fig:clip_dino_all} shows trade-off plots for PtP, PnP and MasaCtrl. PIE-Bench also divides all 700 edits into ten categories, such as random edits or changing objects. We thus evaluate and present our metrics for each category individually in \cref{tab:result_edit_cat}. Additionally, we provide supplementary metrics for the GPT-4~\cite{openai2023gpt4} based VIEScore~\cite{ku2024viescoreexplainablemetricsconditional} on subsets of PIE-Bench in \cref{tab:result_appendix_gpt}, and pairwise compare our method with Direct Inversion through human evaluation in \cref{table_user_study}. Furthermore, \cref{tab:result_appendix_ti2i} provides metrics for the ImageNetR-TI2I \cite{tumanyan2023plug} dataset. In all additional experiments, our method achieves state-of-the-art results in most cases.

In \cref{tab:result_appendix_rec}, we evaluate the reconstruction accuracy and inference time of our method. We observe that our method achieves perfect reconstruction, matching VAE reconstruction, which serves as the upper bound. Additionally, our method introduces negligible overhead and offers an inference speed similar to standard DDIM Inversion, when compared to both standard inversion and all three tested editing methods.

\subsection{Additional Qualitative Results}

In addition to the above quantitative results, we provide more qualitative comparisons of our method with various other inversion and editing approaches. \cref{fig:ptp}, \cref{fig:pnp}, and \cref{fig:masa} show qualitative results on PIE-Bench \cite{ju2023direct} images using PtP \cite{dong2023prompt}, PnP \cite{tumanyan2023plug}, and MasaCtrl \cite{cao2023masactrl} respectively. Additionally, \cref{fig:etainv3} shows examples where a larger $\eta$ as in EtaInv (3) leads to better editing results (e.g., style transfer). Finally, \cref{fig:rat_pig} shows the impact of $\eta$ on image editing where higher $\eta$ values result in better target prompt alignment while negatively influencing structural similarity.

%% file: tables/tab_appendix_result.tex
\begin{table*}[h]
        \caption{PIE-Bench \cite{ju2023direct} evaluation with additional metrics. 
        }
        \centering
        \resizebox{1.0\textwidth}{!}{%
        \small
        \centering
        \setlength{\tabcolsep}{4.25pt}
        \sisetup{table-auto-round}
        \begin{tabular}{@{}cl*{10}{S[table-format=2.2,drop-exponent = true,fixed-exponent = -2,exponent-mode = fixed,]}@{}}%
        \toprule
        && \multicolumn{9}{c}{Metric $(\times {10}^{2})$ }  \\ 
        \cmidrule(r){3-11}
        &&\multicolumn{4}{c}{Text-Image Alignment (CLIP)} & \multicolumn{5}{c}{Structural Similarity} \\ 
        \cmidrule(r){3-6} \cmidrule(lr){7-11} 
        Editing & Inversion & {text-img $\uparrow$} & {text-cap $\uparrow$} & {directional $\uparrow$} & {acc $\uparrow$} & {DINOv1 $\downarrow$} & {DINOv2 $\downarrow$} & {LPIPS $\downarrow$} & {BG-LPIPS $\downarrow$} & {MS-SSIM $\uparrow$} \\ \midrule
        \multirow{10}{*}{PtP}
        
        & DDIM Inv. \cite{song2020denoising} & 0.3098963077579226 & 0.7573099485039712 & 0.01489594351026296 & 0.9457142857142857 & 0.06942023908586374 & 0.00979264195576044 & 0.4664535358973912 & 0.24970233517599158 & 0.6192663074391229 \\ 
        & Null-text Inv. \cite{mokady2023null} & 0.3073391259780952 & 0.7440846993241992 & \cellcolor{mybronze}0.03346213553537382 & 0.9257142857142857 & \cellcolor{mybronze}0.012444119106512518 & \cellcolor{mybronze}0.0032320704425053137 & \cellcolor{mybronze}0.151256681012788 & \cellcolor{mybronze}0.05694563213929151 & \cellcolor{mybronze}0.8876146561758859 \\ 
        & NPI \cite{miyake2023negative} & 0.3048740061053208 & 0.7427982727544649 & \cellcolor{mygold}0.0379693094396498 & 0.9271428571428572 & 0.020258124567002857 & 0.004317765259599713 & 0.19275611174692 & 0.08241629246801105 & 0.8592968276568822 \\ 
        & ProxNPI \cite{han2023improving} & 0.3031321707155023 & 0.7400196839656149 & \cellcolor{mysilver}0.03477620806931684 & 0.9242857142857143 & 0.019193061842317026 & 0.004098424002960591 & 0.17685471942116107 & 0.07764945874233879 & 0.8683452299662999 \\ 
        & EDICT \cite{wallace2023edict} & 0.2927953909763268 & 0.7269143332328115 & 0.008288079663262969 & 0.9271428571428572 & \cellcolor{mygold}0.004124843035575135 & \cellcolor{mygold}0.001586323544906918 & \cellcolor{mygold}0.06648621396760324 & \cellcolor{mygold}0.030968608957004366 & \cellcolor{mygold}0.9370805729287012 \\ 
        & DDPM Inv. \cite{huberman2023edit} & 0.2942695520392486 & 0.7312118480886732 & 0.006871060335543007 & 0.9271428571428572 & \cellcolor{mysilver}0.004167493938834274 & \cellcolor{mysilver}0.0017546577040255735 & \cellcolor{mysilver}0.06871223960737033 & \cellcolor{mysilver}0.03273383083458092 & \cellcolor{mysilver}0.9338151170526232 \\ 
        & Direct Inv. \cite{ju2023direct} & 0.3091879120469093 & 0.7576252316577093 & 0.023372119456159583 & \cellcolor{mybronze}0.9471428571428572 & 0.012763145404169335 & 0.0032352335843773158 & 0.15788746399538858 & 0.06329163567514895 & 0.8833724835940769 \\ 
        & \textbf{EtaInv (1)} & \cellcolor{mybronze}0.31007335275411607 & \cellcolor{mybronze}0.7607581440891539 & 0.023874808910539808 & \cellcolor{mysilver}0.95 & 0.013433658024296165 & 0.0033779775318024414 & 0.16576544786404285 & 0.06568354647575429 & 0.8759700518846512 \\ 
        & \textbf{EtaInv (2)} & \cellcolor{mysilver}0.31247970819473264 & \cellcolor{mysilver}0.7619916633622987 & 0.02674273075338403 & \cellcolor{mygold}0.9542857142857143 & 0.017007591246760316 & 0.003993112342598449 & 0.21136770641963396 & 0.07995662840633096 & 0.831945077266012 \\ 
        & \textbf{EtaInv (3)} & \cellcolor{mygold}0.3151833241752216 & \cellcolor{mygold}0.7627197948949678 & 0.03207994163738996 & \cellcolor{mysilver}0.95 & 0.030452230763954244 & 0.00588432825819057 & 0.33568434526079466 & 0.13305862386694312 & 0.6886161335664136 \\ 
        
        \midrule
        \multirow{10}{*}{PnP}
        
        & DDIM Inv. \cite{song2020denoising} & 0.2937987298624856 & 0.6938250744768552 & \cellcolor{mybronze}0.041720140941366224 & 0.8557142857142858 & 0.061101810900228364 & 0.009447477753939374 & 0.4084159646289689 & 0.20844711300505359 & 0.6929777942384993 \\ 
        & Null-text Inv. \cite{mokady2023null} & 0.30753638380340165 & 0.7376578430192812 & \cellcolor{mysilver}0.046171344291152695 & 0.9042857142857142 & 0.03266450203722343 & 0.006085259692543851 & 0.3051460664719343 & 0.14172093506670666 & 0.7837164359433311 \\ 
        & NPI \cite{miyake2023negative} & 0.307270291192191 & 0.7383094525762967 & \cellcolor{mygold}0.04841411963624913 & 0.9128571428571428 & 0.026717880195771742 & 0.00527817684092692 & 0.26182728880750283 & \cellcolor{mybronze}0.11566226213575906 & 0.8180732675535338 \\ 
        & ProxNPI \cite{han2023improving} & 0.3053714159343924 & 0.7402171358891896 & 0.04122855223349429 & 0.9071428571428571 & \cellcolor{mybronze}0.022860674153281642 & \cellcolor{mysilver}0.004638417416426819 & \cellcolor{mysilver}0.2175858683032649 & \cellcolor{mysilver}0.09565648488966481 & \cellcolor{mysilver}0.8443145848172051 \\ 
        & EDICT \cite{wallace2023edict} & 0.2469451144444091 & 0.6009387849697045 & 0.02577517635895804 & 0.6342857142857142 & 0.042556588900874236 & 0.0072898531979548615 & 0.30217926079939517 & 0.14956702334493457 & 0.7646284541061946 \\ 
        & DDPM Inv. \cite{huberman2023edit} & 0.30258001885243824 & 0.7372387099266052 & 0.020220083087395844 & \cellcolor{mybronze}0.9485714285714286 & \cellcolor{mygold}0.01044435668470604 & \cellcolor{mygold}0.002897609248092132 & \cellcolor{mygold}0.12501146100461483 & \cellcolor{mygold}0.05844134287490306 & \cellcolor{mygold}0.8934911684479032 \\ 
        & Direct Inv. \cite{ju2023direct} & 0.3131952231909548 & 0.7613077391045434 & 0.030966096437436395 & \cellcolor{mysilver}0.9514285714285714 & \cellcolor{mysilver}0.02274692072533071 & \cellcolor{mybronze}0.004779900857746335 & \cellcolor{mybronze}0.2559120710087674 & 0.12984735441078166 & \cellcolor{mybronze}0.8233862844535282 \\ 
        & \textbf{EtaInv (1)} & \cellcolor{mybronze}0.3132966907748154 & \cellcolor{mysilver}0.7646695785863059 & 0.026571786788949146 & \cellcolor{mybronze}0.9485714285714286 & 0.023384014786819795 & 0.004964417909671153 & 0.27332021001194207 & 0.1405317166185601 & 0.8084589408976691 \\ 
        & \textbf{EtaInv (2)} & \cellcolor{mysilver}0.3162599997861045 & \cellcolor{mybronze}0.7635779229232244 & 0.03223152512468264 & \cellcolor{mygold}0.9528571428571428 & 0.03399889323993453 & 0.006371868756333632 & 0.36594185411930086 & 0.18716991466014796 & 0.7082592932241304 \\ 
        & \textbf{EtaInv (3)} & \cellcolor{mygold}0.31920644049133573 & \cellcolor{mygold}0.7709907117911747 & 0.039845791492927156 & 0.9457142857142857 & 0.051576449676815954 & 0.008770590557583741 & 0.5039191923716239 & 0.2661278968174468 & 0.5053479700854846 \\ 
        
        \midrule
        \multirow{10}{*}{Masa}
        
        & DDIM Inv. \cite{song2020denoising} & \cellcolor{mygold}0.30742404401302337 & \cellcolor{mygold}0.7521762572016035 & 0.013699371784293491 & \cellcolor{mygold}0.95 & 0.07545868670301778 & 0.010215338541394366 & 0.4767742032238415 & 0.25374347115294116 & 0.6036750082245895 \\ 
        & Null-text Inv. \cite{mokady2023null} & 0.30068163248045104 & 0.7279130147184645 & \cellcolor{mysilver}0.02638700956223017 & 0.93 & 0.04494556720335303 & 0.0068012517306488005 & 0.2501945654968066 & 0.11923769590272319 & 0.7735224713172232 \\ 
        & NPI \cite{miyake2023negative} & 0.29544316662209374 & 0.7117334362012999 & \cellcolor{mygold}0.029506508847074914 & 0.8728571428571429 & 0.045089687282951284 & 0.006975050355706896 & 0.2603068265744618 & 0.12409586446343122 & 0.7793228232434818 \\ 
        & ProxNPI \cite{han2023improving} & 0.29493104496172495 & 0.7149522835016251 & \cellcolor{mybronze}0.025920723030015195 & 0.8814285714285715 & 0.03921245703839564 & 0.006251328894535878 & \cellcolor{mybronze}0.22985350452895675 & \cellcolor{mybronze}0.10990528252940358 & \cellcolor{mybronze}0.8050647920370102 \\ 
        & EDICT \cite{wallace2023edict} & 0.29676960666264807 & 0.7345920038649014 & 0.008100226531137846 & 0.9328571428571428 & \cellcolor{mysilver}0.007871210374037868 & \cellcolor{mysilver}0.002304856096868337 & \cellcolor{mygold}0.08588012009726038 & \cellcolor{mysilver}0.042012395707575444 & \cellcolor{mygold}0.923195264169148 \\ 
        & DDPM Inv. \cite{huberman2023edit} & 0.2956833677419594 & 0.7316503988419261 & 0.006980817656169945 & 0.93 & \cellcolor{mygold}0.007514222693363471 & \cellcolor{mygold}0.002231737565348989 & \cellcolor{mysilver}0.08646675587764809 & \cellcolor{mygold}0.04115439930333586 & \cellcolor{mysilver}0.9151499950034278 \\ 
        & Direct Inv. \cite{ju2023direct} & 0.30372176436441284 & \cellcolor{mysilver}0.7449550858991487 & 0.013893943320352783 & \cellcolor{mysilver}0.9457142857142857 & 0.04321882900914976 & 0.006632671428212364 & 0.2691428067535162 & 0.1375892746267969 & 0.7694688713976315 \\ 
        & \textbf{EtaInv (1)} & 0.30390003306525093 & 0.7400398475357465 & 0.016195030842375543 & 0.9314285714285714 & \cellcolor{mybronze}0.03661353356404496 & \cellcolor{mybronze}0.005908169171723005 & 0.23115233425050974 & 0.11566979534219302 & 0.7961219504049846 \\ 
        & \textbf{EtaInv (2)} & \cellcolor{mybronze}0.3062442834462438 & \cellcolor{mybronze}0.7432072950686728 & 0.019605730133813008 & 0.9385714285714286 & 0.05237354726530612 & 0.007951788205287552 & 0.33066985155854905 & 0.16644072913740404 & 0.6818421719329698 \\ 
        & \textbf{EtaInv (3)} & \cellcolor{mysilver}0.3071766421198845 & 0.7423931103093283 & 0.022181702222941177 & \cellcolor{mybronze}0.94 & 0.07208526262880437 & 0.010562405071021724 & 0.4496590716391802 & 0.23651542364653974 & 0.4764318342506886 \\ 
        
        \bottomrule
        \end{tabular}
        }%
        \label{tab:result_appendix}
        \end{table*}

%% file: figs/fig_graph_all.tex
\begin{figure}
    \captionsetup[subfigure]{labelformat=empty}
    \centering
    \hspace{-0.5cm}
    \begin{subfigure}{0.33\linewidth}
        \centering
    \begin{tikzpicture}[scale=0.50]
        \begin{axis}[
            axis lines = left,
            every axis y label/.style={at={(current axis.north west)},above=27.5mm},
            xlabel = {CLIP similarity ($\times 10^{2}$)},
            ylabel = {DINO ($\times 10^{2}$)},
            legend pos=north west,
            xtick={29.5,30,30.5,31},
            ytick={2,4,6},
            xmin=29, xmax=31.5,
            ymin=0, ymax=7,
        ]
            \addplot[mark size=1pt,black,mark=*,mark options={fill=black},nodes near coords,only marks,
               point meta=explicit symbolic,
               visualization depends on={value \thisrow{anchor}\as\myanchor},
               every node near coord/.append style={anchor=\myanchor}
            ] table[meta=label] {
            x y label anchor
                30.99 6.94 {\small DDIM Inv.} north
                30.73 1.24 {\small NTI} south
                30.49 2.03 {\small NPI} south
                30.31 1.92 {\small ProxNPI} east
                29.28 0.41 {\small EDICT} south
                29.82 1.19 {\small DDPM $\textrm{Inv.}$} south
                30.92 1.28 {\small Dir. Inv.} north
            };\addlegendentry{PtP}
            \addplot[mark size=1pt,red,mark=*,mark options={fill=red},nodes near coords,only marks,
               point meta=explicit symbolic,
               visualization depends on={value \thisrow{anchor}\as\myanchor},
               every node near coord/.append style={anchor=\myanchor}
            ] table[meta=label] {
            x y label anchor
                31.01 1.34  {EtaInv (1)} west
                31.25 1.70   {EtaInv (2)} west
            };
        \end{axis}
        \end{tikzpicture}
               \end{subfigure}
       \begin{subfigure}{0.33\linewidth}
        \centering
    \begin{tikzpicture}[scale=0.50]
        \begin{axis}[
            axis lines = left,
            every axis y label/.style={at={(current axis.north west)},above=27.5mm},
            xlabel = {CLIP similarity ($\times 10^{2}$)},
            ylabel = {DINO ($\times 10^{2}$)},
            legend pos=north east,
            xtick={29.5,30,30.5,31,31.5},
            ytick={2,4,6},
            xmin=29.2, xmax=32,
            ymin=0.5, ymax=6.5,
        ]
            \addplot[mark size=1pt,black,mark=*,mark options={fill=black},nodes near coords,only marks,
               point meta=explicit symbolic,
               visualization depends on={value \thisrow{anchor}\as\myanchor},
               every node near coord/.append style={anchor=\myanchor}
            ] table[meta=label] {
            x y label anchor
            29.38 6.11 {DDIM Inv.} west
            30.75 3.27 {NTI} south
            30.73 2.67 {NPI} south
            30.54 2.29 {ProxNPI} east
            30.26 1.04 {DDPM $\textrm{Inv.}$} south
            31.32 2.27 {Dir. Inv.} east
            };\addlegendentry{PnP}
            \addplot[mark size=1pt,red,mark=*,mark options={fill=red},nodes near coords,only marks,
               point meta=explicit symbolic,
               visualization depends on={value \thisrow{anchor}\as\myanchor},
               every node near coord/.append style={anchor=\myanchor}
            ] table[meta=label] {
            x y label anchor
                31.33 2.34  {EtaInv (1)} west
                31.63 3.40   {EtaInv (2)} north
            };
        \end{axis}
        \end{tikzpicture}
           \end{subfigure}
       \begin{subfigure}{0.33\linewidth}
        \centering
    \begin{tikzpicture}[scale=0.50]
        \begin{axis}[
            axis lines = left,
            every axis y label/.style={at={(current axis.north west)},above=27.5mm},
            xlabel = {CLIP similarity ($\times 10^{2}$)},
            ylabel = {DINO ($\times 10^{2}$)},
            legend pos=north west,
            xtick={29.5,30,30.5},
            ytick={2,4,6,8},
            xmin=29.2, xmax=30.75,
            ymin=0, ymax=9,
        ]
            \addplot[mark size=1pt,black,mark=*,mark options={fill=black},nodes near coords,only marks,
               point meta=explicit symbolic,
               visualization depends on={value \thisrow{anchor}\as\myanchor},
               every node near coord/.append style={anchor=\myanchor}
            ] table[meta=label] {
            x y label anchor
            30.74 7.55 {DDIM Inv.} east
            30.07 4.49 {NTI} south
            29.54 4.51 {NPI} south
            29.49 3.92 {ProxNPI} north
            29.68 0.79 {EDICT} south
            29.57 0.75 {DDPM Inv.} north
            30.37 4.32 {Dir. Inv.} east
            };\addlegendentry{MasaCtrl}
            \addplot[mark size=1pt,red,mark=*,mark options={fill=red},nodes near coords,only marks,
               point meta=explicit symbolic,
               visualization depends on={value \thisrow{anchor}\as\myanchor},
               every node near coord/.append style={anchor=\myanchor}
            ] table[meta=label] {
            x y label anchor
                30.39 3.66  {EtaInv (1)} west
                30.62 5.24   {EtaInv (2)} north
            };
        \end{axis}
        \end{tikzpicture}
           \end{subfigure}
        \caption{Visualization of CLIP text-image metrics (higher is better) and DINO metrics (lower is better) on PIE-Bench for PtP (left), PnP (middle), and MasaCtrl (right).}
        \label{fig:clip_dino_all}
        \end{figure}
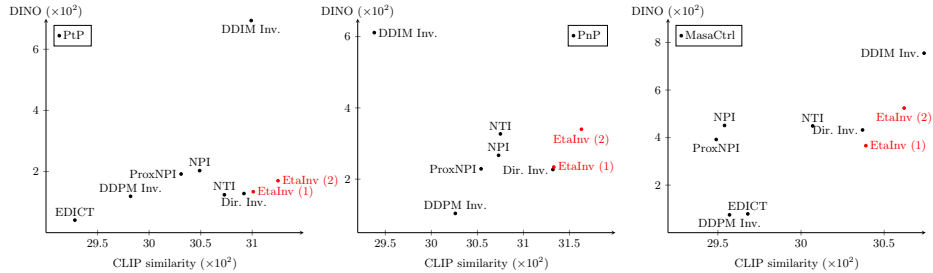

%% file: tables/tab_category.tex
\begin{sidewaystable*}
    \captionsetup{width=15cm}
        \caption{Per category evaluation for PIE-Bench \cite{ju2023direct}: Categories are defined as: (0) random; (1) change object; (2) add object; (3) delete object; (4) change attribute content; (5) change attribute pose; (6) change attribute color; (7) change attribute material; (8) change background; and (9) change style. 
        }
        \medskip
        \centering
        \resizebox{0.9\linewidth}{!}{%
        \small
        \centering
        \setlength{\tabcolsep}{4.25pt}
        \sisetup{table-auto-round}
        \begin{tabular}{@{}cl*{21}{S[table-format=2.2,drop-exponent = true,fixed-exponent = -2,exponent-mode = fixed,]}@{}}%
        \toprule
        && \multicolumn{20}{c}{Metric $(\times {10}^{2})$ }  \\ 
        \cmidrule(r){3-22}
         &\multicolumn{10}{c}{Text-Image alignment (CLIP) $\uparrow$} & \multicolumn{10}{c}{Structural similarity (DINOv1) $\downarrow$} \\ 
        \cmidrule(r){3-12} \cmidrule(l){13-22} 
        Editing & Inversion & {0} & {1} & {2} & {3} & {4} & {5} & {6} & {7} & {8} & {9} & {0} & {1} & {2} & {3} & {4} & {5} & {6} & {7} & {8} & {9} \\ \midrule
        \multirow{9}{*}{PtP} & DDIM Inv. & 0.307326 & 0.307877 & 0.312836 & 0.299251 & 0.312863 & 0.322136 & 0.311816 & 0.324791 & 0.307973 & 0.310033 & 0.067908 & 0.067148 & 0.067044 & 0.072343 & 0.066201 & 0.062148 & 0.075006 & 0.080242 & 0.07556 & 0.064694 \\
        & NTI     & 0.306799 & 0.304231 & 0.302367 & 0.295219 & 0.303331 & 0.316315 & 0.313104 & 0.31604 & 0.305457 & 0.320649 & 0.014148 & 0.014472 & 0.00931 & 0.01152 & 0.008951 & 0.007608 & 0.011721 & 0.011144 & 0.013148 & 0.015966 \\
        & NPI     & 0.306142 & 0.302937 & 0.30124  & 0.292681 & 0.304927 & 0.313722 & 0.308853 & 0.309746 & 0.302043 & 0.314374 & 0.0198 & 0.026453 & 0.019142 & 0.019713 & 0.013697 & 0.015156 & 0.020007 & 0.018437 & 0.021468 & 0.022184 \\ 
        & ProxNPI & 0.304369 & 0.302053 & 0.299707 & 0.292917 & 0.303927 & 0.31319  & 0.307453 & 0.308657 & 0.299679 & 0.308793 & 0.018499 & 0.025543 & 0.018474 & 0.018964 & 0.012948 & 0.014837 & 0.019148 & 0.018099 & 0.019841 & 0.020227 \\ 
        & EDICT & 0.285026 & 0.283076 & 0.295225 & 0.290474 & 0.297652 & 0.31773 & 0.295827 & 0.304465 & 0.292055 & 0.294497 & 0.004695 & 0.004585 & 0.003518 & 0.004008 & 0.003499 & 0.003772 & 0.004172 & 0.003751 & 0.004105 & 0.004065 \\ 
        & DDPM Inv. & 0.28547 & 0.284947 & 0.297418 & 0.290803 & 0.297971 & 0.320896 & 0.297985 & 0.306135 & 0.292827 & 0.297799 & 0.004457 & 0.004295 & 0.0037 & 0.004072 & 0.003618 & 0.00414 & 0.004376 & 0.004005 & 0.004237 & 0.004293 \\ 
        & Direct Inv.       & 0.307543 & 0.304338 & 0.308609 & 0.299785 & 0.304977 & 0.319798 & 0.312669 & 0.321377 & 0.307917 & 0.317135 & 0.013727 & 0.012361 & 0.011154 & 0.012283 & 0.01008 & 0.011471 & 0.011755 & 0.011266 & 0.013189 & 0.016383 \\
        
        & \textbf{EtaInv (1)} & 0.3088670386799744 & 0.3049957582727075 & 0.30978108290582895 & 0.3002709051594138 & 0.3065885674208403 & 0.3196298360824585 & 0.31402608528733256 & 0.3222586393356323 & 0.3074317755177617 & 0.3188934329897165 & 0.01437125562640306 & 0.012706887372769416 & 0.012147608186933213 & 0.01335526649345411 & 0.010850934393238277 & 0.012248380086384713 & 0.012117453827522695 & 0.011736998218111694 & 0.013322794999112375 & 0.01738537005148828 \\
        & \textbf{EtaInv (2)} & 0.3107388467660972 & 0.31022951621562245 & 0.31196989994496105 & 0.30106402933597565 & 0.30900383852422236 & 0.31919717639684675 & 0.31631172075867653 & 0.32266714945435526 & 0.3115272387862206 & 0.32202383801341056 & 0.01704761429739717 & 0.015518841680022887 & 0.016473714867606758 & 0.018038870037707967 & 0.013206117681693285 & 0.015176402957877144 & 0.014543411717750133 & 0.014501588104758411 & 0.016972376362537032 & 0.023265535209793596 \\
        & \textbf{EtaInv (3)} & 0.3126392132469586 & 0.31085608918219804 & 0.31414625234901905 & 0.30339304320514204 & 0.30995155684649944 & 0.32172586023807526 & 0.32013949528336527 & 0.32239895313978195 & 0.3167287286370993 & 0.3285034172236919 & 0.02916099014442547 & 0.027437129954341798 & 0.028489639080362394 & 0.03419724965497153 & 0.022634642524644734 & 0.026800072158221156 & 0.025283850822597743 & 0.022823651152430104 & 0.03467892831831705 & 0.04185123109491542 \\

        \midrule
        \multirow{10}{*}{PnP} & DDIM Inv.   & 0.297129 & 0.291502 & 0.292808 & 0.28123 & 0.293455 & 0.29697 & 0.300111 & 0.309246 & 0.283235 & 0.302098 & 0.063373 & 0.062158 & 0.055789 & 0.062081 & 0.049262 & 0.057084 & 0.059192 & 0.060756 & 0.069685 & 0.060879 \\
        & NTI     & 0.305549 & 0.305114 & 0.300723 & 0.294675 & 0.305168 & 0.306015 & 0.308688 & 0.316005 & 0.309869 & 0.327915 & 0.032252 & 0.037054 & 0.029658 & 0.032975 & 0.022361 & 0.024297 & 0.032763 & 0.029243 & 0.035522 & 0.039832 \\
        & NPI     & 0.306708 & 0.30185 & 0.302235 & 0.297513  & 0.304536 & 0.312419 & 0.310455 & 0.314841 & 0.305483 & 0.323669 & 0.02613 & 0.034589 & 0.022405 & 0.028249 & 0.017816 & 0.017077 & 0.02537 & 0.024443 & 0.027508 & 0.03295 \\
        & ProxNPI & 0.305434 & 0.300901 & 0.300878 & 0.297035 & 0.303044 & 0.313172 & 0.308931 & 0.31187 & 0.303615 & 0.316555 & 0.021704 & 0.03088 & 0.020176 & 0.025176 & 0.01569 & 0.015147 & 0.022754 & 0.020736 & 0.02344 & 0.025214 \\
        & NSI     & 0.308074 & 0.307608 & 0.313703 & 0.305411 & 0.308962 & 0.322299 & 0.316572 & 0.322731 & 0.31354 & 0.323429 & 0.020357 & 0.023274 & 0.018513 & 0.020547 & 0.021033 & 0.018831 & 0.022227 & 0.021955 & 0.023497 & 0.021793 \\
        & EDICT   & 0.245296 & 0.240349 & 0.244591 & 0.238628 & 0.255315 & 0.25793 & 0.233699 & 0.263676 & 0.249378 & 0.253245 & 0.041683 & 0.045929 & 0.042775 & 0.04398 & 0.033287 & 0.041811 & 0.044667 & 0.0373 & 0.046033 & 0.042176 \\
        & DDPM Inv.  & 0.297219 & 0.297912 & 0.30257 & 0.296602 & 0.30233 & 0.32085 & 0.305622 & 0.316406 & 0.301403 & 0.30635 & 0.011048 & 0.011731 & 0.008747 & 0.011406 & 0.009662 & 0.009197 & 0.010516 & 0.010077 & 0.010697 & 0.009748 \\
        & Direct Inv.  & 0.308072 & 0.307606 & 0.313697 & 0.305409 & 0.30894 & 0.322279 & 0.316558 & 0.32272 & 0.31353 & 0.32342 & 0.020357 & 0.023274 & 0.018514 & 0.020547 & 0.021033 & 0.018831 & 0.022227 & 0.021955 & 0.023497 & 0.021793 \\
        
        & \textbf{EtaInv (1)} & 0.3094754481954234 & 0.30763834081590175 & 0.3120006473734975 & 0.30612894501537086 & 0.30736390575766565 & 0.32162765935063364 & 0.3184932924807072 & 0.32267407327890396 & 0.31564259342849255 & 0.3232740178704262 & 0.022610991288508686 & 0.02564124616328627 & 0.020400327671086415 & 0.023243395966710524 & 0.023170134972315282 & 0.020761285437038167 & 0.023741583875380456 & 0.02353287994628772 & 0.02558025565231219 & 0.024572727060876785 \\
        & \textbf{EtaInv (2)} & 0.31437261988009724 & 0.3137755969539285 & 0.3139154825359583 & 0.30762694850564004 & 0.3072288006544113 & 0.3232583858072758 & 0.32299443148076534 & 0.32235062047839164 & 0.31541760731488466 & 0.32847115881741046 & 0.03147501499126 & 0.03603489915840328 & 0.03076102090999484 & 0.035501986916642636 & 0.03282403383636847 & 0.0309066888759844 & 0.03451595264486969 & 0.03196657916996628 & 0.03734000659314916 & 0.0376644987729378 \\
        & \textbf{EtaInv (3)} & 0.31574875861406326 & 0.3162227764725685 & 0.31690542846918107 & 0.30996477194130423 & 0.30985031351447107 & 0.3228080585598946 & 0.32518419772386553 & 0.33121442049741745 & 0.3216624580323696 & 0.3312120966613293 & 0.048629204370081426 & 0.05694386085961014 & 0.049023758713155986 & 0.05810008600819856 & 0.04752900272142142 & 0.04897871147841215 & 0.04965760502964258 & 0.048550447751767936 & 0.05316394686233252 & 0.05160329109057784 \\

        \midrule
        \multirow{10}{*}{Masa} & DDIM Inv.   & 0.298524 & 0.308096 & 0.311807 & 0.309586 & 0.309313 & 0.323071 & 0.300613 & 0.322153 & 0.303822 & 0.306657 & 0.076363 & 0.0765 & 0.067555 & 0.088841 & 0.065761 & 0.063603 & 0.079375 & 0.081611 & 0.079573 & 0.068984 \\
        & NTI     & 0.297125 & 0.302908 & 0.301317 & 0.298063 & 0.298899 & 0.317839 & 0.298424 & 0.311435 & 0.295757 & 0.299652 & 0.050931 & 0.051538 & 0.039783 & 0.0518 & 0.028902 & 0.029284 & 0.05202 & 0.042592 & 0.043949 & 0.040676 \\
        & NPI      & 0.291854 & 0.298781 & 0.29552 & 0.293385 & 0.296347 & 0.310881 & 0.285467 & 0.307449 & 0.290657 & 0.295969 & 0.048809 & 0.057715 & 0.042754 & 0.052415 & 0.029931 & 0.035059 & 0.044066 & 0.042032 & 0.040297 & 0.040394 \\
        & ProxNPI & 0.290803 & 0.298467 & 0.294953 & 0.294868 & 0.29683  & 0.311312 & 0.28554 & 0.307049  & 0.289264 & 0.293824 & 0.041484 & 0.052309 & 0.037137 & 0.048294 & 0.024138 & 0.032123 & 0.038355 & 0.03677 & 0.033487 & 0.033592 \\
        & NSI   & 0.292297 & 0.297471 & 0.303208 & 0.300085 & 0.299717 & 0.315236 & 0.298584 & 0.309353 & 0.296036 & 0.299578 & 0.039358 & 0.043739 & 0.032995 & 0.049973 & 0.030269 & 0.032229 & 0.035628 & 0.047428 & 0.040349 & 0.035471 \\
        & EDICT        & 0.287542 & 0.284458 & 0.302455 & 0.296313 & 0.298843 & 0.321136 & 0.299148 & 0.308138 & 0.297402 & 0.299274 & 0.008046 & 0.008271 & 0.006455 & 0.010752 & 0.007023 & 0.007703 & 0.006812 & 0.008294 & 0.00713 & 0.007269 \\
        & DDPM Inv. & 0.287163 & 0.285726 & 0.300875 & 0.295101 & 0.299325 & 0.320229 & 0.295838 & 0.305342 & 0.29482 & 0.297805 & 0.007727 & 0.009382 & 0.006074 & 0.010225 & 0.006157 & 0.006354 & 0.006262 & 0.006925 & 0.007129 & 0.006568 \\
        & Direct Inv.  & 0.296265 & 0.30301  & 0.306502 & 0.303756 & 0.301745 & 0.32037 & 0.300017 & 0.317445 & 0.298657 & 0.302186 & 0.033976 & 0.039918 & 0.028155 & 0.043967 & 0.027073 & 0.029447 & 0.028747 & 0.041718 & 0.032563 & 0.028447 \\
        
        & \textbf{EtaInv (1)} & 0.29811294270413263 & 0.3037538196891546 & 0.3051359286531806 & 0.3051352959126234 & 0.3019089061766863 & 0.3209380455315113 & 0.3023850996047258 & 0.31844585463404657 & 0.2983286209404469 & 0.30323502141982317 & 0.03640367245368127 & 0.042454510723473504 & 0.02876522419974208 & 0.0459367937175557 & 0.029499739105813206 & 0.03285879553295672 & 0.032200653688050807 & 0.044395982252899556 & 0.03776790829142555 & 0.03225996966939419 \\
        & \textbf{EtaInv (2)} & 0.3003685529742922 & 0.3073809051886201 & 0.30809705778956414 & 0.3041563708335161 & 0.30712787322700025 & 0.319422060996294 & 0.30297422520816325 & 0.3183720842003822 & 0.30154303219169376 & 0.30886702463030813 & 0.05153397250521396 & 0.05536038337741047 & 0.046386676467955114 & 0.06417977624805644 & 0.042859296500682834 & 0.047520422469824554 & 0.05240240152925253 & 0.053905405360274015 & 0.05480161677114666 & 0.04901187089271843 \\
        & \textbf{EtaInv (3)} & 0.3009531635258879 & 0.3080802982673049 & 0.3083322636783123 & 0.3067431280389428 & 0.3073941767215729 & 0.32096946984529495 & 0.3014859385788441 & 0.3188445925712585 & 0.3054265633225441 & 0.30819824021309616 & 0.07303433041088284 & 0.07607623809017242 & 0.06576247510965913 & 0.08403962370939552 & 0.060595654509961606 & 0.06438433551229536 & 0.07142013660632074 & 0.07564583676867187 & 0.07414130459073931 & 0.06689334658440202 \\

        \bottomrule
        \end{tabular}
        }%
        \label{tab:result_edit_cat}
\end{sidewaystable*}

%% file: tables/tab_appendix_gpt.tex
    \begin{table*}[h]
        \caption{VIEScore~\cite{ku2024viescoreexplainablemetricsconditional} for the PIE-Bench random (0) and change style (9) subsets. VIEScore ranges from 0 to 10 (higher is better) and provides three scores: overall, text-image alignment (align.), and structural similarity (sim.). For PtP and MasaCtrl, our method achieves state-of-the-art overall and alignment scores.}
        \centering
        \resizebox{1\textwidth}{!}{%
        \small
        \centering
        \setlength{\tabcolsep}{4.25pt}
        \sisetup{table-auto-round}
        \begin{tabular}{@{}cl*{6}{S[table-format=2.2,drop-exponent = true,fixed-exponent = 0,exponent-mode = fixed,]}@{}}%
        \toprule
        && \multicolumn{3}{c}{PIE-Bench Random} & \multicolumn{3}{c}{PIE-Bench Style}  \\ 
        \cmidrule(r){3-5}  \cmidrule(r){6-8}
        Editing & Inversion & {overall $\uparrow$} & {align. $\uparrow$} & {sim. $\uparrow$} & {overall $\uparrow$} & {align. $\uparrow$} & {sim. $\uparrow$} \\ \midrule
        \multirow{10}{*}{PtP}
 & DDIM Inv. \cite{song2020denoising} & 4.30597 & 4.776119 & 5.358209 & 1.717949 & 1.794872 & 3.102564 \\
 & Null-text Inv. \cite{mokady2023null} & 5.431655 & 5.827338 & 8.071942 & \cellcolor{mysilver}4.025 & \cellcolor{mysilver}4.1625 & 8.0125 \\
 & NPI \cite{miyake2023negative} & \cellcolor{mysilver}5.870504 & \cellcolor{mysilver}6.352518 & 7.834532 & 3.850000 & 4.087500 & 7.175000 \\
 & ProxNPI \cite{han2023improving} & 5.23741 & 5.647482 & 7.633094 & 3.175 & 3.325 & 7.5625 \\
 & EDICT \cite{wallace2023edict} & 1.482014 & 1.561151 & 8.064748 & 0.425 & 0.425 & \cellcolor{mygold}9.1875 \\
 & DDPM Inv. \cite{huberman2023edit} & 1.402878 & 1.482014 & \cellcolor{mygold}8.374101 & 0.375 & 0.375 & \cellcolor{mysilver}8.625 \\
 & Direct Inv. \cite{ju2023direct} & 5.271429 & 5.657143 & 8.078571 & 3.177215 & 3.240506 & 7.088608 \\
 & \textbf{Eta Inversion (1)} & 4.978571 & 5.342857 & \cellcolor{mysilver}8.128571 & 3.164557 & 3.227848 & 7.088608 \\
 & \textbf{Eta Inversion (2)} & 5.604317 & 6.071942 & 8.014388 & 3.481013 & 3.64557 & 6.696203 \\
 & \textbf{Eta Inversion (3)} & \cellcolor{mygold}6.208633 & \cellcolor{mygold}6.726619 & 7.654676 & \cellcolor{mygold}4.240506 & \cellcolor{mygold}4.544304 & 6.063291 \\
        \midrule
        \multirow{10}{*}{PnP}
 & DDIM Inv. \cite{song2020denoising} & 4.264286 & 4.985714 & 5.128571 & 4.250000 & 4.600000 & 5.425000 \\
 & Null-text Inv. \cite{mokady2023null} & \cellcolor{mysilver}6.021428571428571 & \cellcolor{mysilver}6.628571428571429 & 7.228571 & \cellcolor{mygold}5.9875 & \cellcolor{mysilver}6.3 & 7.800000 \\
 & NPI \cite{miyake2023negative} & \cellcolor{mygold}6.371428571428571 & \cellcolor{mygold}6.9071428571428575 & \cellcolor{mybronze}7.678571428571429 & \cellcolor{mysilver}5.9375 & \cellcolor{mygold}6.3375 & \cellcolor{mysilver}7.9625 \\
 & ProxNPI \cite{han2023improving} & 5.785714 & 6.285714 & \cellcolor{mysilver}7.828571428571428 & 4.700000 & 4.875000 & \cellcolor{mygold}8.1 \\
 & EDICT \cite{wallace2023edict} & 2.600000 & 2.871429 & 4.171429 & 2.350000 & 2.687500 & 3.900000 \\
 & DDPM Inv. \cite{huberman2023edit} & 4.121429 & 4.385714 & \cellcolor{mygold}8.492857142857142 & 2.025000 & 2.062500 & 7.475000 \\
 & Direct Inv. \cite{ju2023direct} & 5.042857 & 5.428571 & 7.464286 & 4.737500 & 4.925000 & \cellcolor{mybronze}7.9375 \\
 & \textbf{Eta Inversion (1)} & 5.057143 & 5.435714 & 7.178571 & 3.525000 & 3.625000 & 7.325000 \\
 & \textbf{Eta Inversion (2)} & 5.564286 & 5.985714 & 6.914286 & 4.437500 & 4.812500 & 6.237500 \\
 & \textbf{Eta Inversion (3)} & \cellcolor{mybronze}5.878571428571429 & \cellcolor{mybronze}6.485714285714286 & 6.628571 & \cellcolor{mybronze}5.3375 & \cellcolor{mybronze}5.775 & 6.550000 \\
        \midrule
        \multirow{10}{*}{Masa}
 & DDIM Inv. \cite{song2020denoising} & 2.546763 & 2.906475 & 3.928058 & 1.287500 & 1.387500 & 2.625000 \\
 & Null-text Inv. \cite{mokady2023null} & \cellcolor{mysilver}3.8642857142857143 & \cellcolor{mybronze}4.128571428571429 & 7.635714 & 1.837500 & 1.987500 & 6.812500 \\
 & NPI \cite{miyake2023negative} & 3.807143 & \cellcolor{mysilver}4.185714285714286 & 6.971429 & \cellcolor{mygold}2.7375 & \cellcolor{mygold}2.9125 & 6.612500 \\
 & ProxNPI \cite{han2023improving} & 3.650000 & 4.007143 & 7.371429 & 1.675000 & 1.737500 & 6.375000 \\
 & EDICT \cite{wallace2023edict} & 1.564286 & 1.628571 & \cellcolor{mysilver}8.635714285714286 & 0.569620 & 0.582278 & \cellcolor{mysilver}8.341772151898734 \\
 & DDPM Inv. \cite{huberman2023edit} & 1.307143 & 1.385714 & \cellcolor{mygold}8.907142857142857 & 0.425000 & 0.425000 & \cellcolor{mygold}8.7125 \\
 & Direct Inv. \cite{ju2023direct} & 2.607143 & 2.778571 & 7.478571 & 1.237500 & 1.287500 & 5.775000 \\
 & \textbf{Eta Inversion (1)} & 3.192857 & 3.328571 & \cellcolor{mybronze}8.042857142857143 & 1.175000 & 1.212500 & \cellcolor{mybronze}7.3 \\
 & \textbf{Eta Inversion (2)} & \cellcolor{mygold}4.357142857142857 & \cellcolor{mygold}4.5285714285714285 & 7.478571 & \cellcolor{mybronze}2.15 & \cellcolor{mybronze}2.175 & 6.025000 \\
 & \textbf{Eta Inversion (3)} & \cellcolor{mybronze}3.8285714285714287 & 4.107143 & 6.657143 & \cellcolor{mysilver}2.475 & \cellcolor{mysilver}2.55 & 5.550000 \\
        \bottomrule
        \end{tabular}
        }
        \label{tab:result_appendix_gpt}
        \end{table*}

%% file: tables/tab_appendix_user.tex
\begin{table}[h]
        \caption{Human evaluation on PIE-Bench. We conducted 740 comparisons on PIE-Bench's random (0) and change style (9) subsets, where participants were asked to choose if Direct Inversion's output is better, if Eta Inversion's output is better, or if it is a tie. In non-tie cases, Eta Inversion was preferred approximately 2 to 3.5 times more than Direct Inversion.}
        \centering
        \resizebox{1\textwidth}{!}{%
        \small
        \centering
        \setlength{\tabcolsep}{4.25pt}
        \renewcommand{\arraystretch}{0.1}
\begin{tabular}{l|ccc}
        \toprule
        Editing method: PtP &Tie & Direct Inversion & \textbf{Eta Inversion} \\
        \midrule
        PIE-Bench (random) & 69.76\% & 10.70\% & \cellcolor{mygold}19.53\% \\
        PIE-Bench (change style) & 52.50\% & 10.63\% & \cellcolor{mygold}36.86\% \\
         \bottomrule
        \end{tabular}
        }%
        \label{table_user_study}
        \end{table}

%% file: tables/tab_appendix_ti2i.tex
\begin{table*}[h]
        \caption{ImageNet-R-TI2I \cite{tumanyan2023plug} evaluation with various metrics. For all three editing methods, our approach achieves the best CLIP text-image and text-caption metrics.
        }
        \centering
        \resizebox{1.0\textwidth}{!}{%
        \small
        \centering
        \setlength{\tabcolsep}{4.25pt}
        \sisetup{table-auto-round}
        \begin{tabular}{@{}cl*{8}{S[table-format=2.2,drop-exponent = true,fixed-exponent = -2,exponent-mode = fixed,]}@{}}%
        \toprule
        && \multicolumn{8}{c}{Metric $(\times {10}^{2})$ }  \\ 
        \cmidrule(r){3-10}
        &&\multicolumn{4}{c}{Text-Image Alignment (CLIP)} & \multicolumn{4}{c}{Structural Similarity} \\ 
        \cmidrule(r){3-6} \cmidrule(lr){7-10} 
        Editing & Inversion & {text-img $\uparrow$} & {text-cap $\uparrow$} & {directional $\uparrow$} & {acc $\uparrow$} & {DINOv1 $\downarrow$} & {DINOv2 $\downarrow$} & {LPIPS $\downarrow$} & {MS-SSIM $\uparrow$} \\ \midrule
        \multirow{10}{*}{PtP}
        
 & DDIM Inv. \cite{song2020denoising} & 0.300632 & 0.698667 & 0.024871 & \cellcolor{mygold}0.9888888888888889 & 0.083532 & 0.010722 & 0.504776 & 0.551669 \\
 & Null-text Inv. \cite{mokady2023null} & 0.303227 & 0.705829 & \cellcolor{mybronze}0.054039139332922384 & \cellcolor{mybronze}0.9444444444444444 & 0.021777 & 0.005250 & 0.261035 & 0.815069 \\
 & NPI \cite{miyake2023negative} & 0.303377 & 0.684326 & \cellcolor{mygold}0.0615244996902119 & \cellcolor{mybronze}0.9444444444444444 & 0.032457 & 0.006280 & 0.287074 & 0.793274 \\
 & ProxNPI \cite{han2023improving} & 0.299941 & 0.690854 & \cellcolor{mysilver}0.05412450299546537 & 0.922222 & 0.030834 & 0.005934 & 0.256126 & 0.809384 \\
 & EDICT \cite{wallace2023edict} & 0.293859 & 0.676906 & 0.026189 & \cellcolor{mybronze}0.9444444444444444 & \cellcolor{mysilver}0.009404323671737479 & \cellcolor{mysilver}0.003173777186829183 & \cellcolor{mysilver}0.12395210979092453 & \cellcolor{mysilver}0.8985121422343784 \\
 & DDPM Inv. \cite{huberman2023edit} & 0.289843 & 0.671433 & 0.006860 & \cellcolor{mysilver}0.9777777777777777 & \cellcolor{mygold}0.00460396135588073 & \cellcolor{mygold}0.002229099212369571 & \cellcolor{mygold}0.07873167813652092 & \cellcolor{mygold}0.9222212546401554 \\
 & Direct Inv. \cite{ju2023direct} & 0.303374 & 0.702348 & 0.032677 & \cellcolor{mysilver}0.9777777777777777 & \cellcolor{mybronze}0.018036740557808014 & \cellcolor{mybronze}0.004538325217112692 & \cellcolor{mybronze}0.21982821428941357 & \cellcolor{mybronze}0.8298795898755391 \\
 & \textbf{EtaInv (1)} & \cellcolor{mybronze}0.3054346842898263 & \cellcolor{mybronze}0.7088675684399075 & 0.034022 & \cellcolor{mysilver}0.9777777777777777 & 0.018853 & 0.004640 & 0.226764 & 0.821688 \\
 & \textbf{EtaInv (2)} & \cellcolor{mysilver}0.30942235456572637 & \cellcolor{mysilver}0.7117388566335042 & 0.042446 & \cellcolor{mybronze}0.9444444444444444 & 0.037025 & 0.006706 & 0.386631 & 0.648443 \\
 & \textbf{EtaInv (3)} & \cellcolor{mygold}0.31441925532288023 & \cellcolor{mygold}0.7301973561445873 & 0.052512 & \cellcolor{mygold}0.9888888888888889 & 0.050549 & 0.008064 & 0.456492 & 0.564557 \\
        
        \midrule
        \multirow{10}{*}{PnP}
 & DDIM Inv. \cite{song2020denoising} & 0.288298 & 0.671258 & \cellcolor{mysilver}0.07907485843429135 & 0.877778 & 0.074509 & 0.011089 & 0.460250 & 0.624951 \\
 & Null-text Inv. \cite{mokady2023null} & 0.303485 & 0.707477 & \cellcolor{mybronze}0.0749444331922051 & 0.900000 & 0.047465 & 0.007860 & 0.418736 & 0.681922 \\
 & NPI \cite{miyake2023negative} & 0.308220 & 0.707688 & \cellcolor{mygold}0.08110463526358622 & 0.922222 & 0.042287 & 0.007434 & 0.376241 & 0.726913 \\
 & ProxNPI \cite{han2023improving} & 0.302026 & 0.692809 & 0.062840 & \cellcolor{mybronze}0.9666666666666667 & 0.036819 & 0.006430 & \cellcolor{mybronze}0.31199564006593494 & \cellcolor{mybronze}0.7682385550604927 \\
 & EDICT \cite{wallace2023edict} & 0.234416 & 0.581498 & 0.056548 & 0.577778 & 0.065405 & 0.009536 & 0.432080 & 0.628998 \\
 & DDPM Inv. \cite{huberman2023edit} & 0.299064 & 0.691761 & 0.031989 & 0.933333 & \cellcolor{mygold}0.014282594046865901 & \cellcolor{mygold}0.003969529303463383 & \cellcolor{mygold}0.16628073079304562 & \cellcolor{mygold}0.851397309700648 \\
 & Direct Inv. \cite{ju2023direct} & 0.308042 & 0.713955 & 0.053568 & \cellcolor{mygold}0.9888888888888889 & \cellcolor{mysilver}0.030206094971961445 & \cellcolor{mysilver}0.005974539100295968 & \cellcolor{mysilver}0.3070564554797279 & \cellcolor{mysilver}0.7722945113976797 \\
 & \textbf{EtaInv (1)} & \cellcolor{mybronze}0.30969284110599093 & \cellcolor{mybronze}0.7164372139506869 & 0.055690 & \cellcolor{mysilver}0.9777777777777777 & \cellcolor{mybronze}0.03297534180391166 & \cellcolor{mybronze}0.00621007154436989 & 0.325032 & 0.751517 \\
 & \textbf{EtaInv (2)} & \cellcolor{mysilver}0.3127148946126302 & \cellcolor{mysilver}0.7330768446127574 & 0.068616 & \cellcolor{mysilver}0.9777777777777777 & 0.052363 & 0.008610 & 0.478414 & 0.547032 \\
 & \textbf{EtaInv (3)} & \cellcolor{mygold}0.3161194344361623 & \cellcolor{mygold}0.7384766393237644 & 0.073254 & \cellcolor{mysilver}0.9777777777777777 & 0.058571 & 0.009597 & 0.535969 & 0.447680 \\
        
        \midrule
        \multirow{10}{*}{Masa}
         & DDIM Inv. \cite{song2020denoising} & \cellcolor{mybronze}0.2967420544889238 & \cellcolor{mysilver}0.6881464123725891 & 0.012349 & \cellcolor{mysilver}0.9888888888888889 & 0.097966 & 0.011802 & 0.539644 & 0.516811 \\
 & Null-text Inv. \cite{mokady2023null} & 0.296128 & 0.657415 & \cellcolor{mybronze}0.029650776284850307 & 0.900000 & 0.067908 & 0.009343 & 0.345713 & 0.673883 \\
 & NPI \cite{miyake2023negative} & 0.280663 & 0.639803 & \cellcolor{mygold}0.03908606185319109 & 0.844444 & 0.073938 & 0.010173 & 0.365953 & 0.670935 \\
 & ProxNPI \cite{han2023improving} & 0.274236 & 0.633457 & 0.027010 & 0.877778 & 0.064165 & 0.008982 & 0.322750 & 0.709539 \\
 & EDICT \cite{wallace2023edict} & 0.293836 & \cellcolor{mybronze}0.685013359453943 & \cellcolor{mysilver}0.03328529838472605 & 0.955556 & \cellcolor{mysilver}0.019621735428356463 & \cellcolor{mysilver}0.0044216456635492955 & \cellcolor{mysilver}0.16925901296652027 & \cellcolor{mysilver}0.8567955798572964 \\
 & DDPM Inv. \cite{huberman2023edit} & 0.292804 & 0.679085 & 0.009276 & \cellcolor{mygold}1.0 & \cellcolor{mygold}0.009276034831742032 & \cellcolor{mygold}0.003063239445651157 & \cellcolor{mygold}0.11149187653015057 & \cellcolor{mygold}0.8908810774485271 \\
 & Direct Inv. \cite{ju2023direct} & 0.294003 & 0.670548 & 0.005269 & \cellcolor{mybronze}0.9777777777777777 & \cellcolor{mybronze}0.056977016923742164 & \cellcolor{mybronze}0.007977131089299089 & \cellcolor{mybronze}0.30589368748995993 & \cellcolor{mybronze}0.7167490538623598 \\
 & \textbf{EtaInv (1)} & 0.294922 & 0.676686 & 0.006206 & \cellcolor{mybronze}0.9777777777777777 & 0.059918 & 0.008269 & 0.319825 & 0.697101 \\
 & \textbf{EtaInv (2)} & \cellcolor{mysilver}0.297080311510298 & 0.678817 & 0.009184 & \cellcolor{mysilver}0.9888888888888889 & 0.078229 & 0.010496 & 0.429114 & 0.520659 \\
 & \textbf{EtaInv (3)} & \cellcolor{mygold}0.29915328125158946 & \cellcolor{mygold}0.6885678966840109 & 0.011879 & \cellcolor{mybronze}0.9777777777777777 & 0.088821 & 0.011894 & 0.482996 & 0.422169 \\
        
        \bottomrule
        \end{tabular}
        }%
        \label{tab:result_appendix_ti2i}
        \end{table*}

%% file: tables/tab_appendix_rec.tex
\begin{table*}[h]
        \caption{Reconstruction benchmark of inversion methods on the COCO training set \cite{chen2015microsoft} (left). Inference time of inversion methods (right).
        Our method matches VAE reconstruction metrics, thus demonstrating perfect reconstruction. Additionally, our inference time is only slightly higher than DDIM Inversion for inversion (Inv.) and the three editing methods PtP, PnP, and MasaCtrl. Inference time is measured on an NVIDIA V100.}
        \centering
        \resizebox{1\linewidth}{!}{%
        \small
        \centering
        \setlength{\tabcolsep}{5.5pt}
        \sisetup{table-auto-round}
        \begin{tabular}{@{}l|*{1}{S[table-format=2.1]}*{2}{S[table-format=2.3]}|*{4}{S[table-format=2.1]}@{}}%
        \toprule
        &\multicolumn{3}{c}{Reconstruction error}&\multicolumn{4}{c}{Inference time (s)}\\ 
        \cmidrule(lr){2-4} \cmidrule(lr){5-8} 
        Method & {PSNR $\uparrow$} & {LPIPS $\downarrow$} & {SSIM $\uparrow$} & {Inv.} & {PtP} & {PnP} & {Masa} \\ \midrule
        DDIM Inv. \cite{song2020denoising} & 14.320371 & 0.499998 & 0.469395 & 22.060387 & 23.506866 & 19.275929 & 25.593275 \\
        Null-text Inv. \cite{mokady2023null} & 26.329788 & 0.072109 & 0.744787 & 200.926327 & 202.677380 & 198.365538 & 204.591798 \\
        NPI \cite{miyake2023negative} & 23.951053 & 0.147295 & 0.696579 & 22.116447 & 23.553465 & 19.259445 & 25.550220 \\
        ProxNPI \cite{han2023improving} & 26.635513 & 0.067292 & 0.75056 & 25.720187 & 27.288240 & 22.760977 & 29.327472 \\
        EDICT \cite{wallace2023edict} & 26.635451 & 0.067293 & 0.750558 & 43.987208 & 53.930088 & 45.082623 & 55.824568 \\
        DDPM Inv. \cite{huberman2023edit} & 26.635513 & 0.067292 & 0.75056 & 32.115896 & 41.027418 & 35.531602 & 43.045689 \\
        Direct Inv. \cite{ju2023direct} & 26.635513 & 0.067292 & 0.75056 & 22.139811 & 23.517492 & 19.276044 & 25.565936 \\
        \textbf{Eta Inversion} & 26.635513 & 0.067292 & 0.75056 & 22.752689 & 24.319281 & 19.871074 & 26.356343 \\
        \textbf{Eta Inversion w/o mask} & 26.635513 & 0.067292 & 0.75056 & 22.265121 & 23.895624 & 19.429463 & 25.906187 \\
        \midrule
        VAE Reconstruction & 26.635513 & 0.067292 & 0.75056 &{-}&{-}&{-}&{-}\\ 
         \bottomrule
        \end{tabular}
        }%
        \label{tab:result_appendix_rec}
        \end{table*}

%% file: figs/fig_ptp_result.tex
\begin{figure*}[t]
    \captionsetup[subfigure]{labelformat=empty}
    \centering

\begin{minipage}{\linewidth}
        \centering
        {\footnotesize \textit{“a \textbf{kitten} walking through the grass”} $\rightarrow$ \textit{“a \textbf{duck} walking through the grass”}}
        \medskip
        \end{minipage}

\subfloat[]{\includegraphics[width=.155\linewidth]{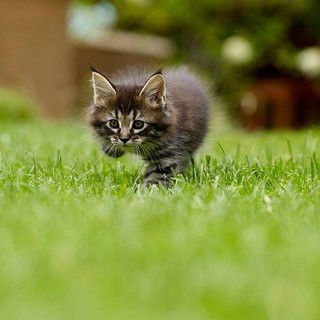}}\hspace*{-1pt}
\subfloat[]{\includegraphics[width=.155\linewidth]{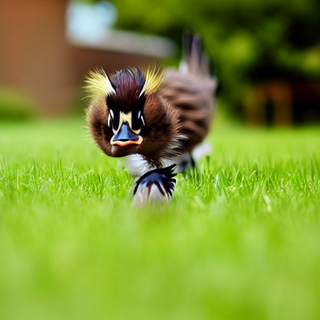}}\hspace*{-1pt}
\subfloat[]{\includegraphics[width=.155\linewidth]{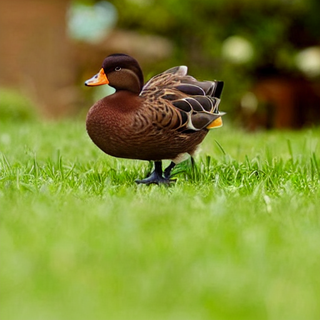}}\hspace*{-1pt}
\subfloat[]{\includegraphics[width=.155\linewidth]{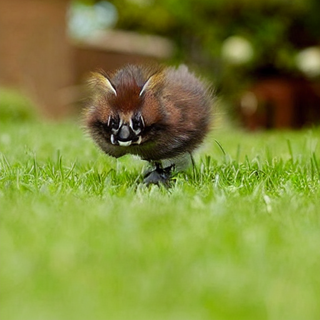}}\hspace*{-1pt}
\subfloat[]{\includegraphics[width=.155\linewidth]{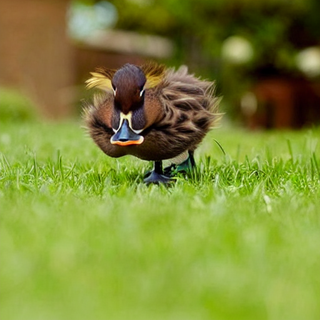}}\hspace*{-1pt}
\subfloat[]{\includegraphics[width=.155\linewidth]{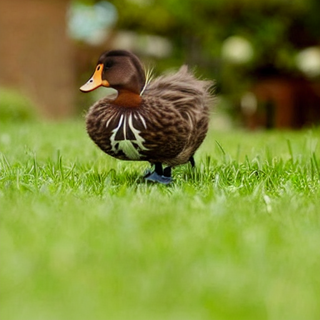}}

\begin{minipage}{\linewidth}
        \centering
        {\footnotesize \textit{“painting of a \textbf{shepherd} dog sitting in a laundry room next to a washing machine”} $\rightarrow$ \textit{“painting of a \textbf{poodle} dog sitting in a laundry room next to a washing machine”}}
        \medskip
        \end{minipage}

\subfloat[]{\includegraphics[width=.155\linewidth]{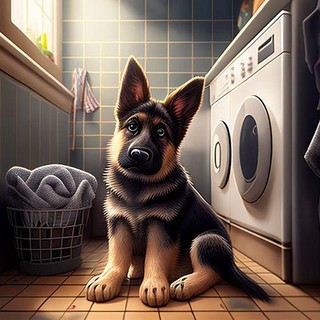}}\hspace*{-1pt}
\subfloat[]{\includegraphics[width=.155\linewidth]{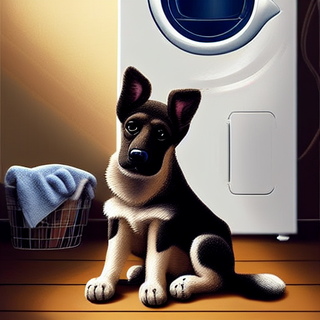}}\hspace*{-1pt}
\subfloat[]{\includegraphics[width=.155\linewidth]{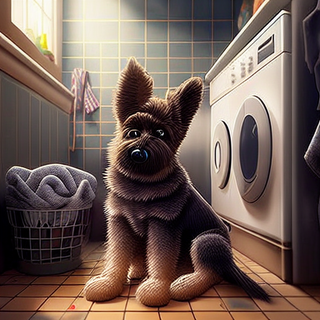}}\hspace*{-1pt}
\subfloat[]{\includegraphics[width=.155\linewidth]{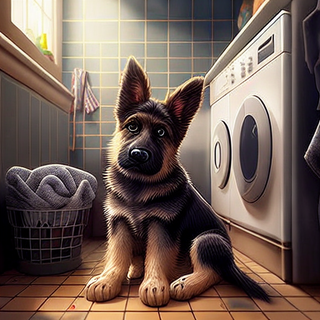}}\hspace*{-1pt}
\subfloat[]{\includegraphics[width=.155\linewidth]{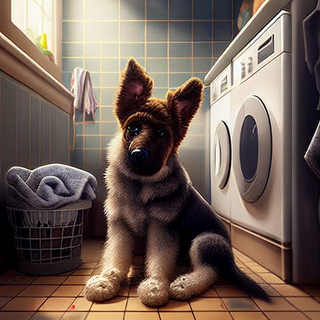}}\hspace*{-1pt}
\subfloat[]{\includegraphics[width=.155\linewidth]{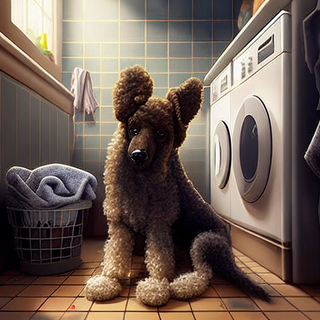}}

\begin{minipage}{\linewidth}
        \centering
        {\footnotesize \textit{“a detailed oil painting of a \textbf{calm} beautiful woman with stars in her hair”} $\rightarrow$ \textit{“a detailed oil painting of a \textbf{laughing} beautiful woman with stars in her hair”}}
        \medskip
        \end{minipage}

\subfloat[]{\includegraphics[width=.155\linewidth]{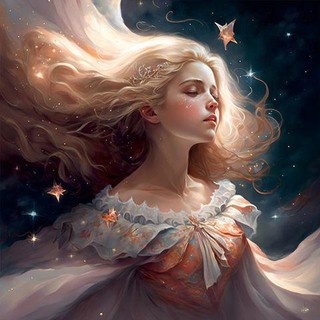}}\hspace*{-1pt}
\subfloat[]{\includegraphics[width=.155\linewidth]{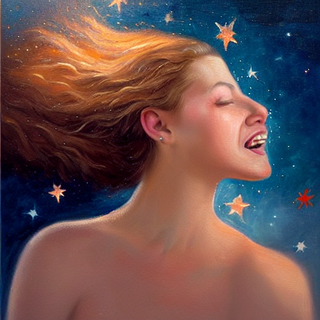}}\hspace*{-1pt}
\subfloat[]{\includegraphics[width=.155\linewidth]{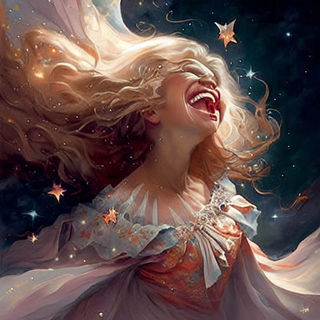}}\hspace*{-1pt}
\subfloat[]{\includegraphics[width=.155\linewidth]{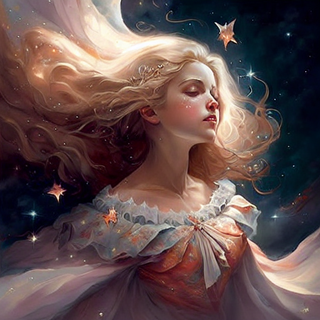}}\hspace*{-1pt}
\subfloat[]{\includegraphics[width=.155\linewidth]{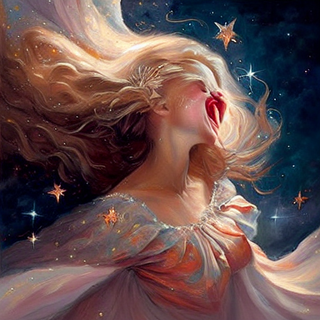}}\hspace*{-1pt}
\subfloat[]{\includegraphics[width=.155\linewidth]{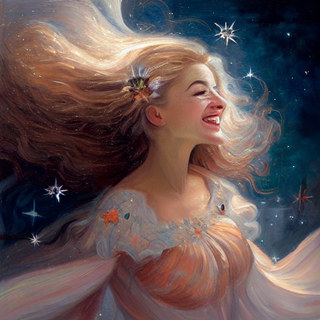}}

\begin{minipage}{\linewidth}
        \centering
        {\footnotesize \textit{“a \textbf{cat} sitting next to a mirror”} $\rightarrow$ \textit{“a \textbf{tiger} sitting next to a mirror”}}
        \medskip
        \end{minipage}

\subfloat[]{\includegraphics[width=.155\linewidth]{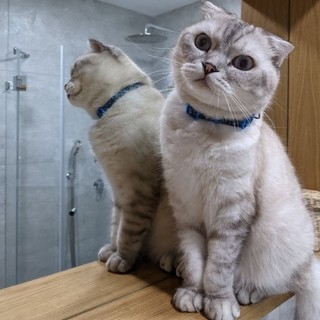}}\hspace*{-1pt}
\subfloat[]{\includegraphics[width=.155\linewidth]{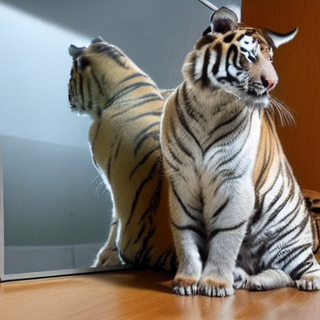}}\hspace*{-1pt}
\subfloat[]{\includegraphics[width=.155\linewidth]{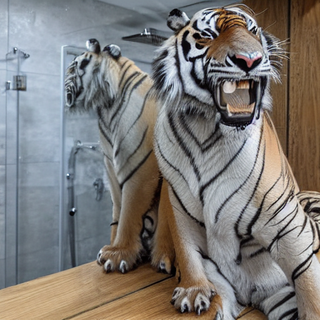}}\hspace*{-1pt}
\subfloat[]{\includegraphics[width=.155\linewidth]{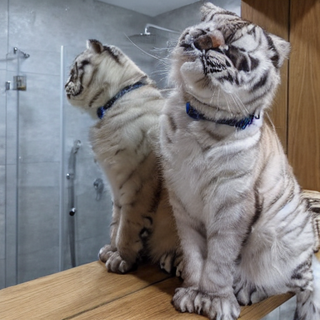}}\hspace*{-1pt}
\subfloat[]{\includegraphics[width=.155\linewidth]{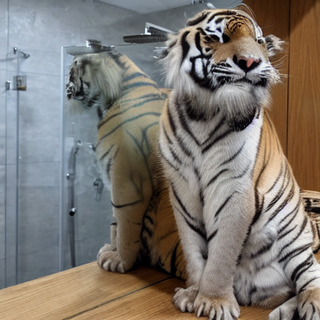}}\hspace*{-1pt}
\subfloat[]{\includegraphics[width=.155\linewidth]{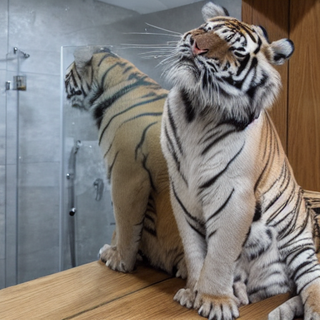}}

\begin{minipage}{\linewidth}
        \centering
        {\footnotesize \textit{“a woman in a black bikini top and yoga pants is meditating”} $\rightarrow$ \textit{“a \textbf{wax statue of} woman in a black bikini top and yoga pants is meditating”}}
        \medskip
        \end{minipage}

\subfloat[]{\includegraphics[width=.155\linewidth]{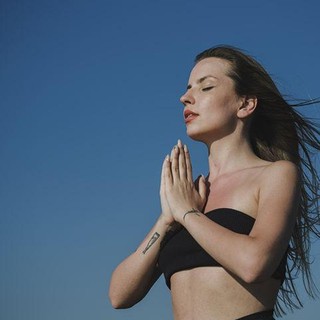}}\hspace*{-1pt}
\subfloat[]{\includegraphics[width=.155\linewidth]{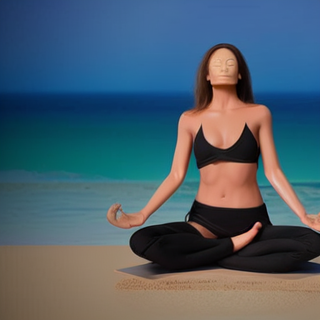}}\hspace*{-1pt}
\subfloat[]{\includegraphics[width=.155\linewidth]{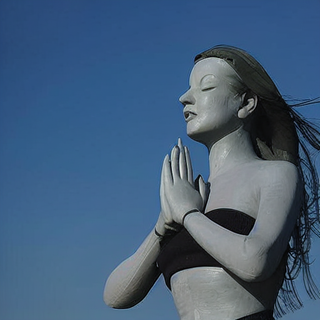}}\hspace*{-1pt}
\subfloat[]{\includegraphics[width=.155\linewidth]{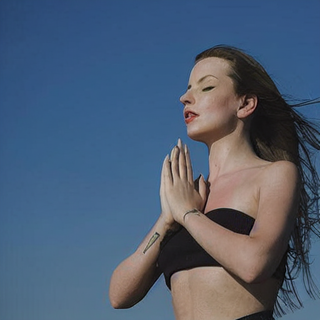}}\hspace*{-1pt}
\subfloat[]{\includegraphics[width=.155\linewidth]{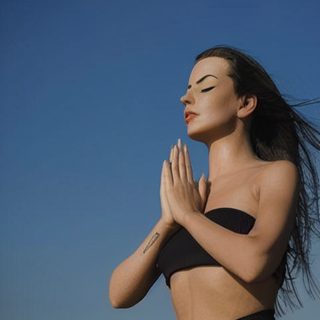}}\hspace*{-1pt}
\subfloat[]{\includegraphics[width=.155\linewidth]{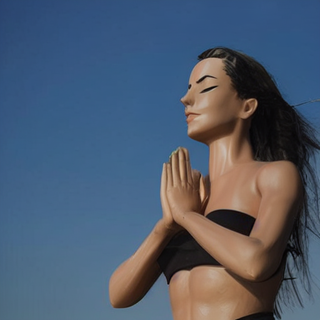}}

\begin{minipage}{\linewidth}
        \centering
        {\footnotesize \textit{“a slanted mountain bicycle on the road in front of a building”} $\rightarrow$ \textit{“a slanted \textbf{rusty} mountain bicycle on the road in front of a building”}}
        \medskip
        \end{minipage}

\subfloat[Source]{\includegraphics[width=.155\linewidth]{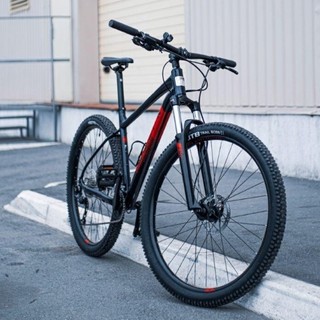}}\hspace*{-1pt}
\subfloat[DDIM~Inv~\cite{song2020denoising}]{\includegraphics[width=.155\linewidth]{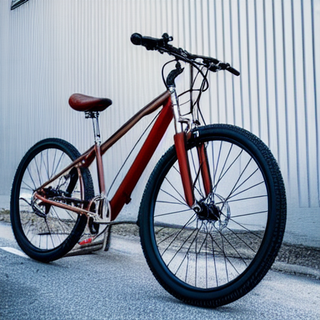}}\hspace*{-1pt}
\subfloat[NTI~\cite{mokady2023null}]{\includegraphics[width=.155\linewidth]{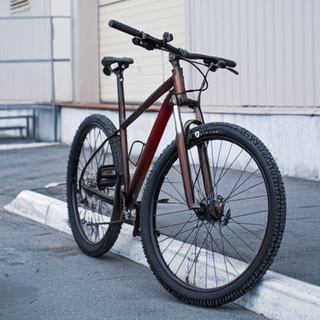}}\hspace*{-1pt}
\subfloat[EDICT~\cite{wallace2023edict}]{\includegraphics[width=.155\linewidth]{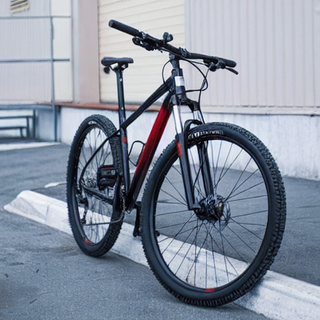}}\hspace*{-1pt}
\subfloat[Dir. Inv.~\cite{ju2023direct}]{\includegraphics[width=.155\linewidth]{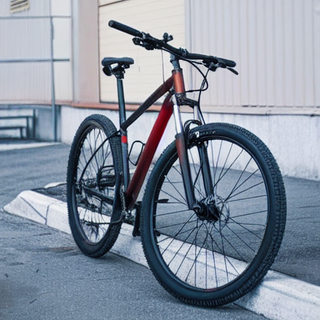}}\hspace*{-1pt}
\subfloat[\textbf{EtaInv~(2)}]{\includegraphics[width=.155\linewidth]{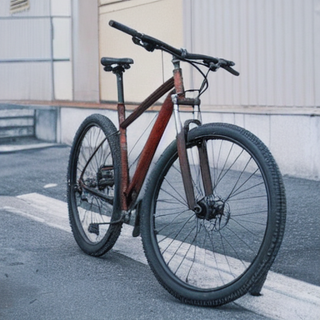}}

    \caption{Additional qualitative results for PtP editing.}
    \label{fig:ptp}
    \end{figure*}

%% file: figs/fig_pnp_result.tex
\begin{figure*}[t]
    \captionsetup[subfigure]{labelformat=empty}
    \centering

\begin{minipage}{\linewidth}
        \centering
        {\footnotesize \textit{“a man wearing a tie”} $\rightarrow$ \textit{“a man wearing a \textbf{black and yellow stripes} tie ”}}
        \medskip
        \end{minipage}

\subfloat[]{\includegraphics[width=.155\linewidth]{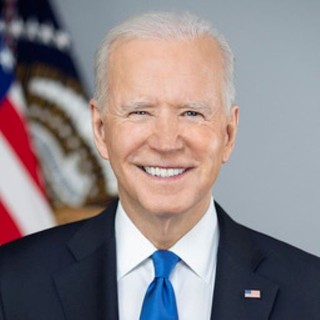}}\hspace*{-1pt}
\subfloat[]{\includegraphics[width=.155\linewidth]{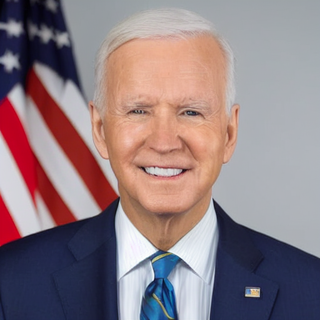}}\hspace*{-1pt}
\subfloat[]{\includegraphics[width=.155\linewidth]{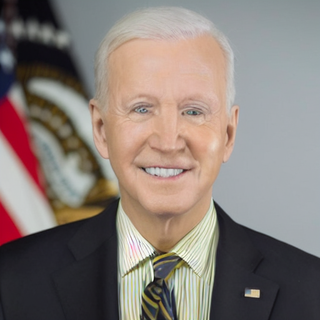}}\hspace*{-1pt}
\subfloat[]{\includegraphics[width=.155\linewidth]{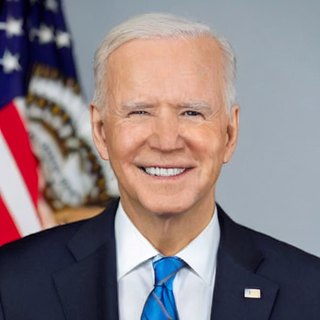}}\hspace*{-1pt}
\subfloat[]{\includegraphics[width=.155\linewidth]{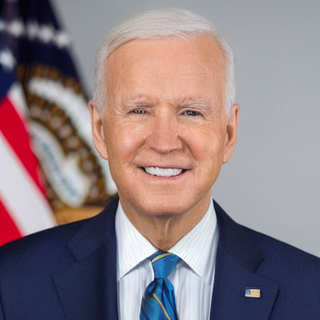}}\hspace*{-1pt}
\subfloat[]{\includegraphics[width=.155\linewidth]{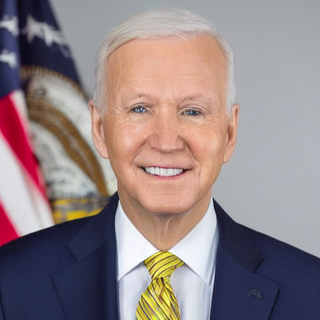}}

\begin{minipage}{\linewidth}
        \centering
        {\footnotesize \textit{“a woman in front of a glowing yellow light”} $\rightarrow$ \textit{“a woman \textbf{riding a lion} in front of a glowing yellow light”}}
        \medskip
        \end{minipage}

\subfloat[]{\includegraphics[width=.155\linewidth]{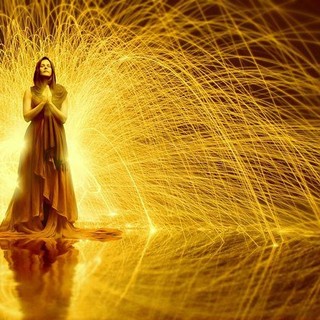}}\hspace*{-1pt}
\subfloat[]{\includegraphics[width=.155\linewidth]{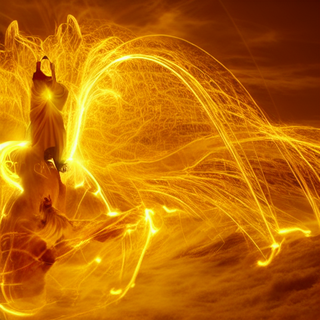}}\hspace*{-1pt}
\subfloat[]{\includegraphics[width=.155\linewidth]{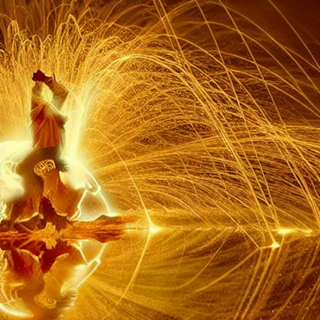}}\hspace*{-1pt}
\subfloat[]{\includegraphics[width=.155\linewidth]{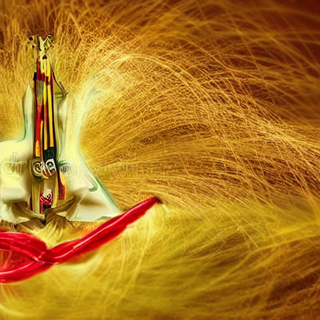}}\hspace*{-1pt}
\subfloat[]{\includegraphics[width=.155\linewidth]{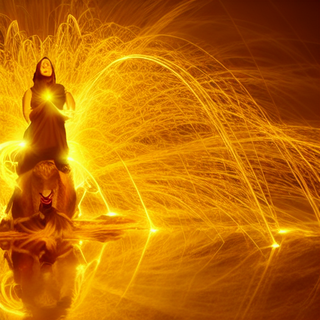}}\hspace*{-1pt}
\subfloat[]{\includegraphics[width=.155\linewidth]{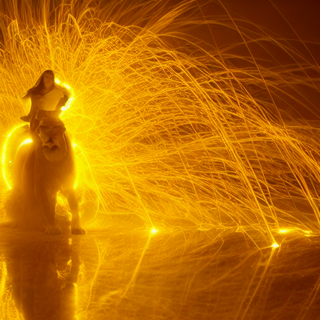}}

\begin{minipage}{\linewidth}
        \centering
        {\footnotesize \textit{“a basket of books and a \textbf{cup}”} $\rightarrow$ \textit{“a basket of books and a \textbf{candle}”}}
        \medskip
        \end{minipage}

\subfloat[]{\includegraphics[width=.155\linewidth]{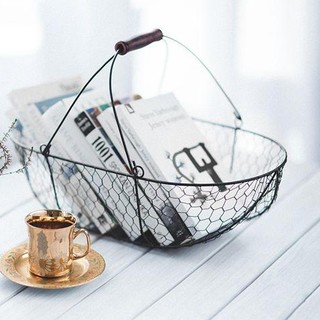}}\hspace*{-1pt}
\subfloat[]{\includegraphics[width=.155\linewidth]{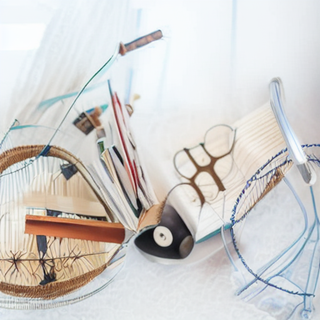}}\hspace*{-1pt}
\subfloat[]{\includegraphics[width=.155\linewidth]{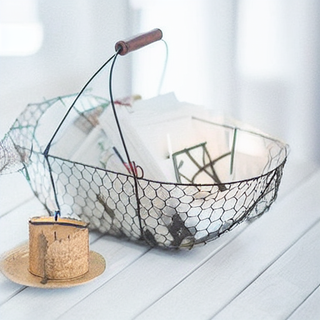}}\hspace*{-1pt}
\subfloat[]{\includegraphics[width=.155\linewidth]{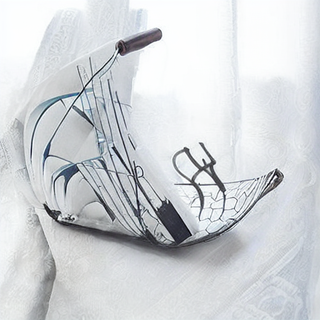}}\hspace*{-1pt}
\subfloat[]{\includegraphics[width=.155\linewidth]{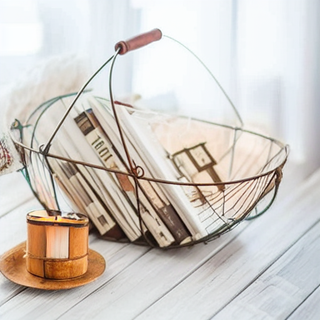}}\hspace*{-1pt}
\subfloat[]{\includegraphics[width=.155\linewidth]{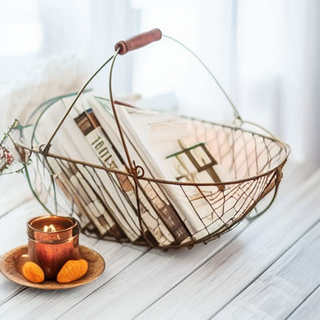}}

\begin{minipage}{\linewidth}
        \centering
        {\footnotesize \textit{“two \textbf{red and white} toy gnomes are sitting on a snow covered surface”} $\rightarrow$ \textit{“two \textbf{blue and green} toy gnomes are sitting on a snow covered surface”}}
        \medskip
        \end{minipage}

\subfloat[]{\includegraphics[width=.155\linewidth]{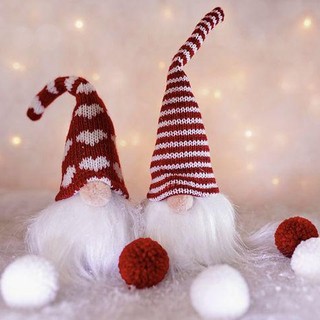}}\hspace*{-1pt}
\subfloat[]{\includegraphics[width=.155\linewidth]{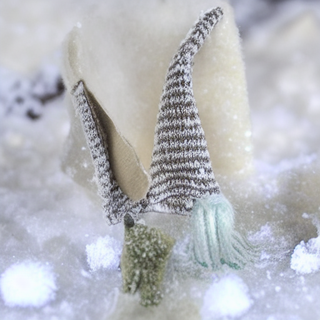}}\hspace*{-1pt}
\subfloat[]{\includegraphics[width=.155\linewidth]{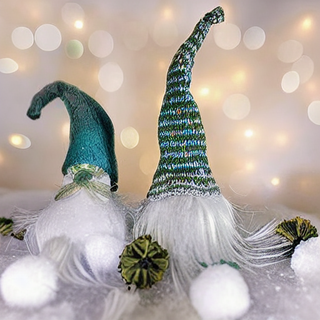}}\hspace*{-1pt}
\subfloat[]{\includegraphics[width=.155\linewidth]{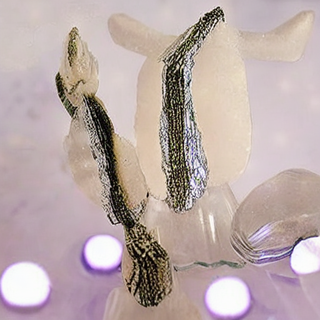}}\hspace*{-1pt}
\subfloat[]{\includegraphics[width=.155\linewidth]{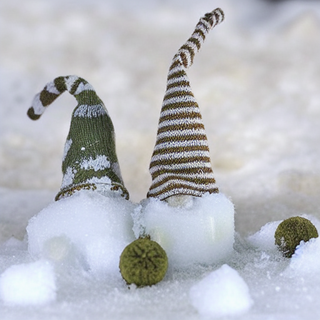}}\hspace*{-1pt}
\subfloat[]{\includegraphics[width=.155\linewidth]{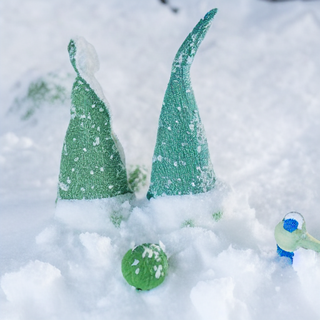}}

\begin{minipage}{\linewidth}
        \centering
        {\footnotesize \textit{“a view of the mountains covered in \textbf{snow}”} $\rightarrow$ \textit{“a view of the mountains covered in \textbf{leaves}”}}
        \medskip
        \end{minipage}

\subfloat[]{\includegraphics[width=.155\linewidth]{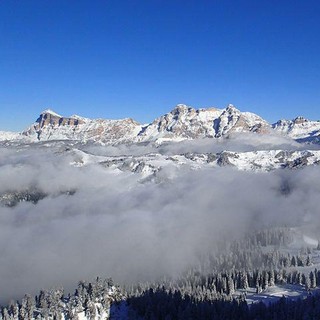}}\hspace*{-1pt}
\subfloat[]{\includegraphics[width=.155\linewidth]{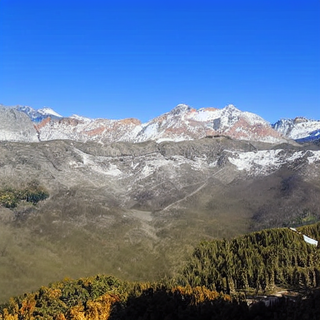}}\hspace*{-1pt}
\subfloat[]{\includegraphics[width=.155\linewidth]{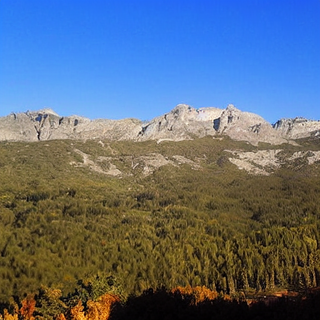}}\hspace*{-1pt}
\subfloat[]{\includegraphics[width=.155\linewidth]{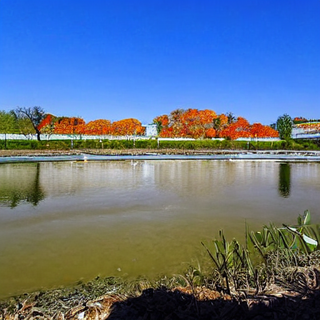}}\hspace*{-1pt}
\subfloat[]{\includegraphics[width=.155\linewidth]{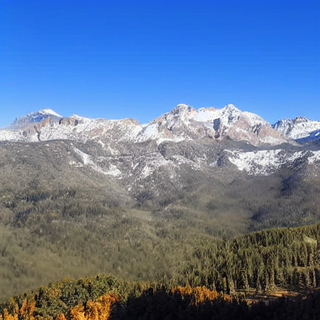}}\hspace*{-1pt}
\subfloat[]{\includegraphics[width=.155\linewidth]{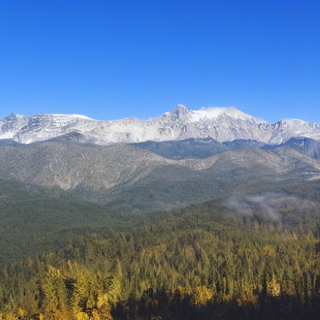}}

\begin{minipage}{\linewidth}
        \centering
        {\footnotesize \textit{“a boat is docked on a lake in the \textbf{heavy fog}”} $\rightarrow$ \textit{“a boat is docked on a lake in the \textbf{sunny day}”}}
        \medskip
        \end{minipage}

\subfloat[Source]{\includegraphics[width=.155\linewidth]{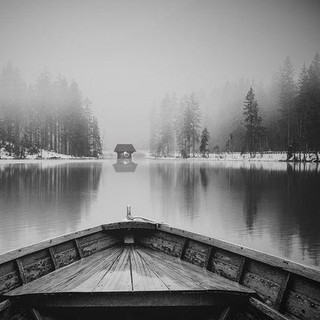}}\hspace*{-1pt}
\subfloat[DDIM~Inv~\cite{song2020denoising}]{\includegraphics[width=.155\linewidth]{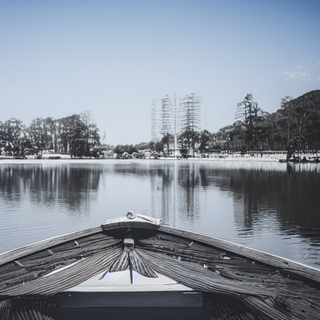}}\hspace*{-1pt}
\subfloat[NTI~\cite{mokady2023null}]{\includegraphics[width=.155\linewidth]{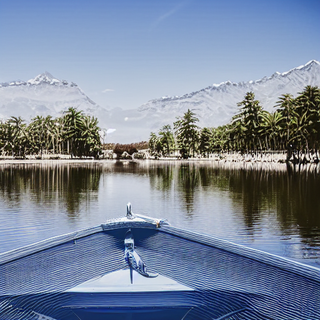}}\hspace*{-1pt}
\subfloat[EDICT~\cite{wallace2023edict}]{\includegraphics[width=.155\linewidth]{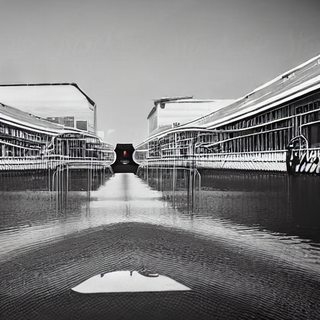}}\hspace*{-1pt}
\subfloat[Dir. Inv.~\cite{ju2023direct}]{\includegraphics[width=.155\linewidth]{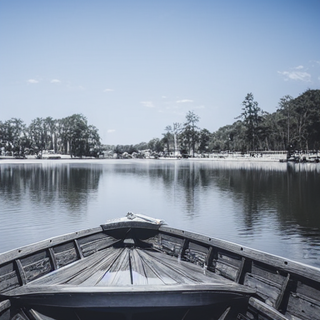}}\hspace*{-1pt}
\subfloat[\textbf{EtaInv~(2)}]{\includegraphics[width=.155\linewidth]{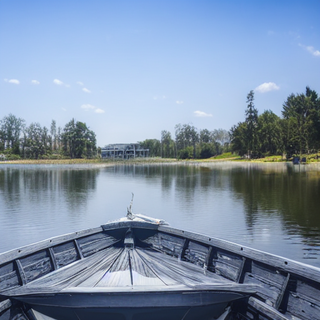}}

    \caption{Additional qualitative results for PnP editing.}
    \label{fig:pnp}
    \end{figure*}

%% file: figs/fig_masa_result.tex
\begin{figure*}[t]
    \captionsetup[subfigure]{labelformat=empty}
    \centering

\begin{minipage}{\linewidth}
        \centering
        {\footnotesize \textit{“the christmas illustration of a santa's \textbf{laughing} face”} $\rightarrow$ \textit{“the christmas illustration of a santa's \textbf{angry} face”}}
        \medskip
        \end{minipage}

\subfloat[]{\includegraphics[width=.155\linewidth]{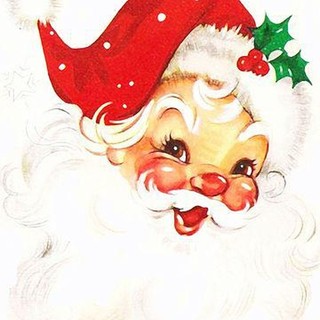}}\hspace*{-1pt}
\subfloat[]{\includegraphics[width=.155\linewidth]{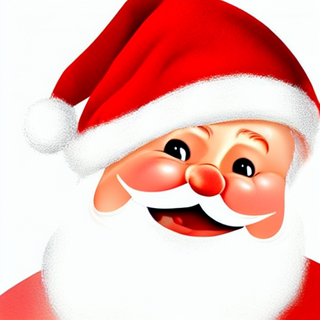}}\hspace*{-1pt}
\subfloat[]{\includegraphics[width=.155\linewidth]{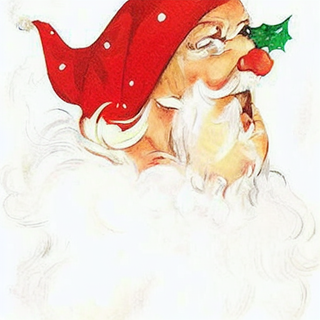}}\hspace*{-1pt}
\subfloat[]{\includegraphics[width=.155\linewidth]{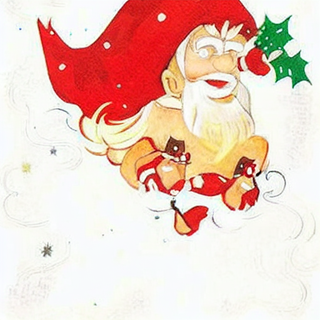}}\hspace*{-1pt}
\subfloat[]{\includegraphics[width=.155\linewidth]{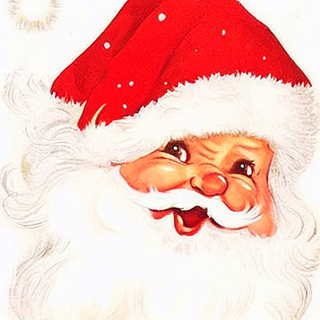}}\hspace*{-1pt}
\subfloat[]{\includegraphics[width=.155\linewidth]{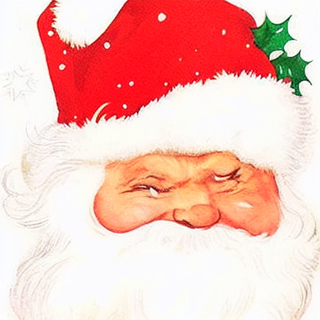}}

\begin{minipage}{\linewidth}
        \centering
        {\footnotesize \textit{“an illustration of a \textbf{cat} sitting on top of a rock”} $\rightarrow$ \textit{“an illustration of a \textbf{bear} sitting on top of a rock”}}
        \medskip
        \end{minipage}

\subfloat[]{\includegraphics[width=.155\linewidth]{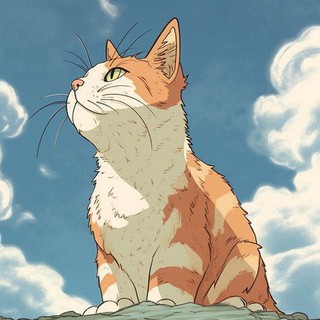}}\hspace*{-1pt}
\subfloat[]{\includegraphics[width=.155\linewidth]{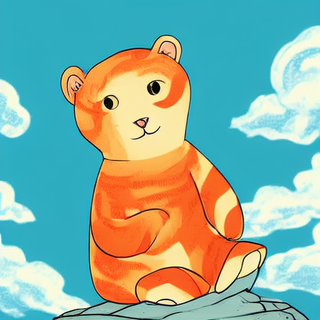}}\hspace*{-1pt}
\subfloat[]{\includegraphics[width=.155\linewidth]{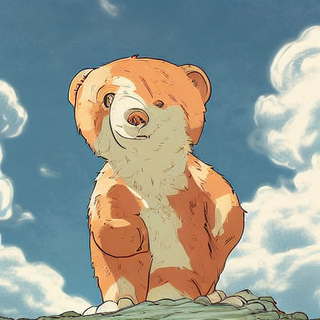}}\hspace*{-1pt}
\subfloat[]{\includegraphics[width=.155\linewidth]{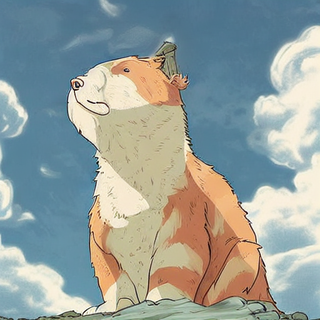}}\hspace*{-1pt}
\subfloat[]{\includegraphics[width=.155\linewidth]{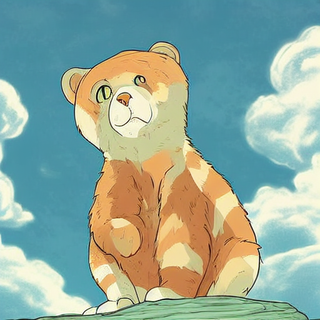}}\hspace*{-1pt}
\subfloat[]{\includegraphics[width=.155\linewidth]{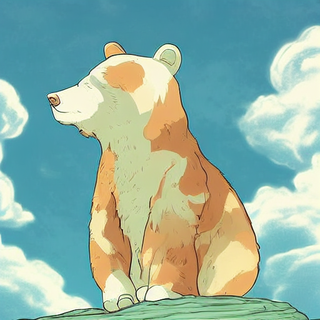}}

\begin{minipage}{\linewidth}
        \centering
        {\footnotesize \textit{“a cat standing on fence”} $\rightarrow$ \textit{“a cat \textbf{wearing hat} standing on fence”}}
        \medskip
        \end{minipage}

\subfloat[]{\includegraphics[width=.155\linewidth]{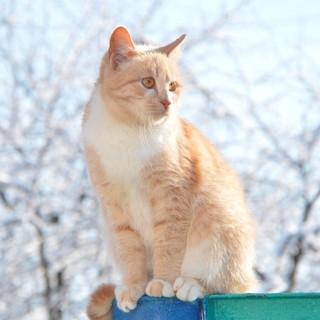}}\hspace*{-1pt}
\subfloat[]{\includegraphics[width=.155\linewidth]{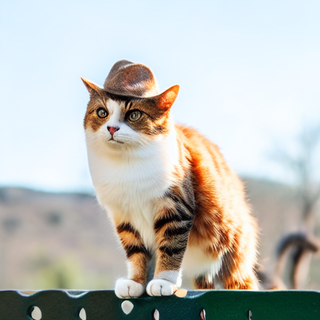}}\hspace*{-1pt}
\subfloat[]{\includegraphics[width=.155\linewidth]{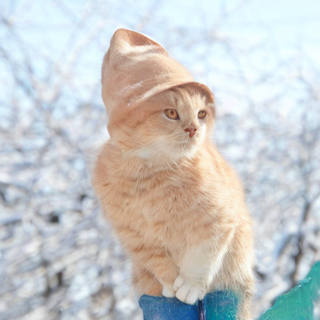}}\hspace*{-1pt}
\subfloat[]{\includegraphics[width=.155\linewidth]{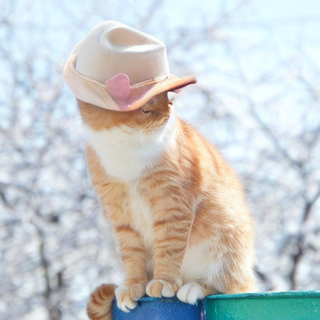}}\hspace*{-1pt}
\subfloat[]{\includegraphics[width=.155\linewidth]{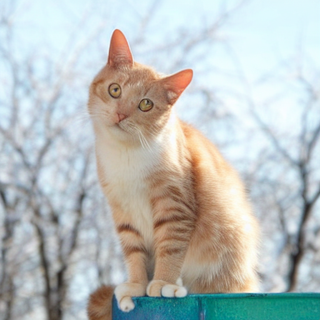}}\hspace*{-1pt}
\subfloat[]{\includegraphics[width=.155\linewidth]{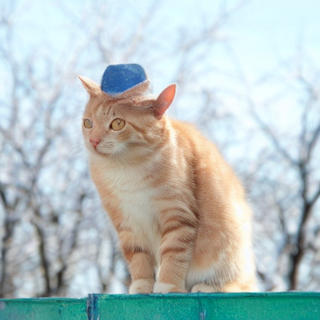}}

\begin{minipage}{\linewidth}
        \centering
        {\footnotesize \textit{“a glass of red \textbf{drink} on the beach”} $\rightarrow$ \textit{“a glass of red \textbf{wine} on the beach”}}
        \medskip
        \end{minipage}

\subfloat[]{\includegraphics[width=.155\linewidth]{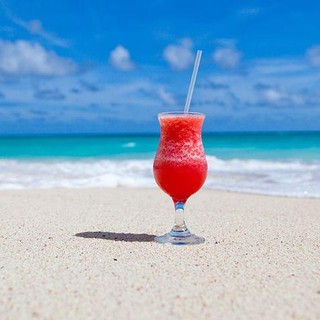}}\hspace*{-1pt}
\subfloat[]{\includegraphics[width=.155\linewidth]{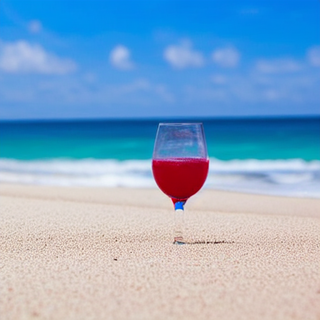}}\hspace*{-1pt}
\subfloat[]{\includegraphics[width=.155\linewidth]{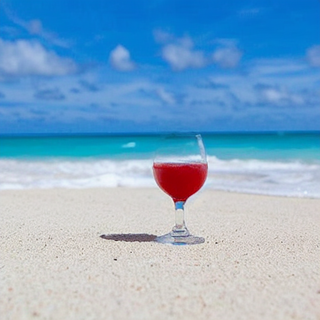}}\hspace*{-1pt}
\subfloat[]{\includegraphics[width=.155\linewidth]{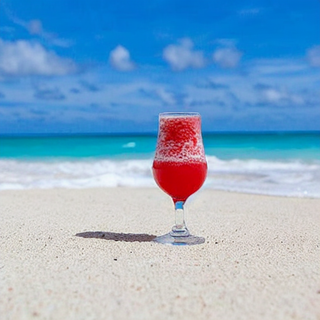}}\hspace*{-1pt}
\subfloat[]{\includegraphics[width=.155\linewidth]{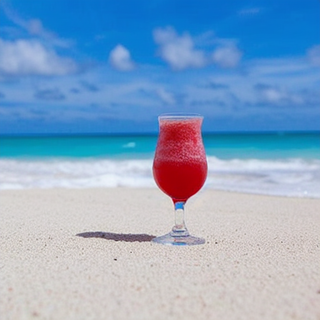}}\hspace*{-1pt}
\subfloat[]{\includegraphics[width=.155\linewidth]{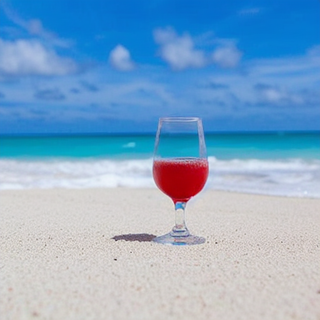}}

\begin{minipage}{\linewidth}
        \centering
        {\footnotesize \textit{“a painting of a \textbf{cabin} in the snow with mountains in the background”} $\rightarrow$ \textit{“a painting of a \textbf{car} in the snow with mountains in the background”}}
        \medskip
        \end{minipage}

\subfloat[]{\includegraphics[width=.155\linewidth]{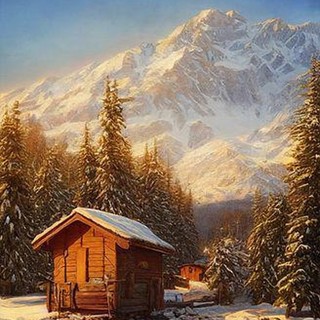}}\hspace*{-1pt}
\subfloat[]{\includegraphics[width=.155\linewidth]{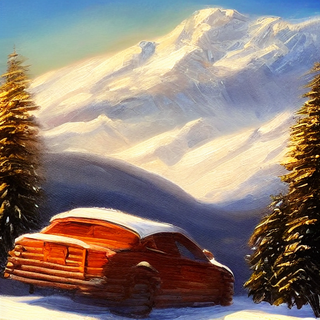}}\hspace*{-1pt}
\subfloat[]{\includegraphics[width=.155\linewidth]{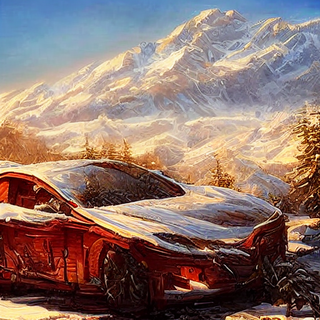}}\hspace*{-1pt}
\subfloat[]{\includegraphics[width=.155\linewidth]{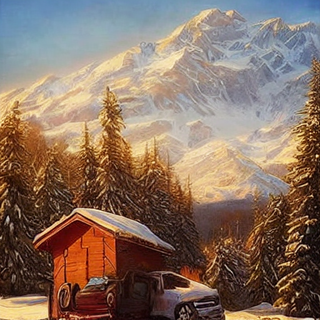}}\hspace*{-1pt}
\subfloat[]{\includegraphics[width=.155\linewidth]{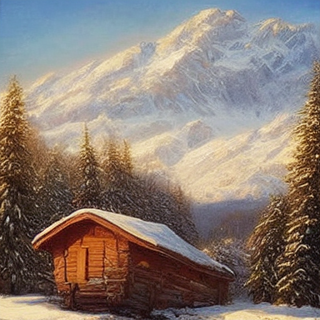}}\hspace*{-1pt}
\subfloat[]{\includegraphics[width=.155\linewidth]{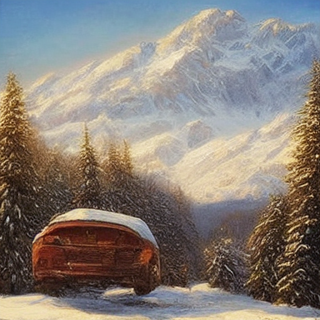}}

\begin{minipage}{\linewidth}
        \centering
        {\footnotesize \textit{“a lion in a suit sitting at a table \textbf{with a laptop}”} $\rightarrow$ \textit{“a lion in a suit sitting at a table”}}
        \medskip
        \end{minipage}

\subfloat[Source]{\includegraphics[width=.155\linewidth]{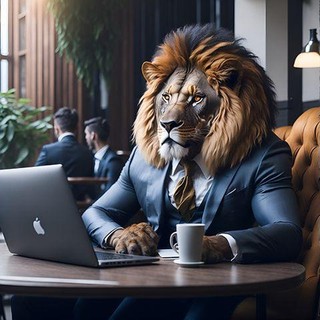}}\hspace*{-1pt}
\subfloat[DDIM~Inv~\cite{song2020denoising}]{\includegraphics[width=.155\linewidth]{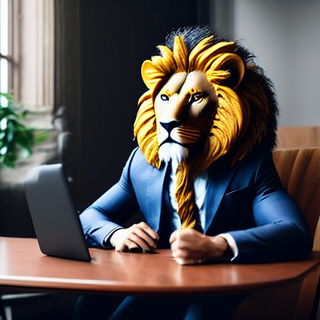}}\hspace*{-1pt}
\subfloat[NTI~\cite{mokady2023null}]{\includegraphics[width=.155\linewidth]{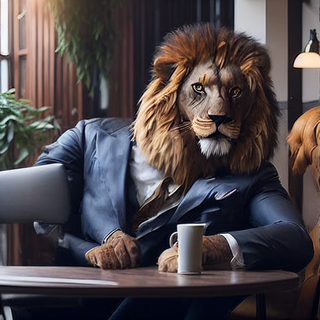}}\hspace*{-1pt}
\subfloat[EDICT~\cite{wallace2023edict}]{\includegraphics[width=.155\linewidth]{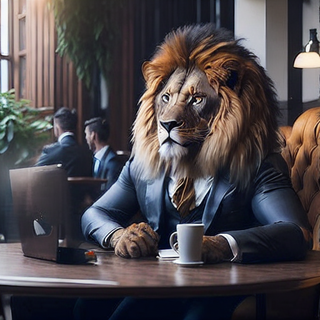}}\hspace*{-1pt}
\subfloat[Dir. Inv.~\cite{ju2023direct}]{\includegraphics[width=.155\linewidth]{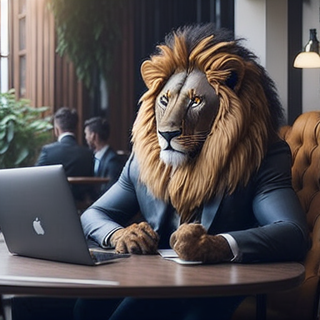}}\hspace*{-1pt}
\subfloat[\textbf{EtaInv~(2)}]{\includegraphics[width=.155\linewidth]{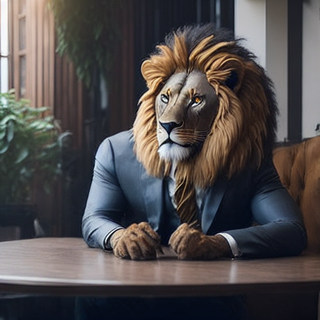}}

    \caption{Additional qualitative results for MasaCtrl editing.}
    \label{fig:masa}
    \end{figure*}

%% file: figs/fig_etainv3_result.tex
\begin{figure*}[t]
    \captionsetup[subfigure]{labelformat=empty}
    \centering

\begin{minipage}{\linewidth}
        \centering
        {\footnotesize \textit{“a \textbf{collie dog} is sitting on a \textbf{bed}”} $\rightarrow$ \textit{“a \textbf{garfield cat} is sitting on a \textbf{sofa}”}}
        \medskip
        \end{minipage}

\subfloat[]{\includegraphics[width=.155\linewidth]{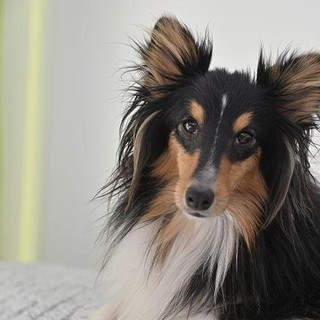}}\hspace*{-1pt}
\subfloat[]{\includegraphics[width=.155\linewidth]{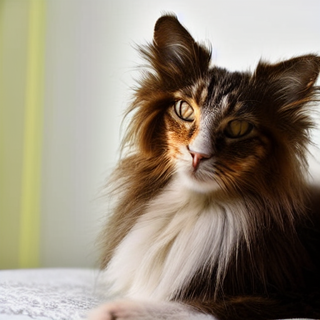}}\hspace*{-1pt}
\subfloat[]{\includegraphics[width=.155\linewidth]{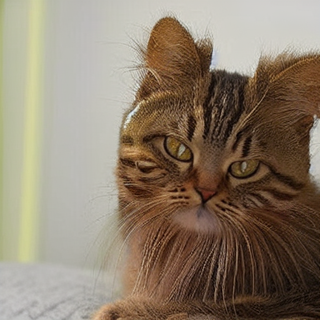}}\hspace*{-1pt}
\subfloat[]{\includegraphics[width=.155\linewidth]{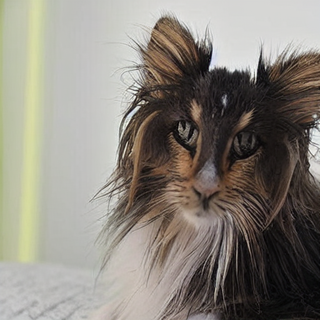}}\hspace*{-1pt}
\subfloat[]{\includegraphics[width=.155\linewidth]{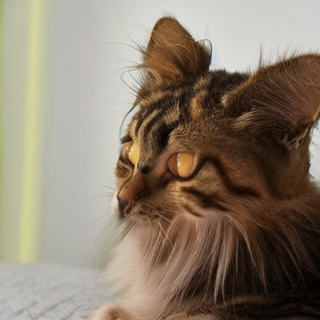}}\hspace*{-1pt}
\subfloat[]{\includegraphics[width=.155\linewidth]{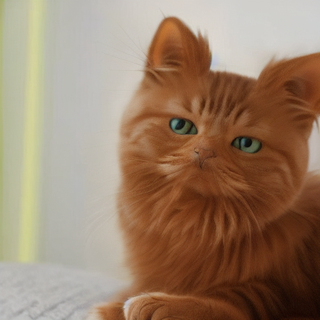}}

\begin{minipage}{\linewidth}
        \centering
        {\footnotesize \textit{“a living room with a couch and a table”} $\rightarrow$ \textit{“\textbf{a watercolor of} a living room with a couch and a table”}}
        \medskip
        \end{minipage}

\subfloat[]{\includegraphics[width=.155\linewidth]{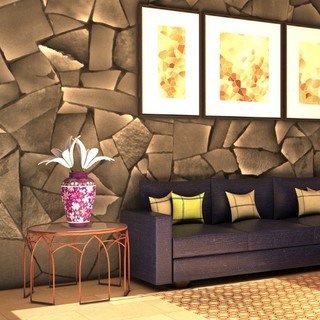}}\hspace*{-1pt}
\subfloat[]{\includegraphics[width=.155\linewidth]{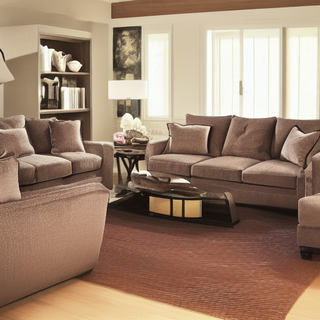}}\hspace*{-1pt}
\subfloat[]{\includegraphics[width=.155\linewidth]{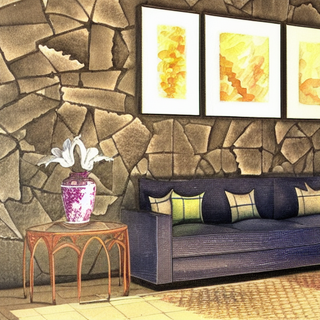}}\hspace*{-1pt}
\subfloat[]{\includegraphics[width=.155\linewidth]{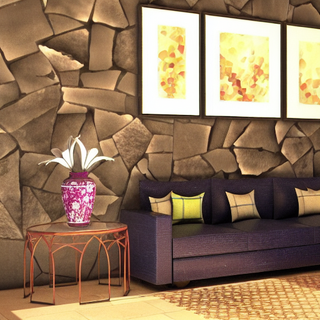}}\hspace*{-1pt}
\subfloat[]{\includegraphics[width=.155\linewidth]{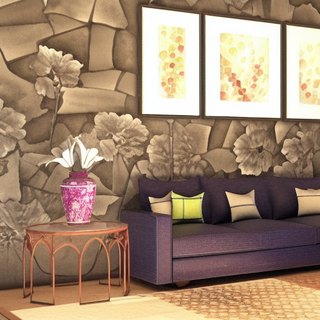}}\hspace*{-1pt}
\subfloat[]{\includegraphics[width=.155\linewidth]{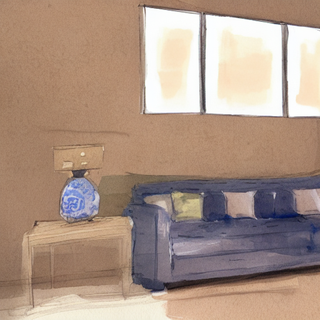}}

\begin{minipage}{\linewidth}
        \centering
        {\footnotesize \textit{“a \textbf{black and white} drawing of a woman with long hair”} $\rightarrow$ \textit{“a \textbf{colorful and detailed} drawing of a woman with long hair”}}
        \medskip
        \end{minipage}

\subfloat[]{\includegraphics[width=.155\linewidth]{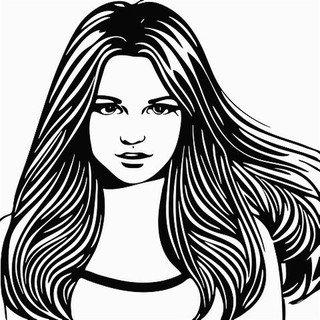}}\hspace*{-1pt}
\subfloat[]{\includegraphics[width=.155\linewidth]{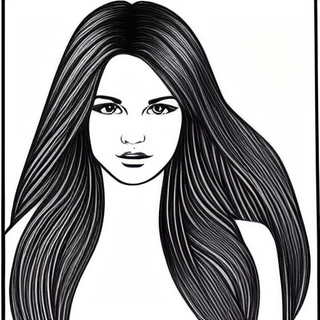}}\hspace*{-1pt}
\subfloat[]{\includegraphics[width=.155\linewidth]{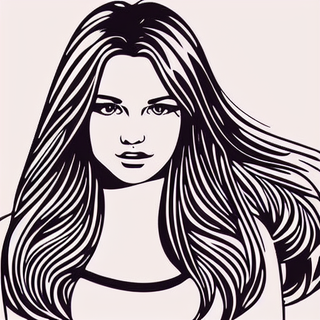}}\hspace*{-1pt}
\subfloat[]{\includegraphics[width=.155\linewidth]{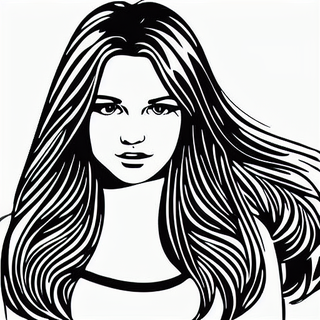}}\hspace*{-1pt}
\subfloat[]{\includegraphics[width=.155\linewidth]{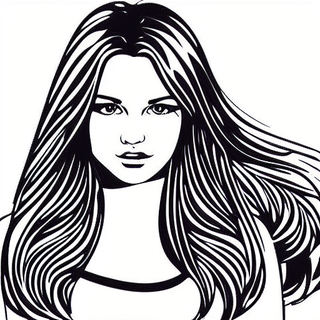}}\hspace*{-1pt}
\subfloat[]{\includegraphics[width=.155\linewidth]{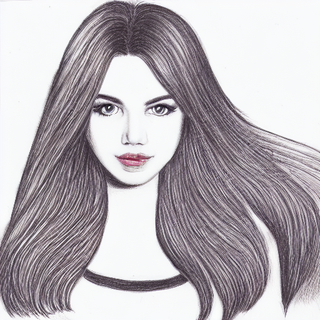}}

\begin{minipage}{\linewidth}
        \centering
        {\footnotesize \textit{“a woman with \textbf{black} hair and red lipstick holding a flower”} $\rightarrow$ \textit{“a woman with \textbf{silver} hair and red lipstick holding a flower”}}
        \medskip
        \end{minipage}

\subfloat[]{\includegraphics[width=.155\linewidth]{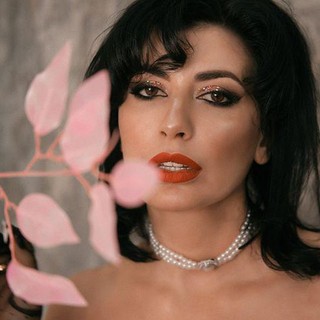}}\hspace*{-1pt}
\subfloat[]{\includegraphics[width=.155\linewidth]{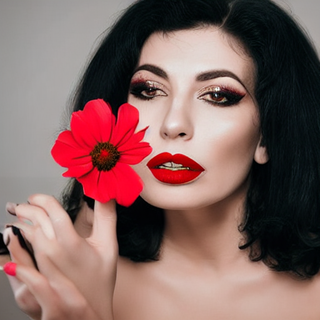}}\hspace*{-1pt}
\subfloat[]{\includegraphics[width=.155\linewidth]{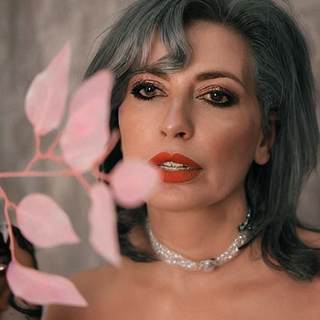}}\hspace*{-1pt}
\subfloat[]{\includegraphics[width=.155\linewidth]{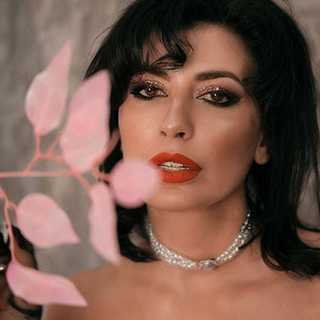}}\hspace*{-1pt}
\subfloat[]{\includegraphics[width=.155\linewidth]{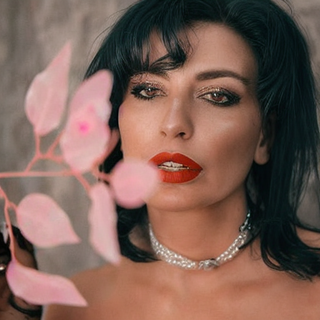}}\hspace*{-1pt}
\subfloat[]{\includegraphics[width=.155\linewidth]{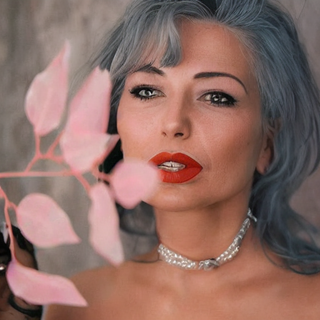}}

\begin{minipage}{\linewidth}
        \centering
        {\footnotesize \textit{“a \textbf{dry} tree in the wild”} $\rightarrow$ \textit{“a \textbf{blooming} tree in the wild”}}
        \medskip
        \end{minipage}

\subfloat[]{\includegraphics[width=.155\linewidth]{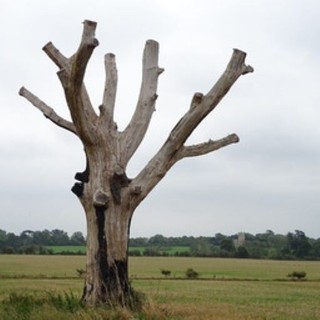}}\hspace*{-1pt}
\subfloat[]{\includegraphics[width=.155\linewidth]{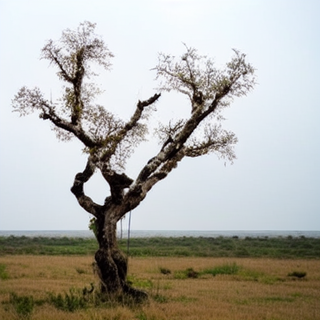}}\hspace*{-1pt}
\subfloat[]{\includegraphics[width=.155\linewidth]{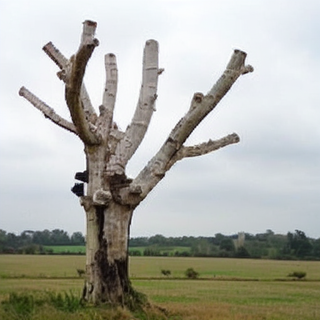}}\hspace*{-1pt}
\subfloat[]{\includegraphics[width=.155\linewidth]{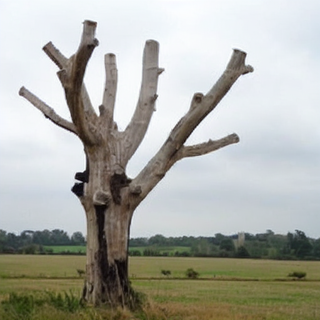}}\hspace*{-1pt}
\subfloat[]{\includegraphics[width=.155\linewidth]{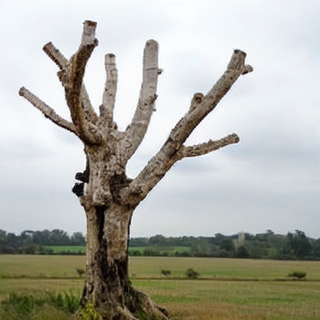}}\hspace*{-1pt}
\subfloat[]{\includegraphics[width=.155\linewidth]{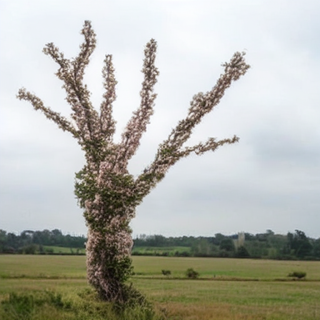}}

\begin{minipage}{\linewidth}
        \centering
        {\footnotesize \textit{“a little girl wearing sunglasses and a gray \textbf{shirt} leaning against a wall”} $\rightarrow$ \textit{“a little girl wearing sunglasses and a gray \textbf{dress} leaning against a wall”}}
        \medskip
        \end{minipage}

\subfloat[Source]{\includegraphics[width=.155\linewidth]{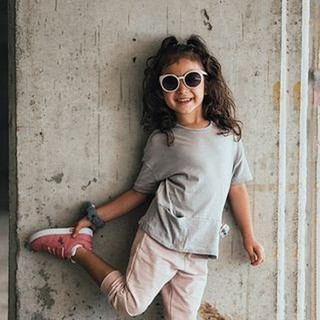}}\hspace*{-1pt}
\subfloat[DDIM~Inv~\cite{song2020denoising}]{\includegraphics[width=.155\linewidth]{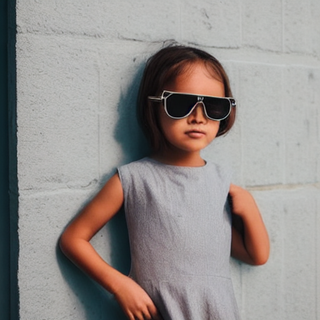}}\hspace*{-1pt}
\subfloat[NTI~\cite{mokady2023null}]{\includegraphics[width=.155\linewidth]{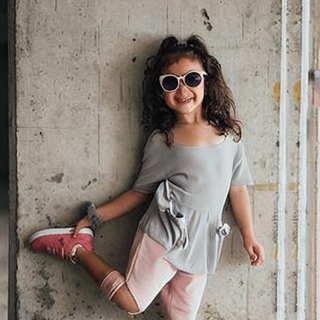}}\hspace*{-1pt}
\subfloat[EDICT~\cite{wallace2023edict}]{\includegraphics[width=.155\linewidth]{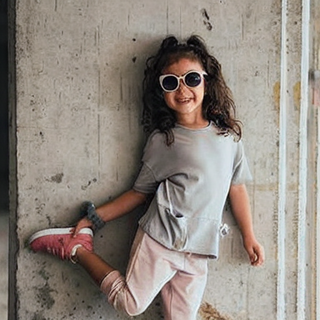}}\hspace*{-1pt}
\subfloat[Dir. Inv.~\cite{ju2023direct}]{\includegraphics[width=.155\linewidth]{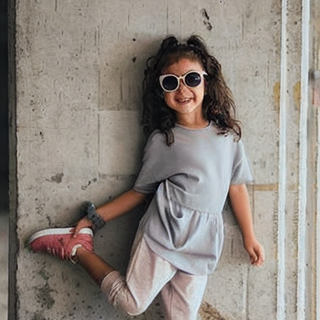}}\hspace*{-1pt}
\subfloat[\textbf{EtaInv~(3)}]{\includegraphics[width=.155\linewidth]{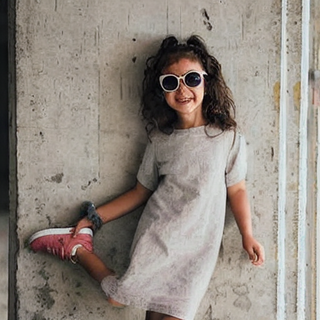}}

    \caption{Additional qualitative results for EtaInv (3) PtP editing. %
    }
    \label{fig:etainv3}
    \end{figure*}

%% file: figs/fig_rat_pig.tex
\begin{figure*}[t]
    \captionsetup[subfigure]{labelformat=empty}
    \centering

    \begin{minipage}{\linewidth}
        \centering
        \bigskip
        {\textit{“a painting of a \textbf{rat} with red eyes”} $\rightarrow$ \textit{“a painting of a \textbf{pig} with red eyes”}}
        \medskip
        \end{minipage}
    
    \subfloat[{$[0.0,0.0]$}]{
    \includegraphics[width=.158\linewidth]{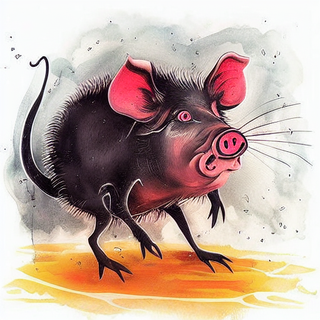}
    }\hspace*{-6pt}
    \subfloat[{$[0.2,0.0]$}]{
    \includegraphics[width=.158\linewidth]{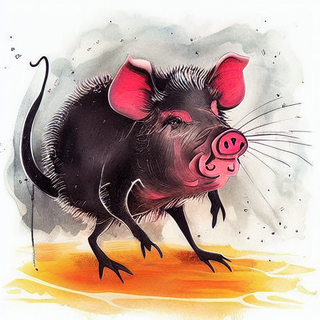}
    }\hspace*{-6pt}
    \subfloat[{$[0.4,0.0]$}]{
    \includegraphics[width=.158\linewidth]{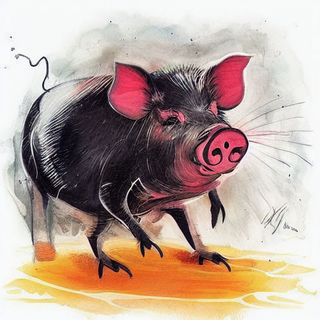}
    }\hspace*{-6pt}
    \subfloat[{$[0.6,0.0]$}]{
    \includegraphics[width=.158\linewidth]{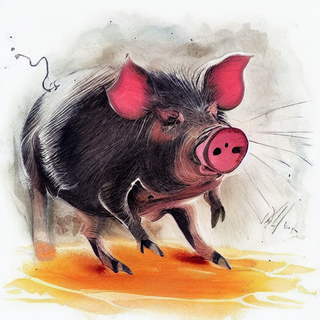}
    }\hspace*{-6pt}
    \subfloat[{$[0.8,0.0]$}]{
    \includegraphics[width=.158\linewidth]{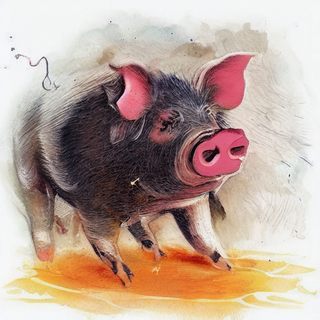}
    }\hspace*{-6pt}
    \subfloat[{$[1.0,0.0]$}]{
    \includegraphics[width=.158\linewidth]{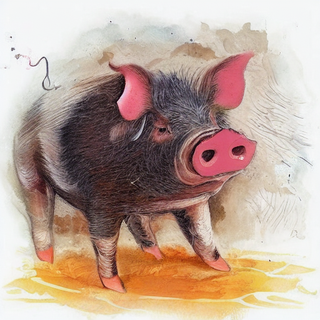}
    }
    
    \subfloat[]{
\resizebox{.16\linewidth}{.16\linewidth}{
\begin{tikzpicture}
\path (0,0) to (0,1);
\end{tikzpicture}
}
    }\hspace*{-6pt}
    \subfloat[{$[0.2,0.2]$}]{
    \includegraphics[width=.158\linewidth]{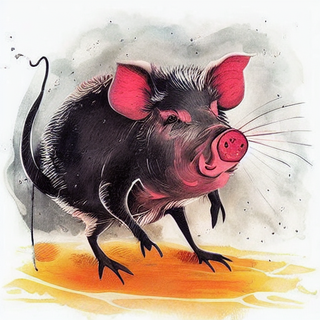}
    }\hspace*{-6pt}
    \subfloat[{$[0.4,0.2]$}]{
    \includegraphics[width=.158\linewidth]{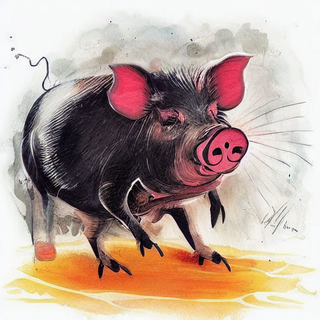}
    }\hspace*{-6pt}
    \subfloat[{$[0.6,0.2]$}]{
    \includegraphics[width=.158\linewidth]{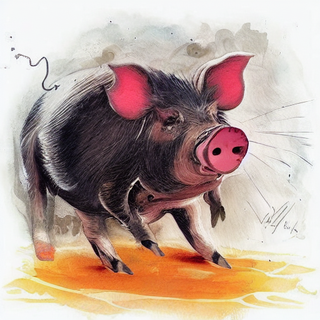}
    }\hspace*{-6pt}
    \subfloat[{$[0.8,0.2]$}]{
    \includegraphics[width=.158\linewidth]{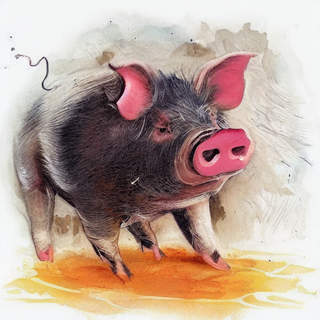}
    }\hspace*{-6pt}
    \subfloat[{$[1.0,0.2]$}]{
    \includegraphics[width=.158\linewidth]{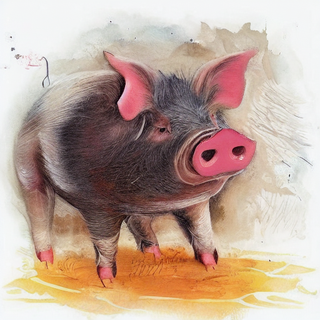}
    }
    
    \subfloat[Source]{
    \includegraphics[width=.158\linewidth]{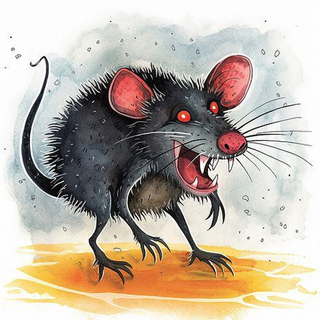}
    } \hspace*{-6pt}
    \subfloat[]{
\resizebox{.16\linewidth}{.16\linewidth}{
\begin{tikzpicture}
\path (0,0) to (0,1);
\end{tikzpicture}
}
    }\hspace*{-6pt}
    \subfloat[{$[0.4,0.4]$}]{
    \includegraphics[width=.158\linewidth]{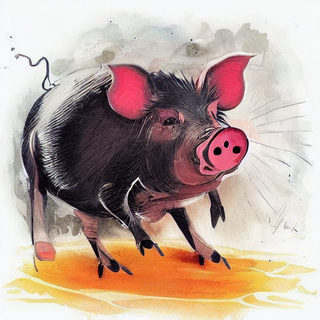}
    }\hspace*{-6pt}
    \subfloat[{$[0.6,0.4]$}]{
    \includegraphics[width=.158\linewidth]{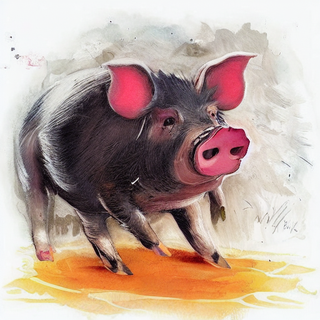}
    }\hspace*{-6pt}
    \subfloat[{$[0.8,0.4]$}]{
    \includegraphics[width=.158\linewidth]{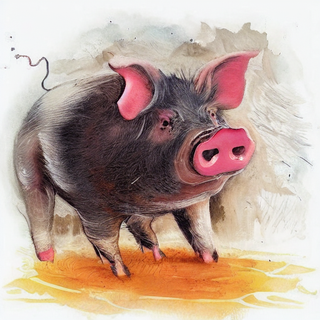}
    }\hspace*{-6pt}
    \subfloat[{$[1.0,0.4]$}]{
    \includegraphics[width=.158\linewidth]{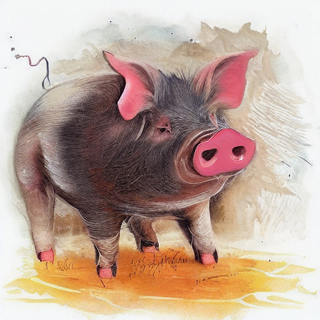}
    }
    
    \caption{Impact of $\eta$ on image editing. We use a different linear $\eta$ function for each generation by linearly interpolating $\eta$ on the interval $[\eta(T), \eta(0)]$. Increasing $\eta$ leads to better target prompt alignment while sacrificing background similarity. We disabled masking for demonstration purposes.}
    \label{fig:rat_pig}
    \end{figure*}

%% file: main.bbl
\begin{thebibliography}{10}
\providecommand{\url}[1]{\texttt{#1}}
\providecommand{\urlprefix}{URL }
\providecommand{\doi}[1]{https://doi.org/#1}

\bibitem{brooks2023instructpix2pix}
Brooks, T., Holynski, A., Efros, A.A.: Instructpix2pix: Learning to follow image editing instructions. In: Proceedings of the IEEE/CVF Conference on Computer Vision and Pattern Recognition. pp. 18392--18402 (2023)

\bibitem{cao2023masactrl}
Cao, M., Wang, X., Qi, Z., Shan, Y., Qie, X., Zheng, Y.: Masactrl: Tuning-free mutual self-attention control for consistent image synthesis and editing. In: Proceedings of the IEEE/CVF International Conference on Computer Vision. pp. 22560--22570 (2023)

\bibitem{cao2023exploring}
Cao, Y., Chen, J., Luo, Y., ZHOU, X.: Exploring the optimal choice for generative processes in diffusion models: Ordinary vs stochastic differential equations. In: Thirty-seventh Conference on Neural Information Processing Systems (2023)

\bibitem{caron2021emerging}
Caron, M., Touvron, H., Misra, I., J\'egou, H., Mairal, J., Bojanowski, P., Joulin, A.: Emerging properties in self-supervised vision transformers. In: Proceedings of the International Conference on Computer Vision (ICCV) (2021)

\bibitem{chang2023muse}
Chang, H., Zhang, H., Barber, J., Maschinot, A., Lezama, J., Jiang, L., Yang, M.H., Murphy, K.P., Freeman, W.T., Rubinstein, M., et~al.: Muse: Text-to-image generation via masked generative transformers. In: International Conference on Machine Learning. pp. 4055--4075. PMLR (2023)

\bibitem{chang2022maskgit}
Chang, H., Zhang, H., Jiang, L., Liu, C., Freeman, W.T.: Maskgit: Masked generative image transformer. In: Proceedings of the IEEE/CVF Conference on Computer Vision and Pattern Recognition. pp. 11315--11325 (2022)

\bibitem{chefer2023attendandexcite}
Chefer, H., Alaluf, Y., Vinker, Y., Wolf, L., Cohen-Or, D.: Attend-and-excite: Attention-based semantic guidance for text-to-image diffusion models. ACM Transactions on Graphics (TOG)  \textbf{42}(4),  1--10 (2023)

\bibitem{chen2023improved}
Chen, H., Lee, H., Lu, J.: Improved analysis of score-based generative modeling: User-friendly bounds under minimal smoothness assumptions. In: International Conference on Machine Learning. pp. 4735--4763. PMLR (2023)

\bibitem{chen2022sampling}
Chen, S., Chewi, S., Li, J., Li, Y., Salim, A., Zhang, A.: Sampling is as easy as learning the score: theory for diffusion models with minimal data assumptions. In: The Eleventh International Conference on Learning Representations (2023)

\bibitem{chen2015microsoft}
Chen, X., Fang, H., Lin, T.Y., Vedantam, R., Gupta, S., Doll{\'a}r, P., Zitnick, C.L.: Microsoft coco captions: Data collection and evaluation server. arXiv preprint arXiv:1504.00325  (2015)

\bibitem{couairon2022diffedit}
Couairon, G., Verbeek, J., Schwenk, H., Cord, M.: Diffedit: Diffusion-based semantic image editing with mask guidance. In: The Eleventh International Conference on Learning Representations (2023)

\bibitem{dhariwal2021diffusion}
Dhariwal, P., Nichol, A.: Diffusion models beat gans on image synthesis. Advances in neural information processing systems  \textbf{34},  8780--8794 (2021)

\bibitem{dong2023prompt}
Dong, W., Xue, S., Duan, X., Han, S.: Prompt tuning inversion for text-driven image editing using diffusion models. In: Proceedings of the IEEE/CVF International Conference on Computer Vision. pp. 7430--7440 (2023)

\bibitem{dosovitskiy2021image}
Dosovitskiy, A., Beyer, L., Kolesnikov, A., Weissenborn, D., Zhai, X., Unterthiner, T., Dehghani, M., Minderer, M., Heigold, G., Gelly, S., Uszkoreit, J., Houlsby, N.: An image is worth 16x16 words: Transformers for image recognition at scale. In: International Conference on Learning Representations (2021)

\bibitem{ge2023making}
Ge, Y., Zhao, S., Zeng, Z., Ge, Y., Li, C., Wang, X., Shan, Y.: Making {LL}a{MA} {SEE} and draw with {SEED} tokenizer. In: The Twelfth International Conference on Learning Representations (2024)

\bibitem{goodfellow2014generative}
Goodfellow, I., Pouget-Abadie, J., Mirza, M., Xu, B., Warde-Farley, D., Ozair, S., Courville, A., Bengio, Y.: Generative adversarial nets. Advances in neural information processing systems  \textbf{27} (2014)

\bibitem{han2023improving}
Han, L., Wen, S., Chen, Q., Zhang, Z., Song, K., Ren, M., Gao, R., Stathopoulos, A., He, X., Chen, Y., et~al.: Proxedit: Improving tuning-free real image editing with proximal guidance. In: Proceedings of the IEEE/CVF Winter Conference on Applications of Computer Vision. pp. 4291--4301 (2024)

\bibitem{hertz2022prompt}
Hertz, A., Mokady, R., Tenenbaum, J., Aberman, K., Pritch, Y., Cohen-or, D.: Prompt-to-prompt image editing with cross-attention control. In: The Eleventh International Conference on Learning Representations (2023)

\bibitem{ho2020denoising}
Ho, J., Jain, A., Abbeel, P.: Denoising diffusion probabilistic models. Advances in neural information processing systems  \textbf{33},  6840--6851 (2020)

\bibitem{ho2022classifier}
Ho, J., Salimans, T.: Classifier-free diffusion guidance. In: NeurIPS 2021 Workshop on Deep Generative Models and Downstream Applications (2021)

\bibitem{huberman2023edit}
Huberman-Spiegelglas, I., Kulikov, V., Michaeli, T.: An edit friendly ddpm noise space: Inversion and manipulations. In: Proceedings of the IEEE/CVF Conference on Computer Vision and Pattern Recognition. pp. 12469--12478 (2024)

\bibitem{ju2023direct}
Ju, X., Zeng, A., Bian, Y., Liu, S., Xu, Q.: Direct inversion: Boosting diffusion-based editing with 3 lines of code. arXiv preprint arXiv:2310.01506  (2023)

\bibitem{kang2023scaling}
Kang, M., Zhu, J.Y., Zhang, R., Park, J., Shechtman, E., Paris, S., Park, T.: Scaling up gans for text-to-image synthesis. In: Proceedings of the IEEE/CVF Conference on Computer Vision and Pattern Recognition. pp. 10124--10134 (2023)

\bibitem{alexnet}
Krizhevsky, A., Sutskever, I., Hinton, G.E.: Imagenet classification with deep convolutional neural networks. In: Proceedings of the 25th International Conference on Neural Information Processing Systems - Volume 1. p. 1097–1105. NIPS'12, Curran Associates Inc., Red Hook, NY, USA (2012)

\bibitem{ku2024viescoreexplainablemetricsconditional}
Ku, M., Jiang, D., Wei, C., Yue, X., Chen, W.: Viescore: Towards explainable metrics for conditional image synthesis evaluation (2024), \url{https://arxiv.org/abs/2312.14867}

\bibitem{li2022blip}
Li, J., Li, D., Xiong, C., Hoi, S.: Blip: Bootstrapping language-image pre-training for unified vision-language understanding and generation. In: International Conference on Machine Learning. pp. 12888--12900. PMLR (2022)

\bibitem{liu2023visual}
Liu, H., Li, C., Wu, Q., Lee, Y.J.: Visual instruction tuning. In: Thirty-seventh Conference on Neural Information Processing Systems (2023)

\bibitem{lu2022maximum}
Lu, C., Zheng, K., Bao, F., Chen, J., Li, C., Zhu, J.: Maximum likelihood training for score-based diffusion odes by high order denoising score matching. In: International Conference on Machine Learning. pp. 14429--14460. PMLR (2022)

\bibitem{meng2021sdedit}
Meng, C., He, Y., Song, Y., Song, J., Wu, J., Zhu, J.Y., Ermon, S.: {SDE}dit: Guided image synthesis and editing with stochastic differential equations. In: International Conference on Learning Representations (2022)

\bibitem{miyake2023negative}
Miyake, D., Iohara, A., Saito, Y., Tanaka, T.: Negative-prompt inversion: Fast image inversion for editing with text-guided diffusion models. arXiv preprint arXiv:2305.16807  (2023)

\bibitem{mokady2023null}
Mokady, R., Hertz, A., Aberman, K., Pritch, Y., Cohen-Or, D.: Null-text inversion for editing real images using guided diffusion models. In: Proceedings of the IEEE/CVF Conference on Computer Vision and Pattern Recognition. pp. 6038--6047 (2023)

\bibitem{mou2024diffeditor}
Mou, C., Wang, X., Song, J., Shan, Y., Zhang, J.: Diffeditor: Boosting accuracy and flexibility on diffusion-based image editing. In: Proceedings of the IEEE/CVF Conference on Computer Vision and Pattern Recognition. pp. 8488--8497 (2024)

\bibitem{nichol2021glide}
Nichol, A.Q., Dhariwal, P., Ramesh, A., Shyam, P., Mishkin, P., Mcgrew, B., Sutskever, I., Chen, M.: Glide: Towards photorealistic image generation and editing with text-guided diffusion models. In: International Conference on Machine Learning. pp. 16784--16804. PMLR (2022)

\bibitem{nie2023blessing}
Nie, S., Guo, H.A., Lu, C., Zhou, Y., Zheng, C., Li, C.: The blessing of randomness: {SDE} beats {ODE} in general diffusion-based image editing. In: The Twelfth International Conference on Learning Representations (2024)

\bibitem{openai2023gpt4}
OpenAI: Gpt-4 technical report (2023)

\bibitem{oquab2024dinov2}
Oquab, M., Darcet, T., Moutakanni, T., Vo, H., Szafraniec, M., Khalidov, V., Fernandez, P., Haziza, D., Massa, F., El-Nouby, A., Assran, M., Ballas, N., Galuba, W., Howes, R., Huang, P.Y., Li, S.W., Misra, I., Rabbat, M., Sharma, V., Synnaeve, G., Xu, H., Jegou, H., Mairal, J., Labatut, P., Joulin, A., Bojanowski, P.: Dinov2: Learning robust visual features without supervision (2024)

\bibitem{parmar2023zero}
Parmar, G., Kumar~Singh, K., Zhang, R., Li, Y., Lu, J., Zhu, J.Y.: Zero-shot image-to-image translation. In: ACM SIGGRAPH 2023 Conference Proceedings. pp. 1--11 (2023)

\bibitem{paszke2019pytorch}
Paszke, A., Gross, S., Massa, F., Lerer, A., Bradbury, J., Chanan, G., Killeen, T., Lin, Z., Gimelshein, N., Antiga, L., Desmaison, A., Köpf, A., Yang, E., DeVito, Z., Raison, M., Tejani, A., Chilamkurthy, S., Steiner, B., Fang, L., Bai, J., Chintala, S.: Pytorch: An imperative style, high-performance deep learning library (2019)

\bibitem{patashnik2021styleclip}
Patashnik, O., Wu, Z., Shechtman, E., Cohen-Or, D., Lischinski, D.: Styleclip: Text-driven manipulation of stylegan imagery. In: Proceedings of the IEEE/CVF international conference on computer vision. pp. 2085--2094 (2021)

\bibitem{von-platen-etal-2022-diffusers}
von Platen, P., Patil, S., Lozhkov, A., Cuenca, P., Lambert, N., Rasul, K., Davaadorj, M., Wolf, T.: Diffusers: State-of-the-art diffusion models. \url{https://github.com/huggingface/diffusers} (2022)

\bibitem{podell2023sdxl}
Podell, D., English, Z., Lacey, K., Blattmann, A., Dockhorn, T., M{\"u}ller, J., Penna, J., Rombach, R.: {SDXL}: Improving latent diffusion models for high-resolution image synthesis. In: The Twelfth International Conference on Learning Representations (2024)

\bibitem{radford2021learning}
Radford, A., Kim, J.W., Hallacy, C., Ramesh, A., Goh, G., Agarwal, S., Sastry, G., Askell, A., Mishkin, P., Clark, J., et~al.: Learning transferable visual models from natural language supervision. In: International conference on machine learning. pp. 8748--8763. PMLR (2021)

\bibitem{ramesh2022hierarchical}
Ramesh, A., Dhariwal, P., Nichol, A., Chu, C., Chen, M.: Hierarchical text-conditional image generation with clip latents. arXiv preprint arXiv:2204.06125  \textbf{1}(2), ~3 (2022)

\bibitem{ramesh2021zero}
Ramesh, A., Pavlov, M., Goh, G., Gray, S., Voss, C., Radford, A., Chen, M., Sutskever, I.: Zero-shot text-to-image generation. In: International Conference on Machine Learning. pp. 8821--8831. PMLR (2021)

\bibitem{richardson2021encoding}
Richardson, E., Alaluf, Y., Patashnik, O., Nitzan, Y., Azar, Y., Shapiro, S., Cohen-Or, D.: Encoding in style: a stylegan encoder for image-to-image translation. In: Proceedings of the IEEE/CVF conference on computer vision and pattern recognition. pp. 2287--2296 (2021)

\bibitem{rombach2022high}
Rombach, R., Blattmann, A., Lorenz, D., Esser, P., Ommer, B.: High-resolution image synthesis with latent diffusion models. In: Proceedings of the IEEE/CVF conference on computer vision and pattern recognition. pp. 10684--10695 (2022)

\bibitem{ronneberger2015unet}
Ronneberger, O., Fischer, P., Brox, T.: U-net: Convolutional networks for biomedical image segmentation (2015)

\bibitem{saharia2022photorealistic}
Saharia, C., Chan, W., Saxena, S., Li, L., Whang, J., Denton, E.L., Ghasemipour, K., Gontijo~Lopes, R., Karagol~Ayan, B., Salimans, T., et~al.: Photorealistic text-to-image diffusion models with deep language understanding. Advances in Neural Information Processing Systems  \textbf{35},  36479--36494 (2022)

\bibitem{song2020denoising}
Song, J., Meng, C., Ermon, S.: Denoising diffusion implicit models. In: International Conference on Learning Representations (2021)

\bibitem{song2020score}
Song, Y., Sohl-Dickstein, J., Kingma, D.P., Kumar, A., Ermon, S., Poole, B.: Score-based generative modeling through stochastic differential equations. In: International Conference on Learning Representations (2021)

\bibitem{tumanyan2023plug}
Tumanyan, N., Geyer, M., Bagon, S., Dekel, T.: Plug-and-play diffusion features for text-driven image-to-image translation. In: Proceedings of the IEEE/CVF Conference on Computer Vision and Pattern Recognition. pp. 1921--1930 (2023)

\bibitem{wallace2023edict}
Wallace, B., Gokul, A., Naik, N.: Edict: Exact diffusion inversion via coupled transformations. In: Proceedings of the IEEE/CVF Conference on Computer Vision and Pattern Recognition. pp. 22532--22541 (2023)

\bibitem{wang2003multiscale}
Wang, Z., Simoncelli, E.P., Bovik, A.C.: Multiscale structural similarity for image quality assessment. In: The Thrity-Seventh Asilomar Conference on Signals, Systems \& Computers, 2003. vol.~2, pp. 1398--1402. Ieee (2003)

\bibitem{wu2022unifying}
Wu, C.H., De~la Torre, F.: Unifying diffusion models' latent space, with applications to cyclediffusion and guidance. arXiv preprint arXiv:2210.05559  (2022)

\bibitem{xia2022gan}
Xia, W., Zhang, Y., Yang, Y., Xue, J.H., Zhou, B., Yang, M.H.: Gan inversion: A survey. IEEE transactions on pattern analysis and machine intelligence  \textbf{45}(3),  3121--3138 (2022)

\bibitem{yue2024exploring}
Yue, Z., Wang, J., Sun, Q., Ji, L., Chang, E.I.C., Zhang, H.: Exploring diffusion time-steps for unsupervised representation learning. In: The Twelfth International Conference on Learning Representations (2024)

\bibitem{zhang2022fast}
Zhang, Q., Chen, Y.: Fast sampling of diffusion models with exponential integrator. In: The Eleventh International Conference on Learning Representations (2023)

\bibitem{zhang2018perceptual}
Zhang, R., Isola, P., Efros, A.A., Shechtman, E., Wang, O.: The unreasonable effectiveness of deep features as a perceptual metric. In: CVPR (2018)

\end{thebibliography}
